\documentclass[lettersize,journal]{IEEEtran}
\usepackage{amsmath,amsfonts}
\usepackage{amsmath,amssymb,amsthm}
\usepackage{algorithmic}
\usepackage{algorithm}
\usepackage{array}
\usepackage[caption=false,font=footnotesize,labelfont=sf,textfont=sf]{subfig}
\usepackage{textcomp}
\usepackage{stfloats}
\usepackage{url}
\usepackage{verbatim}
\usepackage{graphicx}
\usepackage{cite}
\usepackage{tcolorbox}
\usepackage{booktabs} 
\usepackage{orcidlink}
\hypersetup{hidelinks,
	colorlinks=true,
	allcolors=blue,
	pdfstartview=Fit,
	breaklinks=true}
\usepackage{float}
\DeclareMathOperator*{\argmax}{argmax} %为了使用 \argmax
\usepackage{makecell}
\newtheorem{theorem}{Theorem}
\newtheorem{lemma}{Lemma} 

\begin{document}

\title{Model-Based Offline Reinforcement Learning with Adversarial Data Augmentation}

% Model-Based Offline Reinforcement Learning with Data Augmented Adversarial Policy Optimization

\author{Anonymous Submission}

\author{Hongye Cao, Fan Feng, Jing Huo, Shangdong Yang, Meng Fang, Tianpei Yang, and Yang Gao
% \thanks{Corresponding
% authors: Jing Huo; Shangdong Yang.}
\thanks{
Hongye Cao, Jing Huo, Tianpei Yang, and Yang Gao are with the National Key Laboratory for Novel Software Technology, Nanjing University, Nanjing 210093, China (e-mail: hongyecao528@gmail.com; huojing@nju.edu.cn; tianpei.yang@nju.edu.cn; gaoy@nju.edu.cn).
}
\thanks{
Fan Feng is with the Department of Electrical Engineering, City University of Hong Kong, Hong Kong, China (email: fan.feng@my.cityu.edu.hk).
}
\thanks{
Shangdong Yang is with the School of Computer Science, Nanjing University of Posts and Telecommunications, Nanjing 210023, China (e-mail: sdyang@njupt.edu.cn).
}
\thanks{
Meng Fang is with the Department of Computer Science, University of Liverpool, L69 3BX Liverpool, U.K. (e-mail: meng.fang@liverpool.ac.uk).
}
}

% % The paper headers
% \markboth{Journal of \LaTeX\ Class Files,~Vol.~14, No.~8, August~2021}%
% {Shell \MakeLowercase{\textit{et al.}}: A Sample Article Using IEEEtran.cls for IEEE Journals}

% \IEEEpubid{0000--0000/00\$00.00~\copyright~2021 IEEE}
% % Remember, if you use this you must call \IEEEpubidadjcol in the second
% % column for its text to clear the IEEEpubid mark.

\maketitle

\begin{abstract}
    Model-based offline Reinforcement Learning (RL) constructs environment models from offline datasets to perform conservative policy optimization. 
    Existing approaches focus on learning state transitions through ensemble models, rollouting conservative estimation to mitigate extrapolation errors. 
    However, the static data makes it challenging to develop a robust policy, and offline agents cannot access the environment to gather new data.
    To address these challenges, we introduce \textbf{M}odel-based \textbf{O}ffline \textbf{R}einforcement learning with \textbf{A}dversaria\textbf{L} data augmentation (MORAL). 
    In MORAL, we replace the fixed horizon rollout by employing adversaria data augmentation to execute alternating sampling with ensemble models to enrich training data. Specifically, this adversarial process dynamically selects ensemble models against policy for biased sampling, mitigating the optimistic estimation of fixed models, thus robustly expanding the training data for policy optimization. 
    Moreover, a differential factor is integrated into the adversarial process for regularization, ensuring error minimization in extrapolations. 
    This data-augmented optimization adapts to diverse offline tasks without rollout horizon tuning, showing remarkable applicability. 
    Extensive experiments on D4RL benchmark demonstrate that MORAL outperforms other model-based offline RL methods in terms of policy learning and sample efficiency.
\end{abstract}


\section{Introduction}
Reinforcement Learning (RL) solves sequential decision problems through interaction with the environment~\cite{sutton2018reinforcement}. 
RL has been applied in various real-world applications, including robotics~\cite{10771594, 10145841}, smart transportation~\cite{9537641}, and sequential recommendation systems~\cite{10144689}.
However, in safety-critical areas, online interaction is dangerous and expensive ~\cite{yu2021reinforcement,kiran2021deep}. Hence, model-based offline RL proposes to directly learn the state transition from datasets without online interaction, and this learning paradigm is widely regarded as an approach for deploying RL in real-world settings, such as robotic manipulation~\cite{9940310, Kumar2019}, autonomous driving~\cite{li2023boosting}, and healthcare~\cite{tang2021model}.

In offline setting, the learned policy of the agent is susceptible to extrapolation errors, which caused by the distribution shift between environmental transitions and offline datasets generated by the behavior policy~\cite{10310284,10301548}. These errors can be accumulated by bootstrapping, resulting in severe estimation errors. Naturally, it is challenging to correct these errors due to the lack of online exploration in offline RL. 

Pioneer attempts to mitigate these challenges in existing methods focus on the policy constraints~\cite{anuncertainty,Kumar2020}. 
In particular, model-free approaches~\cite{Kumar2020, Xu2022} introduce regularization for the value functions to confine the learned policy within the offline data manifold. Conversely, the model-based counterpart~\cite{mopo,kidambi2020morel} involves constructing environment models from offline data and implementing conservative rollouts to enrich training data for policy optimization, as illustrated in Fig.~\ref{sub_1}. Since model-based approaches easily access prior knowledge of dynamics, they are likely more effective at enhancing policy stability and generalization in policy learning~\cite{chen2023offline,pmdb}. 

\begin{figure}[t]
\centering
\subfloat[]{\includegraphics[width=1.75in]{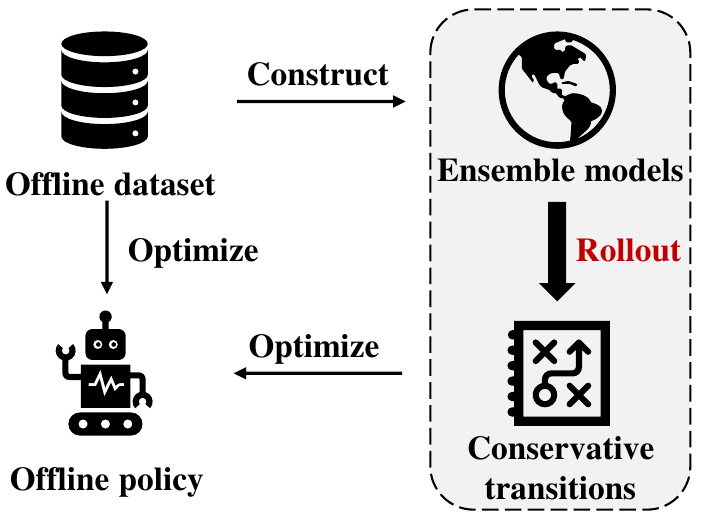}%
\label{sub_1}
}
\hfil
\subfloat[]{\includegraphics[width=1.6in]{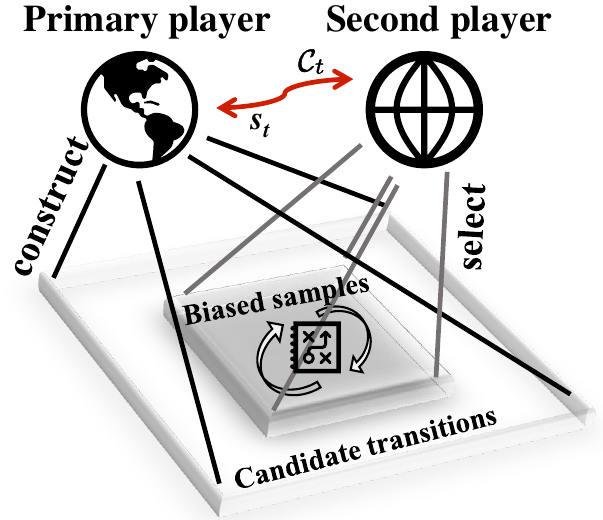}%
\label{sub_2}
}
\caption{(a) The framework of model-based offline RL. (b) We utilize a adversarial data augmentation process to execute alternating sampling, replacing the fixed horizon rollout marked red in (a). The primary player constructs candidate transition set and the second player selects biased samples to enrich training data without the rollout horizon configuration across datasets and tasks.}
% \vspace{-0.2cm}%%减小图片上间隔
\label{pipeline}
\end{figure}

\IEEEpubidadjcol 
Learning from the offline data is difficult in offline settings since the data does not cover the full distribution. Although model-based approaches offer advantages in data augmentation, their design patterns tend to be inefficient and inflexible. 
Specifically, these methods employ offline datasets to train models that approximate environmental dynamics, with policies further optimized using synthetic data generated from model rollouts~\cite{kidambi2020morel,chen2023offline}. Crafting the rollout horizon meticulously for diverse datasets often leads to low sample efficiency and restricted applicability. The scope of rollout exploration is influenced by different datasets and environments, requiring a balance between the risk and return associated with the rollout process. 
Consequently, designing appropriate rollout strategies for diverse scenarios remains a long-standing challenge~\cite{rambo,crop}.
To be specific, exploratory small-scale rollouts provide insufficient data for offline optimization, failing to mitigate the impact of extrapolation errors and reducing sample efficiency. Conversely, extensive large-scale rollout exploration may lead to overly conservative policies that constrain policy performance. Thus, the central question becomes \textit{whether we can adaptively design the rollout horizon for data augmentation, without rigid pre-specifications, to ensure flexibility and robustness across varying scenarios}. 

To this end, we introduce a novel approach named \textbf{M}odel-based \textbf{O}ffline \textbf{R}einforcement learning with \textbf{A}dversaria\textbf{L} data augmentation (MORAL). 
MORAL first trains ensemble models using collected offline datasets for data augmentation by executing model rollout. To address the challenge of lacking online exploration, MORAL employs these trained models to execute alternating sampling, constructing an adversarial dataset that enriches the policy optimization process more robustly than traditional rollout.
MORAL eliminates the need for manually setting the rollout horizon for different datasets and tasks by framing the problem to an adversarial paradigm. 
%Central to our approach is the implementation of ensemble models combined with an adversarial game-driven mechanism. This mechanism selectively augments samples to optimize policy robustness. This mechanism adaptively and flexibly augments samples to optimize policy robustness.
%\fan{first highlight that we do not need to specify manually the quantity of ensemble models and rollout horizon}

Specifically, MORAL leverages the alternating sampling process (Fig.~\ref{sub_2}) for adversarial data augmentation, providing a flexible sampling strategy that adapts to the requirements of policy optimization. In this adversarial process, the primary player, consisting of a dynamic selection of ensemble models, performs offline sampling to enhance data diversity. The secondary player introduces stochastic disturbances and selects biased data from the samples conduct by the first player as part of the adversarial process, preventing overly optimistic data estimation and achieving robust adversarial data augmentation. 
Moreover, due to the extrapolation errors in offline rollouts of ensemble models, MORAL introduces a differential factor (DF) to the policy optimization process for regularization. Hence, this adversarial data-augmented design in MORAL effectively mitigates policy instability caused by discrepancies among a large number of ensemble models and fixed rollout horizons. Notably, despite leveraging a substantial ensemble of models, MORAL maintains computational efficiency, avoiding any increase in execution time compared to other offline model-based RL methods (as evaluated in Section~\ref{sec:time_cost}). Overall, by expanding policy spectrum and minimizing the risk of converging to suboptimal local maxima, MORAL ensures a robust policy learning process.

The main contributions of this paper are 3-fold:
\begin{itemize}

\item We frame an adversarial process to conduct alternating sampling, robustly enhancing policy performance with data augmentation under stochastic perturbations induced by this process. 
This approach efficiently utilizes ensemble models, replacing the fixed horizon rollout, and introduces an applicable architecture for model-based offline RL.
\item 
% differential factor是什么，后面正文好像没有了，实验消融里面有
To prevent extrapolation errors during the adversarial process, we introduce a differential factor to regularize the offline policy. 
This iterative regularization prevents divergence and improves sample efficiency during the policy optimization.

\item Extensive experiments conducted on widely studied D4RL benchmark demonstrate the superior performance of MORAL compared to other model-based offline RL methods, achieving an average score of $86.3$ across $15$ offline tasks and improving sample efficiency. Notably, MORAL consistently exhibits robust performance across a wide range of tasks, excelling in both Q-value estimation and uncertainty quantification.
\end{itemize}

The remainder of this paper is organized as follows. Section~\ref{sec:related work} provides an overview of related works, including offline RL, and adversarial RL. Section~\ref{sec:preliminaries} details the necessary preliminaries and backgrounds. Section~\ref{sec:approach} presents the proposed approach in detail. In section \ref{sec:experiments}, we introduce the experiments conducted to demonstrate the superiority of MORAL. Finally, section~\ref{sec:conclusion} draws conclusions and discusses future works.

\section{Related Work}
\label{sec:related work}
\subsection{Offline RL} 
Offline RL seeks to learn a policy from offline datasets without engaging in online exploration. The pioneering works directly utilize off-policy RL algorithms for policy learning in the offline settings~\cite{kumar2019stabilizing,fujimoto2019off}. However, policy learning failures occurred due to inherent discrepancies between the real-world environment and offline datasets, exacerbated by the lack of the online corrective mechanism~\cite{mopo}.
The learned policy is vulnerable to extrapolation errors arising from the distribution shift between offline datasets and the states visited by the behavior policy during offline training~\cite{chen2023offline,10432784}. To tackle these challenges, offline RL restricts policy divergence, enabling the acquisition of conservative policies on offline datasets~\cite{kidambi2020morel,Kumar2020}.

Existing offline RL methods can be broadly classified into two main approaches: model-free and model-based offline RL~\cite{prudencio2023survey,levine2020offline}. 
Policies within model-free offline RL encompass importance sampling algorithms, which constrain the learned policy to resemble the behavior policy~\cite{Kumar2019,Liu2019OFF}. Model-based offline RL methods operate within the framework of supervised learning~\cite{wang2024goplan,kidambi2020morel}. These methods involve the construction of ensemble models through the acquisition of transitions from offline datasets and execute rollout for conservative data augmentation to optimize offline policy\cite{crop,goo2022you,kidambi2020morel}. However, existing methods require meticulous design for the quantity of ensemble models and the rollout horizon across diverse tasks, resulting in increased algorithmic costs and limited universality (as analyzed in Section~\ref{sec:problem}). In our work, we focus on model-based approaches and aim to address the fundamental challenges in the fixed horizon rollout process of data augmentation via adversarial frameworks. 
% For additional related works of model-free offline RL, please refer to Appendix B. 

\subsection{Adversarial RL} 
Adversarial RL involves training an adversarial agent to introduce disturbances during the learning process, aiming to impede the target policy learning~\cite{pinto2017robust,Chenadv2022}. In the highly adversarial environment, this approach ensures that the target policy learns a robust policy~\cite{oikarinen2021robust,zhai2022robust}.
Existing methods involve constructing adversarial agents that conform to the state and action spaces of the original policy and introduce disturbances through adversarial reward~\cite{pinto2017robust,pmdb}. In model-based offline RL, adversarial methods introduce perturbation factors between optimizing the policy and the model. 
However, these approaches (i.e., RAMBO~\cite{rambo}, ARMOR~\cite{armor}) not only amplify computational complexity but also fall short of resolving the problem of inadequate applicability. 
Inspired by adversarial learning, we propose an adversarial data augmentation method for model-based offline RL innovatively, where the construction of adversarial players does not need to adhere to the target policy and replaces the fixed horizon rollout process in existing state-of-the-arts.

\section{Preliminaries and Backgrounds}
\label{sec:preliminaries}
\subsection{Markov Decision Process (MDP)}
In the RL environments, the interaction between the agent and the environment can be formalized as an MDP. The standard MDP is defined by the tuple $ \mathcal{M} = \langle \mathcal{S}, \mathcal{A}, T, \mu_0, r, \gamma \rangle $, where $\mathcal{S}$ denotes the state space, $\mathcal{A}$ represents the action space, $T(s' | s, a)$ is the transition dynamic model, $r(s, a)$ is the reward function, and $\mu_0$ is the distribution of the initial state $s_0$. The discount factor $\gamma \in [0, 1)$ is also included. The objective of RL is to learn a policy $\pi: \mathcal{S} \times \mathcal{A} \to [0, 1]$ that maximizes the expected discounted cumulative reward ${\eta _\mathcal{M}}(\pi) := \mathbb{E}_{s_0 \sim \mu_0, s_t \sim T, a_t \sim \pi} \left[\sum\nolimits_{t = 0}^\infty {\gamma^t}r(s_t, a_t)\right]$. The value function $V_\mathcal{M}^\pi(s): = \mathbb{E}_{s_t \sim T, a_t \sim \pi} \left[\sum\nolimits_{t = 0}^\infty {\gamma^t}r(s_t, a_t) \,|\, s_0 = s\right]$ represents the expected discounted return under policy $\pi$ when starting from the state $s$.

\begin{figure}
\centering
% \vspace{-0.5cm}%%减小图片上间隔
% \includegraphics[width=3.4in]{figs/half.pdf}
\includegraphics[width=0.49\linewidth]{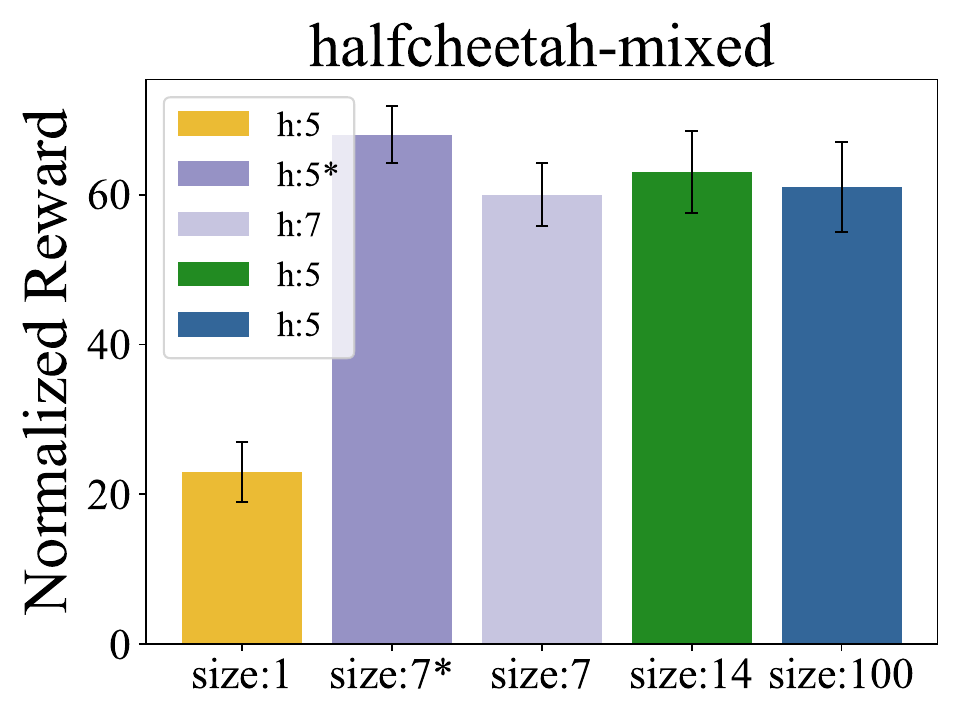}
\includegraphics[width=0.49\linewidth]{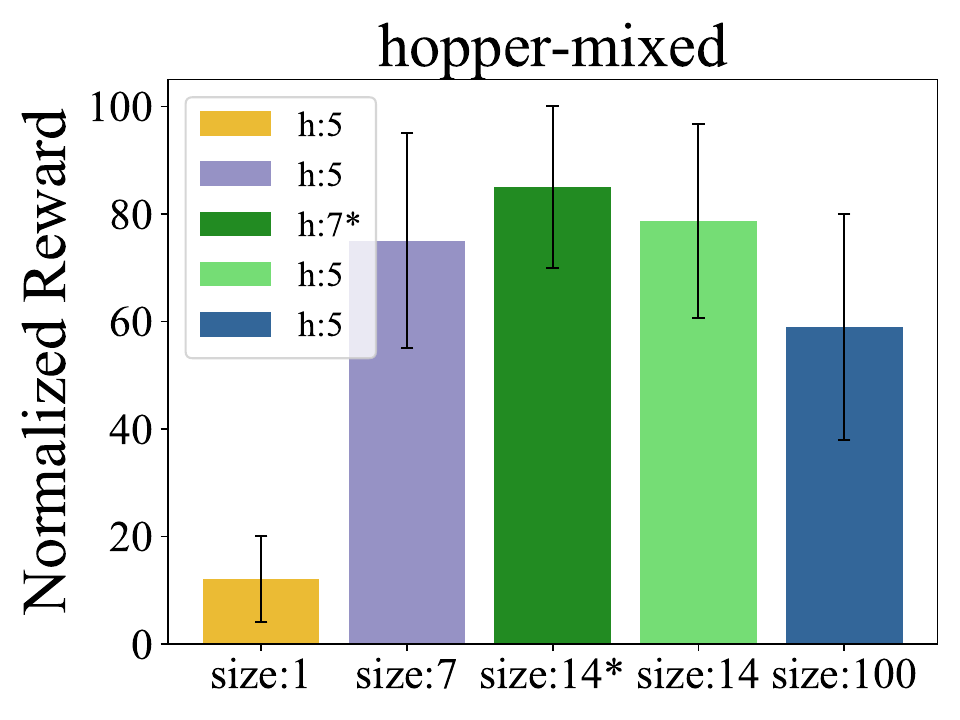}

\caption{The comparative experiments are conducted across various ensemble sizes and rollout horizons (h) within mixed datasets of two environments. We use stars (*) to denote the setups with optimal results. We find that the performance exhibited substantial differences in different settings.}
% \vspace{-0.5cm}%%减小图片上间隔
\label{fig1}
\end{figure}

\subsection{Model-based Offline RL with Ensemble Models}
\subsubsection{Definitions}
In offline RL, the policy is exclusively learned from pre-collected datasets and does not interact with the real-world environment. The offline dataset $\mathcal{D}_{\rm{env}} = \left\{(s_i, a_i, r_i, s_i')\right\}_{i=1}^{|\mathcal{D}_{\rm{env}}|}$ comprises all the gathered state-action transitions. In model-based offline RL, $\hat T(s' | s, a)$ is the dynamic model estimated from transitions in $\mathcal{D}_{\rm{env}}$, which is trained by maximum likelihood estimation: $\min_{\hat{T}} \mathbb{E}_{(s,a,s')\sim \mathcal{D}_{\rm{env}} }[-\log{\hat{T}(s'|s,a) }] $. This dynamic model defines the estimated MDP $\hat {\mathcal{M}} = \langle \mathcal{S}, \mathcal{A}, \hat T, \mu_0, r, \gamma \rangle$. 
% $\mathcal{P}_{\hat T, t}^\pi (s)$ denotes the probability of being in state $s$ at time step $t$ when actions and transitions are sampled from $\pi$ and $\hat T$. Additionally, $\rho _{\hat T}^\pi (s, a): = \pi (a|s)\sum\nolimits_{t = 0}^\infty {\gamma^t}\mathcal{P}_{\hat T, t}^\pi (s)$ represents the discounted occupancy measure of policy $\pi$ under $\hat T$. 
Most existing approaches employ the ensemble models $\{\hat{T}^{i}\}_{i=1}^{N}$ to execute fixed horizon rollout for data augmentation $\mathcal{D}_{\rm{aug}}$, where $N$ is the number of ensemble models.
The objective is to learn a policy that maximizes ${\eta _{\hat{\mathcal{M}}}}(\pi)$ using the datasets $\mathcal{D}_{\rm{env}} \cup \mathcal{D}_{\rm{aug}}$~\cite{mopo,chen2023offline,Yu2021}.

\subsubsection{Challenges on Data Augmentation}
\label{sec:problem}
The meticulous design for the quantity of ensemble models and the rollout horizon across diverse tasks during data augmentation, may result in increased algorithmic costs and limited applicability. 

To better investigate this challenge, we conducted an empirical analysis focusing on the relationship between learning performance, ensemble model quantity, and rollout horizons. Our study used MOPO~\cite{mopo}, one of the representative methods, for evaluation in halfcheetah and hopper environments~\cite{Todorov2012}. 
Results (Fig.~\ref{fig1}) indicate that policy performance declines sharply with one single dynamic model but improves as the ensemble size increases. 
However, the optimal ensemble size and rollout horizon differ across tasks and environments, highlighting the difficulty in finding a universal optimal value empirically. Therefore, crafting a rollout process tailored to diverse tasks and environments poses a persistent challenge, demanding the achievement of a balance between exploration and exploitation. 
% It is crucial to achieve optimal utilization of ensemble models to robustly improve the performance of offline policy. 
%Data augmentation can enrich training data to robustly improve policy performance~\cite{tkde_data1,tkde_data2}. 

%an adversarial framework based on the MZG, which achieves alternating biased sampling for policy optimization under stochastic perturbations without costly configuration. 

\subsection{Adversarial RL}

Adversarial RL aims to find a robust policy, by posing the problem as a two-player zero sum game against adversary policy. The alternating decision-making framework facilitates efficient policy learning through adversarial interactions. At timestep $t$, two players take actions $a_{t} \sim \pi^{\theta}{(s_{t})}$, $\bar{a}_{t} \sim \pi^{\vartheta}{(\bar{s}_{t})}$, where $\theta$ is the parameter of the primary player policy and $\vartheta$ is the parameter of the second player policy. 
The primary player gets a reward $r_{t}$, and the second player gets a reward $\bar{r}_{t}$, where $\bar{r}_{t}=-r_{t}$. 
This two-player zero-sum game process efficiently facilitates alternating learning for agents.
The player seeks to maximize the 
cumulative reward: 
\begin{eqnarray}
R(\pi^{\theta}, \pi^{\vartheta})=\mathbb{E}_{s_{0}\sim \mu,a \sim \pi^{\theta}{(s)}, \bar{a} \sim \pi^{\vartheta}(\bar{s}) }\left [ { \sum_{t=0}^{\infty}} \gamma^{t} r_{t}(s_{t},a_{t},\bar{a}_{t})  \right].
\end{eqnarray}

This problem is solved by using the minimax equilibrium for this game with optimal equilibrium reward \cite{patek1997stochastic,perolat2015approximate}, learning a robust policy under this process:
\begin{eqnarray}
R^{\ast}=\min_{\pi^{\theta}}\max_{\pi^{\vartheta}}R(\pi^{\theta}, \pi^{\vartheta})    =\max_{\pi^{\theta}}\min_{\pi^{\vartheta}}R(\pi^{\theta}, \pi^{\vartheta}).
\end{eqnarray}
Hence, the goal is to find a robust agent policy $\pi$ against the adversary policy, formulated as a two-player zero-sum game:
\begin{eqnarray} 
   \pi =\argmax_{\pi^{\theta}}{\min_{\pi^{\vartheta}}{V_{\tilde{\mathcal{M}}}(\pi^{\theta},\pi^{\vartheta})}},
\end{eqnarray}
where $V_{\tilde{\mathcal{M}}}(\pi^{\theta},\pi^{\vartheta})$ is the expected value from executing two-player policies $\pi^{\theta}$ and $\pi^{\vartheta}$.

For a comprehensive analysis of this adversarial game framework, we list Algorithm~\ref{alg:mzg} below. 
The initial parameters for both players are randomly generated. In each of $I$ iterations, a two-stage alternating optimization for two players is executed. 
Firstly, for each iteration, the parameter $\vartheta$ of the second player is held constant, while the parameter $\theta$ of the primary player is optimized to maximize the reward $R$. The rollout generates trajectories based on the environment and policies by using two players for data augmentation (Lines 2-6). 
In the second step, the parameter of the primary player is then held constant. The rollout process is executed and the parameter of the second player is optimized (Lines 7-11). This alternating procedure is repeated for a total of $I$ iterations. Finally, we get two players policies (Line 13). In MORAL, we incorporate this game into model-based offline RL and the construction of adversarial players does not need to adhere to the target policy. By adopting an actor-critic learning architecture, the alternating sampling between the two players effectively reduces the computational complexity of the algorithm, providing an elegant solution to complex policy optimization challenges.

\begin{algorithm}[t]
    \caption{Adversarial RL}
    \label{alg:mzg}
    \textbf{Input}: Environment $\mathcal{E} $; Primary player policy and second player policy of $\pi^\theta$, $\pi^\vartheta$
    \\
    \textbf{Parameter}: Initialize parameters $\theta$ for $\pi^\theta$ and $\vartheta$ for $\pi^\vartheta$ randomly
    \begin{algorithmic}[1] %[1] enables line numbers
        \FOR {$i=1$,$\dots$, $I$}
        \STATE $\pi^{\theta}_{i}$ $\leftarrow$  $\pi^{\theta}_{i-1}$
        \FOR{$t=1$,$\dots$, $I_{\theta}$}
        \STATE  \{ ($s^{i}_{t}$,$a^{i}_{t}$, $\Bar{s}^{i}_{t}$, $\Bar{a}^{i}_{t}$, $r^{i}_{t}$, $\Bar{r}^{i}_{t}$) \} $\leftarrow$ rollout($\mathcal{E}$, $\pi^{\theta}_{i}$, $\pi^{\vartheta}_{i}$)
        \STATE $\pi^{\theta}_{i}$ $\leftarrow$ Optimizer({($s^{i}_{t}$,$a^{i}_{t}$,$r^{i}_{t}$)},$\theta$,  $\pi^{\theta}_{i}$)
        \ENDFOR

         \STATE $\pi^{\vartheta}_{i}$ $\leftarrow$  $\pi^{\vartheta}_{i-1}$
        \FOR{$t=1$,$\dots$, $I_{\vartheta}$}
        \STATE  \{ ($s^{i}_{t}$,$a^{i}_{t}$, $\Bar{s}^{i}_{t}$, $\Bar{a}^{i}_{t}$, $r^{i}_{t}$, $\Bar{r}^{i}_{t}$) \} $\leftarrow$ rollout($\mathcal{E}$, $\pi^{\theta}_{i}$, $\pi^{\vartheta}_{i}$)
        \STATE $\pi^{\vartheta}_{i}$ $\leftarrow$ Optimizer({($\bar{s}^{i}_{t}$,$\bar{a}^{i}_{t}$,$\bar{r}^{i}_{t}$)},$\vartheta$,  $\pi^{\vartheta}_{i}$)
        \ENDFOR
        \ENDFOR
        \RETURN $\pi^{\theta}_{I}$, $\pi^{\vartheta}_{I}$
        % \STATE Let $t=0$.
        % \WHILE{condition}
        % \STATE Do some action.
        % \IF {conditional}
        % \STATE Perform task A.
        % \ELSE
        % \STATE Perform task B.
        % \ENDIF
        % \ENDWHILE
        % \STATE \textbf{return} solution
    \end{algorithmic}
\end{algorithm}

\begin{figure*}[t]
\centering
% \vspace{-0.5cm}%%减小图片上间隔
% \includegraphics[width=3.4in]{figs/half.pdf}
\includegraphics[width=0.95\linewidth]{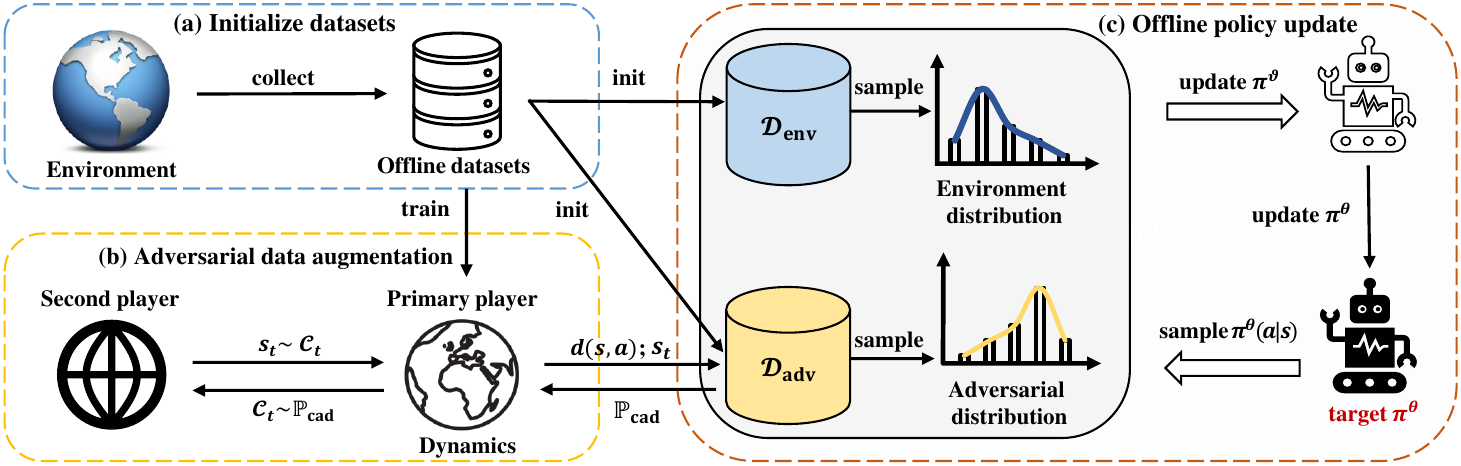}
\caption{Overall framework of MORAL. (a) Initialize offline datasets $\mathcal{D}_{\rm{env}}$, $\mathcal{D}_{\rm{adv}}$ and train ensemble models. (b) Alternating sampling with two players for data augmentation in $\mathcal{D}_{\rm{adv}}$. (c) Update offline policy with both $\mathcal{D}_{\rm{env}}$ and $\mathcal{D}_{\rm{adv}}$, and provide sample policy $\pi^\theta$ in $\mathcal{D}_{\rm{adv}}$.}
% \vspace{-0.3cm}%%减小图片上间隔
\label{moral}
\end{figure*}

\section{Proposed Method}
\label{sec:approach}

% \fan{First paragraph: 1) 1-2 sentences introduce our approach again and highlight the differences with prev methods. 2) mention the fig}
In this section, we introduce MORAL (Fig.~\ref{moral}), which advances beyond traditional methods by replacing the conventional fixed horizon rollout with an adversarial framework for adaptive and flexible data augmentation in model-based offline RL. 

Specifically, MORAL facilitates offline alternating sampling for adversarial data augmentation, utilizing ensemble models to navigate the sampling process. Additionally, a differential factor is incorporated as a regularization mechanism. MORAL not only extends the capability of the learning model but also robustly enhances policy performance. The integration of adversarial data augmentation within MORAL provides an applicable framework that improves the robustness and effectiveness of policy learning across various domains.

%As illustrated in Fig.~\ref{moral}, we propose MORAL, which replaces the fixed horizon rollout in existing state-of-the-arts with adversarial game-driven data augmentation. 

%which incorporates the MZG process for offline alternating sampling based on ensemble models and introduces a differential factor for regularization. MORAL replaces the fixed horizon rollout in existing state-of-the-arts with adversarial game-driven data augmentation, robustly enhancing policy performance with a universal framework.

In this section, we first introduce how to leverage the alternating sampling for adversarial data augmentation. 
Subsequently, we give the pipeline of the policy optimization strategies in MORAL. Following this, we present the differential factor to regularize the offline policy, aiming for error minimization in extrapolations. Finally, we provide insights into the practical implementation of MORAL.

% \begin{figure}[]
% \centering
% \subfloat[]{\includegraphics[width=1.7in]{figs/pipeline/pip_1.pdf}%
% \label{sub1}
% }
% \hfil
% \subfloat[]{\includegraphics[width=1.6in]{figs/pipeline/pip_2.pdf}%
% \label{sub2}
% }
% \caption{(a) The framework of Model-based offline RL. (b) We formulate the MZG in MORAL, replacing the fixed horizon rollout marked red in (a).}
% % \fan{It would be better to make some minor changes like: (a). prev. model-based offline RL (with citations); (b). the framework of this paper (could be pretty much similar to (a), but highlight the difference); (c). zoom in the major difference.}}
% % \vspace{-0.5cm}%%减小图片上间隔
% \label{pipeline}
% \end{figure}

% \subsection{Markov Zero-sum Game} 
% \fan{We can remove some fundamental ideas to sec 3 and highlight how we utilize the MZG in offline RL here in this sec. Consider giving more straightforward definitions of how we use it in MORAL}

\subsection{Adversarial Data Augmentation}

In model-based offline RL, the ensemble models $\{\hat{T}^{i}\}_{i=1}^{N}$ execute the fixed horizon rollout procedure to sample trajectories for data augmentation, which necessitates a meticulous design process. 
In MORAL, this procedure has transformed into an alternating sampling process between two players without rollout horizon configuration. We incorporate this adversarial game process into offline RL, where the primary player represents the dynamic model to emit transitions, and the second player embodies the selection of MDP transition from ensemble models, as shown in Fig.~\ref{moral}. 
% \fan{Annotate the roles of these two players in fig 2}
The primary player seeks to optimize a dependable policy for the relevant MDP in response to stochastic disruptions caused by the biased samples from the second player.

The conventional approach to identifying equilibrium policies often necessitates the greedy resolution of $M$ minimax equilibria within a zero-sum matrix game, where $M$ corresponds to the quantity of observed data points. The computational complexity of this greedy solution grows exponentially with the cardinality of the action spaces, rendering it impractical~\cite{perolat2015approximate,pinto2017robust}. 
In our approach, we only need to approximate the advantage function without the need to ascertain the equilibrium solution at every iteration. Hence, we focus on learning stationary policies $\pi^{\theta^{\ast}}$ and $\pi^{\vartheta^{\ast}}$. 
By adopting these policies, we can circumvent the need for computationally intensive optimization at each iteration. 

Specifically, we apply this adversarial framework through a concrete sampling mechanism. The primary player constructs an $N'$-sized state candidate set from ensemble model rollouts, while the secondary player selects states using a k-th minimum criterion on these transitions. Both players share the same state-action space of transition selection but with opposing objectives - the primary player aims to maximize data diversity while the secondary player introduces pessimistic bias that helps prevent policy extrapolation errors, driving the system toward a minimax equilibrium. \textbf{Through this interplay, the sampling process constructs a dataset that balances diversity with conservative estimates, enabling robust policy optimization in the offline setting.}

Both players operate on the same offline policy but employ opposing selection criteria during the alternating sampling process. The adversarial nature manifests in the sampling process rather than in policy updates, as both players use actor-critic methods to update their policies based on the collected transitions. This framework effectively leverages ensemble models while addressing the uncertainty inherent in fixed-step rollouts.

% The primary player seeks to optimize a dependable policy for the relevant MDP in response to stochastic disruptions caused 
% by the biased samples from the second player.

% \begin{algorithm}[tb]

\subsection{Policy Optimization}

In MORAL, the primary player constructs a set of system dynamic transitions using candidate ensemble models. Based on these dynamic transitions, the second player selects a transition to expand offline data for policy optimization. 
The establishment of dynamic transition sets fully capitalizes on the crafted extensive ensemble models, ensuring the abundance of data. Furthermore, the adversarial game achieves robust optimization to strike a balance between return and risk in the offline policy learning process, avoiding overly conservative approaches.
This adversarial game process accomplishes alternating sampling to achieve robust policy optimization.

% TODO: Fan

For the primary player, its state space $\mathcal{S}$, action space $\mathcal{A}$, and reward function remain consistent with the original MDP. During each step in the game process, the primary player takes actions, the game produces a set of system transition candidates from ensemble models denoted as $\mathcal{C}$. This set subsequently serves as the state for the second player. Formally, $\mathcal{C}$ of $(s,a)$ is generated according to the following:
% \iffalse
% \begin{eqnarray}
% \label{equ:C}
% \mathcal{C} = \mathbb{P}_{\rm{cad}}^{N'} \left( \{ \hat{T}^{i}(s_{i}|s, a) \}_{i=1}^{N}  \right),
% \end{eqnarray} \fan{distribution or set? Should we formulize it as the samples from the distribution?}
% where $N$ represents the number of system transitions from ensemble models, $N'$ is the number of candidate states and $\mathbb{P}_{\rm{cad}}^{N'}$ is the game transition distribution. The candidate set can be directly presented as follows:
% \begin{eqnarray}
% \label{equ:C2}
% \mathcal{C} = \{ s_{1},s_{2}, \cdots, s_{N'} \}. 
% \end{eqnarray}
% \fi
\begin{eqnarray}
\label{equ:C}
\mathcal{C}_t \sim \mathbb{P}_{\mathrm{cad}} \left( \{ \hat{T}^{i}(s_{i}|s, a) \}_{i=1}^{N}  \right),
\end{eqnarray}
where $N$ represents the number of system transitions from ensemble models.  $\mathcal{C}_t$ is the set of states at time $t$ sampled according to the distribution $\mathbb{P}_{\mathrm{cad}}$, which represents a stable probability distribution used for sampling candidate states\footnote{The subscript $t$ in $\mathcal{C}_t$ indicates its temporal nature. For simplicity, the explicit time dependency is not shown in the later sections.}.  Hence, the candidate set $\mathcal{C}_t$ consists of $N'$ sampled states: $\mathcal{C}_t = \{ s_{1},s_{2}, \cdots, s_{N'} \}$. 

% $\mathcal{P}$ is the dynamic state transition, $\mathbb{P}_{\hat{T}}$ is the game transition distribution over $\mathcal{P}$.

% \begin{eqnarray}
% \label{equ:C}
% \mathcal{P}(\mathcal{C})=\prod^{N'}_{n=1}\mathcal{P}(s_{n}), 
% \end{eqnarray}
% where $N'$ represents the number of system transitions in the dynamic set.

Based on the ensemble models, the primary player constructs a dynamic set of system transitions for the second player to choose. 
According to Eq.~\ref{equ:C}, the elements within $\mathcal{C}$ are independent and distributed samples that adhere to the game distribution. The candidate set is stochastic for each step during the game process. 
Hence, the development of this dynamic set extensively employs ensemble models, avoiding the excessively pessimistic posed by fixed model execution in long-term policy learning.

For the second player, the state is defined by the dynamic set $\mathcal{C}$. The action taken by the second player involves choosing a transition from this set:
\begin{eqnarray}
    \label{equ:asec}
    % \mathcal{G}(s)=\pi^{\vartheta}(\bar{s}|\bar{s}=\mathcal{C}, a_{sec}),
    \bar{s} = \mathcal{P}(\bar{s}'|s, a^{\rm{sec}} ),
\end{eqnarray}
where $s \in \mathcal{C}$, and $a^{\rm{sec}} = \pi^{\vartheta}(s)$.
After the second player chooses the state $\bar{s}$ according to the set $\mathcal{C}$, the primary player receives $\bar{s}$ to continue the game process in MORAL. 
Through this interplay, the sampling process constructs a dataset that balances diversity with conservative estimates, enabling robust policy optimization in the offline setting. 
Hence, the cumulative discounted reward of the primary player with policy $\pi^{\theta}$ can be written as:
\begin{eqnarray}
\begin{aligned}
    V^{\pi}(s) := \underset{\mathcal{P}}{\mathbb{E}} \left[
    \mathbb{E} \left[
        \sum^{\infty}_{t=0}{\gamma^{t}r(s_t,a^{\rm{prm}}_t )|s_{0},\pi^{\theta}}
    \right]
    \right].
    \end{aligned}
\end{eqnarray}

% \fan{give the source of stochasticity for the $\mathbb{E}$}
We denote the selection operator of the action of the second player: $\min_{k}f(\mathcal{C})$, which denotes finding $k$th minimum of transition. The second player demonstrates varying degrees of aggressive disturbances to future rewards. Our goal can be written as:
\begin{eqnarray}
\hspace{-2mm}
\begin{aligned}
    Q^{\pi}(s,a) =   \underset{\mathcal{C}}{{\min}_{k}} \mathbb{E} \left[{\mathbb{E}}\left[\sum_{t=0}^{\infty}\gamma^{t}r(s_t,a_t) | s_{0}=s, a_{0}=a \right] \right].
    \end{aligned}
\end{eqnarray}

Utilizing biased sampling within this game, the entire process entails iteratively computing cumulative rewards through adversarial computation.
From the view of MDP, this behavior showcases a range of adaptability, stretching from pessimistic to optimistic when evaluating the policy $\pi$. The evaluation by the Bellman backup operator can be defined as:
\begin{eqnarray}
    \mathcal{T}^{\pi}Q(s,a)=r(s,a)+\gamma \mathbb{E}_{\mathbb{P}_{\hat{T}}^{N}}\left[ \underset{\mathcal{C}}{{\min}_{k}}  \mathbb{E}_{\pi}\left[Q(s',a')\right] \right].  
\end{eqnarray}

\subsection{Error Minimization in Extrapolations}
Due to the absence of online interactions, the difference between offline models and the real environment dynamic transition can significantly amplify the occurrence of extrapolation errors during long-term policy learning~\cite{mopo,chen2023offline}. Hence, applying MORAL to offline settings directly is not viable. We analyze the estimation errors in extrapolation and introduce the differential factor into the game to avoid the propagation of estimation errors. 

The estimation error $e_{\hat{ \mathcal{M}}}^\pi (s,a)$ concerning the genuine outcome variance between the model assumed to be optimal and the model in practical use is defined as follows: 
\begin{eqnarray}
e_{\hat{\mathcal{M}} }^{\pi} \left(s,a\right):=\mathbb{E}_{s\sim \hat{T}(s,a)}\left[V_{\hat{\mathcal{M}} }^{\pi}\left(s\right) \right] - \mathbb{E}_{s\sim T(s,a)}\left[V_{\mathcal{M}}^{\pi}\left(s\right)\right].
\label{eq4}
\end{eqnarray}

We then combine the objective of policy optimization to maximize the discounted cumulative reward. The Eq.~\ref{eq4} can be derived as follows:
\begin{eqnarray}
\eta _{\hat{\mathcal{M}} } \left(\pi \right)-\eta _{\mathcal{M}} \left(\pi \right)=\gamma \mathbb{E}_{\left(s,a\right)\sim\rho _{\hat{T}}^{\pi} }\left[e_{\hat{\mathcal{M}}}^{\pi}\left(s,a\right)\right].
\label{eq5}
\end{eqnarray}

Based on the estimation errors between models and the expected discounted return under policy $\pi$, the maximized target discounted return can be derived as follows:
\begin{eqnarray}
\begin{aligned}
\eta _{\mathcal{M}}\left(\pi \right) & =\mathbb{E}_{\left(s,a\right)\sim \rho _{\hat{T} }^{\pi } }\left[r(s,a)-\gamma e_{\hat{\mathcal{M}} }^{\pi }(s,a) \right] \\
& \ge \mathbb{E}_{\left(s,a\right)\sim \rho _{\hat{T} }^{\pi } }\left[r\left(s,a\right)-|\gamma e_{\hat{\mathcal{M}} }^{\pi }\left(s,a\right)| \right] \\
& = \mathbb{E}_{(s,a)\sim \rho _{\hat{T} }^{\pi } }\left[r\left(s,a\right)-d\left(s,a\right)\right] \\
& \ge \eta _{\hat{\mathcal{M}} }\left(\pi\right), 
\end{aligned}
\label{eq6}
\end{eqnarray}
where $\rho _{\hat{T} }^{\pi}$ denotes the discounted occupancey measure and $d(s,a)$ is the differential factor. The goal is to learn a policy that can maximize $ \mathbb{E}_{(s,a)\sim \rho _{\hat{T} }^{\pi } }[r(s,a)-d(s,a)]$.
% （uncertainty 与 MORAL的融合，最终的优化公式列出来）
Bootstrap ensembles have demonstrated theoretical consistency in estimating the population mean~\cite{bickel1981some} and have shown strong empirical performance in model-based RL~\cite{mopo,CaoYHCG23}. To reduce estimation errors, MORAL introduces the differential factors to the policy optimization process for regularization as follows:
\begin{eqnarray}
    d(s,a)=\max_{i\in[1,N]}\left|\left| {\textstyle \sum_{1}^{i}(s,a)}\right|\right|_{\rm{F}},
\end{eqnarray}
We employ the maximum of ensemble elements rather than the mean to prioritize conservatism and robustness, which has demonstrated effective in previous works~\cite{mopo,chen2023offline}. Finally, we combine the adversarial game framework with the differential factor to maximize the cumulative discounted reward:
\begin{eqnarray}
\begin{aligned}
    \bar{V}^{\pi}(s) := & \mathbb{E}  \underset{\mathcal{C}}{{\min}_{k}} \left[{\mathbb{E}}\left[\sum_{t=0}^{\infty}\gamma^{t} (r(s_t,a_t) 
    \right.\right.
   \\
    & \left.\left. -\alpha \max_{i\in[1,N]} || {\textstyle \sum_{1}^{i}(s,a)}||_{\rm{F}} ) | s_{0}, \pi \right] \right],
    \end{aligned}
\end{eqnarray}
% \begin{eqnarray}
% \begin{aligned}
%     \bar{V}^{\pi}(s):= & \underset{\mu _{0},\pi,\mathbb{P}_{\hat{T}}^{N}}{\mathbb{E}}\left \lfloor \min \right \rfloor_{c_{0} \in \mathcal{C} }^{k}  \left[\underset{c_{\infty},\pi}{\mathbb{E}}\left[\sum_{t=0}^{\infty}\gamma^{t}(r(s_t,a_t)
%     \right.\right.
%    \\
%     & \left.\left. -\alpha \max_{i=1,\dots ,n}|| {\textstyle \sum_{\phi}^{i}(s,a)}||_{F}  ) \right] \right],
%     \end{aligned}
% \end{eqnarray}
where $\alpha$ is a user-chosen penalty coefficient. The balance between reward maximization and this factor serves a crucial purpose: while pursuing optimal rewards, the formulation simultaneously constrains potential extrapolation errors. From a physical perspective, this composition effectively mitigates the negative impacts of ensemble model errors by implementing a pessimistic learning approach that prevents policy divergence.

\begin{theorem} 
\label{theorem}
 (Convergence) According to the Banach fixed point theorem~\cite{smart1980fixed}, the Bellman backup operator $\mathcal{T}^{\pi}$ is a contraction mapping, ensuring the convergence of MORAL. 
    
\end{theorem}
% \noindent \textbf{Theorem 1.} (Convergence) 

% For a comprehensive analysis of theorem proofs on MORAL, please refer to Appendix D. 

% \textbf{Proof.}
\begin{proof}
Based the Banach Fixed Point Theorem~\cite{smart1980fixed}, we assert that if the Bellman operators act as compression mappings in metric space $(X, d)$, convergence of the approach is assured.

First, we define the metric space $(X,d)$. The metric $d$ is characterized using the $L-\infty$, as elucidated in~\cite{bellemare2017distributional}:
\begin{eqnarray}
    \left|\left|F\right|\right|_{\infty}=\max_{i\in\left[0,|F|\right]}\left|F_{i}\right|.
\end{eqnarray}

Next, we introduce two independent Q-functions, $Q_{1}(s,a)$ and $Q_{2}(s,a)$, to facilitate the proof of Theorem~\ref{theorem}. The analysis unfolds as follows:
\begin{eqnarray}
\hspace{-4.7mm}
\fontsize{8.5}{1}
\begin{aligned}
    & ||\mathcal{T}^{\pi}Q_{1} - \mathcal{T}^{\pi}Q_{2} ||_{\infty} 
    % \\
    % & = \gamma\max_{s,a} \left| r_1(s,a)-d_1(s,a) + \gamma \mathbb{E}_{\mathbb{P}_{\hat{T}}^{N}}\left[ \underset{\mathcal{C}}{{\min}_{k}}  \mathbb{E}_{\pi}\left[Q_1(s',a')\right] \right]
    % \right|
    \\
    & = \gamma\max_{s,a} \left|\mathbb{E}_{\mathbb{P}_{\hat{T}}^{N}}\left[ \underset{\mathcal{C}}{{\min}_{k}}  \mathbb{E}_{\pi}\left[Q_{1}(s',a')\right] 
    \right] - \mathbb{E}_{\mathbb{P}_{\hat{T}}^{N}}\left[ \underset{\mathcal{C}}{{\min}_{k}}  \mathbb{E}_{\pi}\left[Q_{2}{(s',a')}\right]
    \right]
    \right|
    \\ 
    & = \gamma\max_{s,a} \left|\mathbb{E}_{\mathbb{P}_{\hat{T}}^{N}}\left[ \underset{\mathcal{C}}{{\min}_{k}}  \mathbb{E}_{\pi}\left[Q_{1}(s',a')\right] - \underset{\mathcal{C}}{{\min}_{k}}  \mathbb{E}_{\pi}\left[Q_{2}(s',a')\right]
    \right] \right|
    \\
     & \le \gamma\max_{s,a} \left(\mathbb{E}_{\mathbb{P}_{\hat{T}}^{N}}\left| \underset{\mathcal{C}}{{\min}_{k}}  \mathbb{E}_{\pi}\left[Q_{1}(s',a')\right] - \underset{\mathcal{C}}{{\min}_{k}}  \mathbb{E}_{\pi}\left[Q_{2}(s',a')\right]
        \right| \right)
    \\
    & \le \gamma\max_{s,a} \left(\mathbb{E}_{\mathbb{P}_{\hat{T}}^{N}}\left[ \underset{\mathcal{C}}{\max} \left|  \mathbb{E}_{\pi}\left[Q_{1}(s',a') - Q_{2}(s',a')
        \right] \right| \right] \right)
    \\
    & \le \gamma\max_{s,a} \left(\mathbb{E}_{\mathbb{P}_{\hat{T}}^{N}} || Q_{1} - Q_{2} ||_{\infty} \right)
    \\
    & \le \gamma||Q_{1}-Q_{2}||_{\infty},
    \end{aligned}
    \label{long}
\end{eqnarray}
where $\gamma \in [0,1]$. For any state $(s,a)$, we can derive:

\begin{equation}
||\mathcal{T}^{\pi}Q_{1} - \mathcal{T}^{\pi}Q_{2} ||_{\infty} \le \gamma ||Q_{1} - Q_{2}||_{\infty}.
\end{equation}

Therefore, the Bellman backup operator $\mathcal{T}^{\pi}$ acts as a compression map in $(\mathbb{R}^{|S|},L_{\infty})$. By the Banach Fixed Point Theorem~\cite{smart1980fixed}, this implies that $\mathcal{T}^{\pi}$ is a contraction mapping, which ensures policy convergence. Consequently, the convergence of the MORAL is established.
\end{proof}

\begin{lemma}
\label{lemma}
    The inequality relationship in Eq.~\ref{long} is denoted as follows:
\begin{eqnarray*}
\begin{aligned}
    \left| 
    \underset{\mathcal{C}}{{\min}_{k}} \mathbb{E}_{\pi}\left[Q_{1}(s',a')\right] - \underset{\mathcal{C}}{{\min}_{k}}  \mathbb{E}_{\pi}\left[Q_{2}(s',a')\right]
        \right| 
        \\
        \le 
        \underset{\mathcal{C}}{\max} 
        \left|  
        \mathbb{E}_{\pi}\left[Q_{1}(s',a') - Q_{2}(s',a')
        \right] \right|.      
\end{aligned}
\end{eqnarray*}

\end{lemma}
% \noindent \textbf{Lemma 1.}
\begin{proof}
The proof is conducted by considering two cases.

\textbf{Case 1:}
\begin{eqnarray}
\begin{aligned}
     & \underset{\mathcal{C}}{{\min}_{k}} \mathbb{E}_{\pi}\left[Q_{1}(s',a')\right] - \underset{\mathcal{C}}{{\min}_{k}}  \mathbb{E}_{\pi}\left[Q_{2}(s',a')\right]
     \\
     & \ge \underset{\mathcal{C}}{{\min}}(\mathbb{E}_{\pi}[Q_{1}(s',a')-Q_{2}(s',a')])
     \\
     & \ge \underset{\mathcal{C}}{{\min}}(-\left|  
        \mathbb{E}_{\pi}\left[Q_{1}(s',a') - Q_{2}(s',a')
        \right] \right|)
     \\
     & = -\underset{\mathcal{C}}{\max} 
        \left|  
        \mathbb{E}_{\pi}\left[Q_{1}(s',a') - Q_{2}(s',a')
        \right] \right|.  
    \end{aligned}
    \label{eq2}
\end{eqnarray}

\textbf{Case 2:}
\begin{eqnarray}
\begin{aligned}
     & \underset{\mathcal{C}}{{\min}_{k}} \mathbb{E}_{\pi}\left[Q_{1}(s',a')\right] - \underset{\mathcal{C}}{{\min}_{k}}  \mathbb{E}_{\pi}\left[Q_{2}(s',a')\right]
     \\
     & \le \underset{\mathcal{C}}{\max} 
        \left(  
        \mathbb{E}_{\pi}\left[Q_{1}(s',a') - Q_{2}(s',a')
        \right] \right)
        \\
    & \le \underset{\mathcal{C}}{\max} 
        \left|  
        \mathbb{E}_{\pi}\left[Q_{1}(s',a') - Q_{2}(s',a')
        \right] \right|. 
    \end{aligned}
    \label{eq3}
\end{eqnarray}

Combining these two cases of Eq.~\ref{eq2} and Eq.~\ref{eq3}, we can obtain the Lemma~\ref{lemma}. Based on Lemma~\ref{lemma}, we can prove the inequality relationship in Eq.~\ref{long} of Theorem~\ref{theorem}.
\end{proof}
% \begin{eqnarray}
% \begin{aligned}
%     \left| 
%     \underset{\mathcal{C}}{{\min}_{k}} \mathbb{E}_{\pi}\left[Q_{1}(s',a')\right] - \underset{\mathcal{C}}{{\min}_{k}}  \mathbb{E}_{\pi}\left[Q_{2}(s',a')\right]
%         \right| 
%         \\
%         \le 
%         \underset{\mathcal{C}}{\max} 
%         \left|  
%         \mathbb{E}_{\pi}\left[Q_{1}(s',a') - Q_{2}(s',a')
%         \right] \right|.  
% \end{aligned}
% \end{eqnarray}

\subsection{Practical Implementation}
\label{sub:practical}
We describe the practical implementation of MORAL driven by abovementioned analysis. Algorithm \ref{alg:algorithm} of MORAL is listed below. 
MORAL replaces the fixed horizon rollout process with an alternating sampling by adversarial data augmentation. Continuous adversarial one-step sampling is executed by dynamically updating adversarial dataset. Offline policy optimization is accomplished through an actor-critic framework. 
Ensemble models are trained on the offline datasets through a supervised learning mode, which efficiently mines state-action transitions from these offline datasets. 

We employ a neural network to model environment dynamics, generating a Gaussian distribution that predicts the next state and associated reward: $\hat{T}_{ \phi}(s_{t+1}|s_{t},a_{t})=\mathcal{N}(\mu_{\phi}(s_{t},a_{t}),{\textstyle \sum_{\phi}(s_{t},a_{t})})$. Our approach involves training an extensive ensemble of $N$ dynamic models: $\{\hat{T}^{i}_{\phi}=\mathcal{N}(\mu^{i}_{\phi},{\textstyle \sum^{i}_{\phi}})\}^{N}_{i=1}$, with each model independently trained using maximum likelihood estimation (Line 1).
After obtaining ensemble models, we sample states from the dataset for the initialization of the game and initial parameters of actor-critic networks for two players randomly (Lines 2-4). For each epoch, the model is rolled out to explore based on the initialized state. Within the exploration phase, we implement biased sampling using an adversarial game mechanism. 
We choose the state from $\mathcal{C}$ and add a new record to the $\mathcal{D}_{\mathrm{adv}}$. When the states are terminal, MORAL executes random sampling from $\mathcal{D}_{\mathrm{env}}$ to replace them. 
The differential factor is introduced through a modified reward normalization process. We update the policy by minimizing the Bellman residual of both players in actor-critic mode. The algorithm alternates between executing adversarial data augmentation and updating the function approximators (Lines 5-10). Finally, we obtain the optimized target policy $\pi^{\theta}$ (Line 11). 

% After obtaining $N$ ensemble models, we initialize the agent state and episodic memory from the dataset. For each epoch, the model is rolled out to explore the OOD region based on the initialize state. During the rollout exploration process, the energy value is introduced into the reward normalization as an uncertainty term. Afterwards, transitions acquired by rollout are added to episodic memory. The above process is a round of rollout exploration of the OOD region. Finally, SAC~\cite{Haarnoja2018} algorithm is used to update policy $\pi$ with $D_{env}$  and $D_{em}$ until convergence. Meanwhile, the episodic memory is updated with a memory update frequency $p$. 

\begin{algorithm}[]
    \caption{Model-Based Offline Reinforcement Learning with Adversarial Data Augmentation}
    \label{alg:algorithm}
    \textbf{Input}: Dataset $\mathcal{D}_{\rm{env}}$, size of ensemble models $N$, epoch length $H$, adversarial data $\mathcal{D}_{\rm{adv}}$
    \begin{algorithmic}[1]
     \STATE Train an ensemble of $N$ dynamic models based on the dataset $\mathcal{D}_{\rm{env}}$, represented as $\{\hat{T}^{i}_{\phi}=\mathcal{N}(\mu^{i}_{\phi},{\textstyle \sum^{i}_{\phi}})\}^{N}_{i=1}$\
     \STATE Initialize critic network $V^{\omega'}$ and actor network $\pi^{\vartheta}$ with random parameters $\omega'$, $\vartheta$
     \STATE Initialize target critic network $V^{\omega}$ and target actor network $\pi^{\theta}$ with parameters $\omega \gets \omega'$, $\theta \gets \vartheta$\
     \STATE Sample states from $\mathcal{D}_{\rm{env}}$  for the initialization of adversarial data $\mathcal{D}_{\rm{adv}}$ 
     \FOR {$epoch= 1$,$\cdots$, $H$} 
     \STATE  \textbf{Primary player:} Sample actions from $\pi^{\theta }{(a|s)}$ with state from $\mathcal{D}_{\rm{adv}}$ and construct candidate transition set $\{s'_{i}\}_{i=1}^{|\mathcal{C}|}$ according to Eq.~\ref{equ:C}
     \STATE Update $\omega'$ and $\vartheta$ with a batch of transitions from $\{(s_{i},a_{i},s'_{i})\}_{i=1}^{|\mathcal{D}_{\rm{adv}}|}$ and $\mathcal{D}_{\rm{env}}$
     \STATE \textbf{Second player:} Choose state from $\mathcal{C}$ according to Eq.~\ref{equ:asec} and update $\mathcal{D}_{\rm{adv}}$
     \STATE Update parameters of target critic network $\omega$ and target actor network $\theta$ with $\omega'$ and $\vartheta$
     \ENDFOR
     \RETURN $\pi^{\theta}$
    \end{algorithmic}
\end{algorithm}

% \clearpage

\section{Experiments}
\label{sec:experiments}
In this section, we aim to answer the following questions: (i) How does MORAL perform compared to model-based offline RL methods on the standard offline benchmark across diverse tasks? 
(ii) How does this adversarial game-driven data augmentation affect the policy performance?
(iii) Does MORAL demonstrate robust performance throughout the learning process?
(iv) What are the effects of the hyperparameters in MORAL? 
(v) How to analyze the important components in MORAL through ablation studies?
% \fan{Should we one more question iv here about the ablation studies?}

\subsection{Experimental Datasets and Settings}
We use the simulated environment of MuJoCo~\cite{Todorov2012} and D4RL datasets~\cite{Fu2020} for experiments. 
We select three widely used MuJoCo environments in offline RL for our experiments~\cite{mopo,chen2023offline}: halfcheetah, hopper, and walker2d. Each environment encompasses five distinct pre-collected datasets. 
These datasets are generated as follows:
 \begin{itemize}
    \item \textbf{random}: Roll out a random policy for 1M steps as the random dataset.
    \item \textbf{medium}: Train the policy using SAC~\cite{Haarnoja2018}, then roll it out for 1M steps to collect the transitions as the medium dataset.
    \item \textbf{mixed}: SAC is used to train the policy until a certain performance threshold is reached and roll out to construct the mixed dataset.
    \item \textbf{med-expert}: Combine 1M push samples from fully trained policy with another 1M samples from partially trained policy as medium-expert (med-expert) dataset.
    \item \textbf{expert}: Utilize 1M push samples from the fully trained policy as the expert dataset.
\end{itemize}

% We provide the detailed experimental hyperparameters of MORAL in Table \ref{tab:table1}. 
Within all domains, a corpus of 100 ensemble models is trained on the offline datasets. Each model is parameterized through a 4-layer feed-forward neural network featuring 256 hidden units. The discount factor is set to 0.99, $\alpha=0.5$, and the epoch length is uniformly established at 1000. The sampling hyperparameters $N'=10$ and $k=2$. All experiments of this approach are implemented on 2 Intel(R) Xeon(R) Gold 6444Y and 4 NVIDIA RTX A6000 GPUs. 
% The source code is released at 
% \href{https://github.com/HYeCao/MORAL}{https://github.com/HYeCao/MORAL}
% We employ the D4RL~\cite{Fu2020} benchmark, which is built upon Gym-MuJoCo simulator~\cite{Todorov2012}, as the experimental datasets. The datasets include 3 environments (HalfCheetah, Hopper, and Walker2d) and 5 dataset types (random, medium, mixed, med-expert, and expert). Hence, a total of 15 sub-datasets are curated.

% The source code is released at 
% \href{https://github.com/HYeCao/MORAL}{https://github.com/HYeCao/MORAL}.

% \begin{table}[t]
%  \caption{Hyperparameter settings in MORAL.}
%    \centering
%     \begin{tabular}{ll}\toprule
%   \textbf{Hyperparameter}  &  \textbf{Value}  \\ \midrule
%   Number of epochs for training dynamics model &  1000    \\ 
%   Batch size in dynamics training & 256
%   \\
%   Number of ensemble models ($N$) & 100
%   \\
%   Number of hidden layers of dynamics model & 4
%   \\
%   Discount factor for reward ($\gamma$) & 0.99
%   \\
%   Penalty coefficient of the differential factor ($\alpha$) & 0.5
%   \\
%   Number of sampling transition in MZG ($N'$) & 10
%   \\
%   Order of sampling transition for second player ($k$) & 2
%   \\
%   Learning rate of dynamics & 0.0001
%   \\
%   Batch size of training the primary player policy & 128
%   \\
%   Batch size of training the second player policy & 128
%   \\
%   Actor learning rate & 0.00003
%   \\
%   Critic learning rate & 0.0003
%   \\
%   State sample size for estimating expectation & 10
%   \\
%   Action sample size for estimating expectation & 20
%   \\ 
%   Number of steps for policy optimization & 1000
%   \\ \bottomrule
% \end{tabular}
%     \label{tab:table1}
% \end{table}

\begin{table*}[t]
 \caption{Performance comparison on the D4RL datasets. Each score is the normalized score~\cite{Fu2020} over 12 random seeds. We evaluate PMDB with public codes and other values were taken from papers. We bold the highest scores and underline the sub-optimal scores. ‘$\pm$’ is the standard deviation. $\bullet$ indicates MORAL is statistically superior to PMDB (pairwise \textit{t}-test at $95\%$ confidence interval)}. 
 % \vspace{-2mm}
\renewcommand{\arraystretch}{1.1}
\setlength{\tabcolsep}{3.7pt} % 设置列间距为4pt
    \centering 
    \begin{tabular}{ccccccccccc}
    \toprule
        Environment & Dataset & MORAL & PMDB~\cite{pmdb}
        & PBRL~\cite{pbrl} & MOPO~\cite{mopo} & TD3+BC~\cite{td3bc} & ARMOR~\cite{armor} & PBRS~\cite{pbrs} & CROP~\cite{crop} & RAMBO~\cite{rambo} 
        \\ \hline
        HalfCheetah & random & 35.6$\pm$0.6 & 35.1$\pm$0.2 & 11.0$\pm$5.8 & \underline{35.9}  & 10.2  & - & - & 33.3$\pm$2.8 & \textbf{40.0$\pm$2.3} 
        \\
        HalfCheetah & medium & \underline{76.3$\pm$0.8} & 75.6$\pm$0.3$\bullet$   & 57.9$\pm$1.5 & 42.3 & 42.8 & 54.2$\pm$2.4 & 58.2$\pm$0.5 & 68.1$\pm$0.8 & \textbf{77.6$\pm$1.5} 
        \\ 
        HalfCheetah & mixed & \textbf{70.0$\pm$0.4} & \underline{69.7$\pm$1.1}  & 45.1$\pm$8.0 & 53.1 & 43.3 & 50.5$\pm$0.9 & 49.4$\pm$0.2 & 64.9$\pm$1.1 & 68.9$\pm$2.3
        \\ 
        HalfCheetah & med-expert & \textbf{109.2$\pm$2.4} & 95.4$\pm$1.8$\bullet$ & 92.3$\pm$1.1 & 63.3 & \underline{95.9} & 93.5$\pm$0.5 & 66.5$\pm$7.8 & 91.1$\pm$1.1 & 93.7$\pm$10.5  
        \\
         HalfCheetah & expert & \underline{102.9$\pm$0.7} & \textbf{104.7$\pm$1.5} & 92.4$\pm$1.7 & 81.3 & 96.7 & 93.9 & - & - & - 
        \\ \hline
        Hopper & random & \textbf{35.7$\pm$9.2}  & 25.6$\pm$8.2$\bullet$ & \underline{26.8$\pm$9.3} & 16.7 & 11.0 & - & - & 19.1$\pm$11.0 & 21.6$\pm$8.0 
        \\ 
        Hopper & medium & \textbf{107.4$\pm$0.8} & \underline{106.8$\pm$0.2} & 75.3$\pm$31.2 & 28.0 & 98.5 & 101.4$\pm$0.3 & 75.4$\pm$1.8 & 100.6$\pm$3.2 & 92.8$\pm$6.0 
        \\ 
        Hopper & mixed & \textbf{105.2$\pm$0.6} & \underline{104.5$\pm$1.8} & 100.6$\pm$1.0 & 67.5 & 31.4  & 97.1$\pm$4.8  & 102.3$\pm$0.7 
        & 93.0$\pm$2.2 & 96.6$\pm$7.0 
        \\ 
        Hopper & med-expert & \underline{112.0$\pm$1.1} & 111.8$\pm$0.6  & 110.8$\pm$0.8 & 23.7 & \textbf{112.2} & 103.4$\pm$5.9 &109.4$\pm$1.7 & 96.5$\pm$10.2 & 83.3$\pm$9.1
        \\ 
        Hopper & expert & \textbf{113.2$\pm$1.2} & \underline{111.7$\pm$0.3}$\bullet$ & 110.5$\pm$0.4 & 62.5 & 107.8 & 111.6 & - & - & - 
        \\ \hline
        Walker2d & random & \textbf{21.8$\pm$0.2} & 21.5$\pm$0.4 & 8.1$\pm$4.4 & 13.6 & 1.4  & - & - &  \underline{21.6$\pm$0.8} & 11.5$\pm$10.5 
        \\ 
        Walker2d & medium & \textbf{95.2$\pm$1.7} & 86.9$\pm$1.9$\bullet$ & 89.6$\pm$0.7 & 11.8 & 79.7  & \underline{90.7$\pm$4.4} & 88.5$\pm$0.8
        & 89.7$\pm$0.8 & 86.9$\pm$2.7 
        \\ 
        Walker2d & mixed & 77.5$\pm$2.0 & 76.1$\pm$2.8$\bullet$  & 77.7$\pm$14.58 & 39.0 & 25.2 & 85.6$\pm$7.5 & \underline{88.9$\pm$1.0}  & \textbf{89.7$\pm$0.7} & 85.0$\pm$15.0  
        \\ 
        Walker2d & med-expert & \textbf{113.6$\pm$1.5} & 111.9$\pm$0.2$\bullet$ & 110.1$\pm$0.3 & 44.6 & 101.1 & \underline{112.2$\pm$1.7} & 111.7$\pm$0.4 & 109.3$\pm$0.3 & 68.3$\pm$20.6 
        \\
        Walker2d & expert & \textbf{118.8$\pm$0.7} & \underline{115.9$\pm$1.9}$\bullet$ & 108.3$\pm$0.3 & 62.4 & 110.2 & 108.1 & - & - & -  
        \\  \hline
        \multicolumn{2}{c}{\textbf{Average}} &  \textbf{86.3} & \underline{83.5} & 74.4 & 43.0 & 64.5  & - & - & - & - 
        \\
        \bottomrule
    \end{tabular}
      % \vspace{-0.2cm}%%减小图片上间隔
    \label{tab:comp}
\end{table*}

% \vspace{-3mm}
\subsection{Compared Methods}
% \fan{We could add a few sentences to give a simple taxonomy of all these baselines}
We compare MORAL against five model-based and three model-free offline RL algorithms. Comparative methods are shown as follows:
\begin{itemize}
    \item ARMOR~\cite{armor}: A model-based offline RL method by adversarial optimizing dynamic models.
    \item CROP~\cite{crop}: A model-based offline RL method with conservative reward.
    \item PBRS~\cite{pbrs}: An adaptive reward shifting method based on behavior proximity for offline RL.
    \item RAMBO~\cite{rambo}: A robust adversarial model-based offline RL method.
    \item PMDB~\cite{pmdb}: A model-based offline RL method with pessimism-modulated dynamics belief.
    \item PBRL~\cite{pbrl}:  A purely uncertainty-driven model-free offline RL algorithm.
    \item TD3+BC~\cite{td3bc}: A model-free offline RL method with minimal changes.
    \item MOPO~\cite{mopo}: A model-based offline RL method with the uncertainty return penalty.
    % \item OEMA~\cite{tkdeoffline}: A sample-efficient offline-to-online RL algorithm via Optimistic Exploration and Meta Adaptation.
    % \item CQL~\cite{Kumar2020}: A model-free offline RL method optimizes the policy with regularization.
\end{itemize}

\begin{figure}[]
  \centering
    % \vspace{-0.2cm}%%减小图片上间隔
\includegraphics[width=0.49\linewidth]{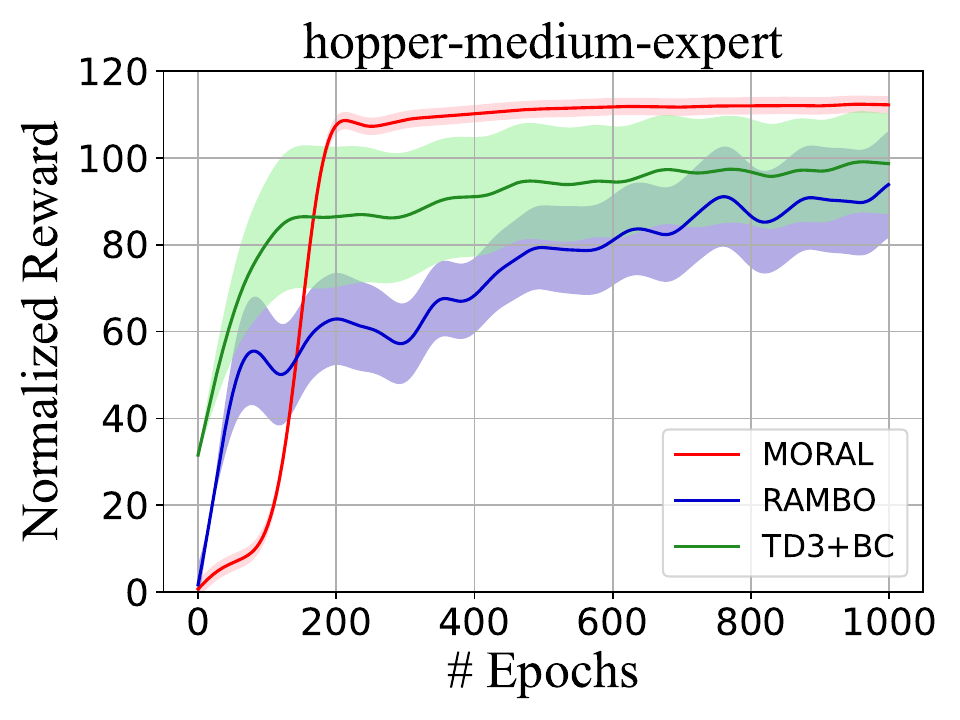}
\includegraphics[width=0.49\linewidth]{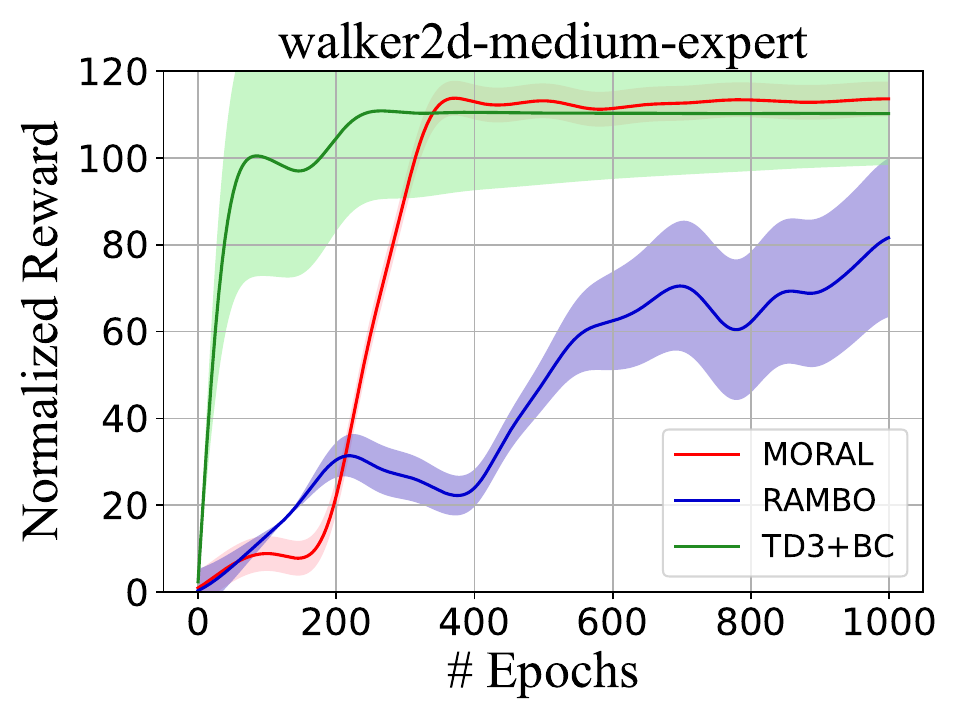}
  \caption{Learning curves on walker2d-medium-expert and hopper-medium-expert datasets. Each number is the normalized reward during training, averaged over $12$ random seeds and the shadow is the standard error.}
  % \vspace{-0.4cm}%%减小图片上间隔
  \label{fig:3}
\end{figure}

\begin{table}[h]
\renewcommand{\arraystretch}{1.1}
\setlength{\tabcolsep}{4.3pt} % 设置列间距为4pt
 \caption{Normalized scores of three expert tasks. All results are averaged over $12$ random seeds. We bold the highest scores across all methods.}
    \centering
    
     % \vspace{-0.2cm}%%减小图片上间隔

% \setlength{\tabcolsep}{6pt} % 设置列间距为4pt
      \begin{tabular}{ccccc}
      \toprule
        Environment & MORAL & PMDB\cite{pmdb} & ARMOR\cite{armor} & TD3+BC\cite{td3bc} 
        \\ 
        \hline
        HalfCheetah & 102.9$\pm$0.7 & \textbf{104.7$\pm$1.5} & 93.9 & 96.7
        \\ 
        Hopper & \textbf{113.2$\pm$1.2} & 111.7$\pm$0.3 & 111.6 & 107.8
        \\ 
        Walker2d & \textbf{118.6$\pm$0.7} & 115.9$\pm$1.9 & 108.1 & 110.2
        \\ 
        \hline
        \textbf{Average} & \textbf{111.6} & 110.7 & 104.5 & 104.9 
        \\
        \bottomrule 
    \end{tabular}
   
    \label{tab:2}
    % \vspace{-3mm}
\end{table}

% \begin{figure}[h]
%   \centering
%     % \vspace{-0.5cm}%%减小图片上间隔
% \includegraphics[width=0.32\linewidth]{figs/compare_pmdb/hopper_compare.pdf}
% \includegraphics[width=0.32\linewidth]{figs/compare_pmdb/walker2d_compare.pdf}
% \includegraphics[width=0.32\linewidth]{figs/compare_pmdb/half_compare.pdf}
%   % \vspace{-0.6cm}%%减小图片上间隔
%   \caption{Learning curves on three expert datasets. Each number is the episodic reward during training, averaged over 5 random seeds and the shadow is the standard error.}
%   \label{fig:4}
%   \vspace{-3mm}
% \end{figure}
\begin{figure}[h]
  \centering
    % \vspace{-0.5cm}%%减小图片上间隔
\includegraphics[width=0.32\linewidth]{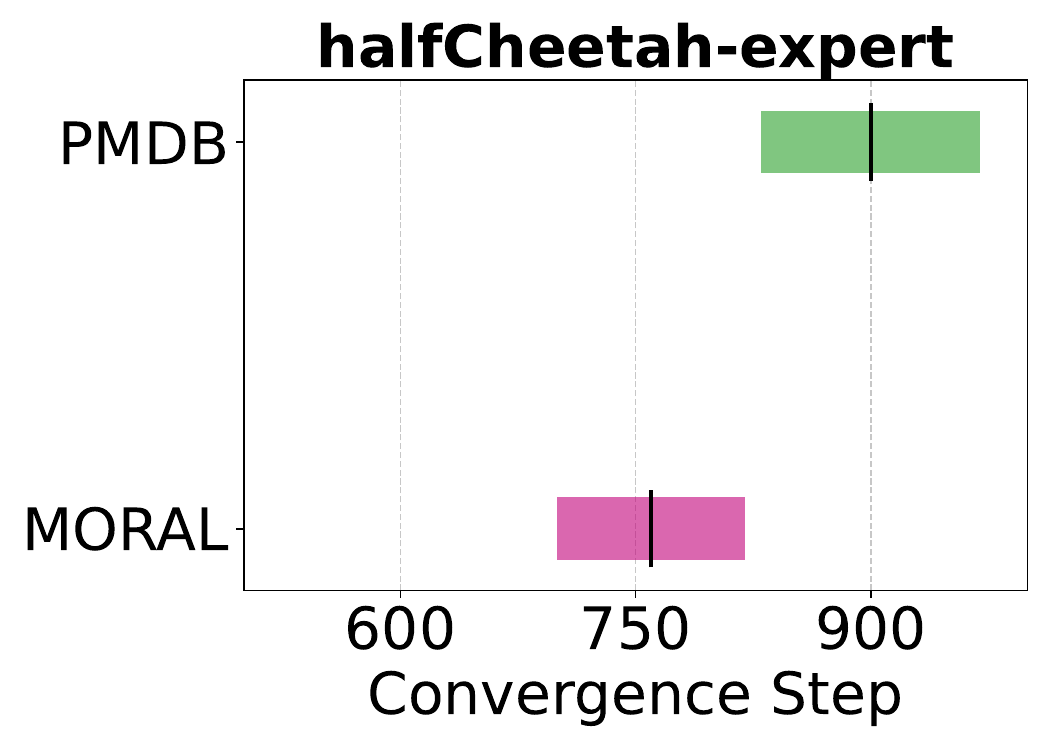}
\includegraphics[width=0.32\linewidth]{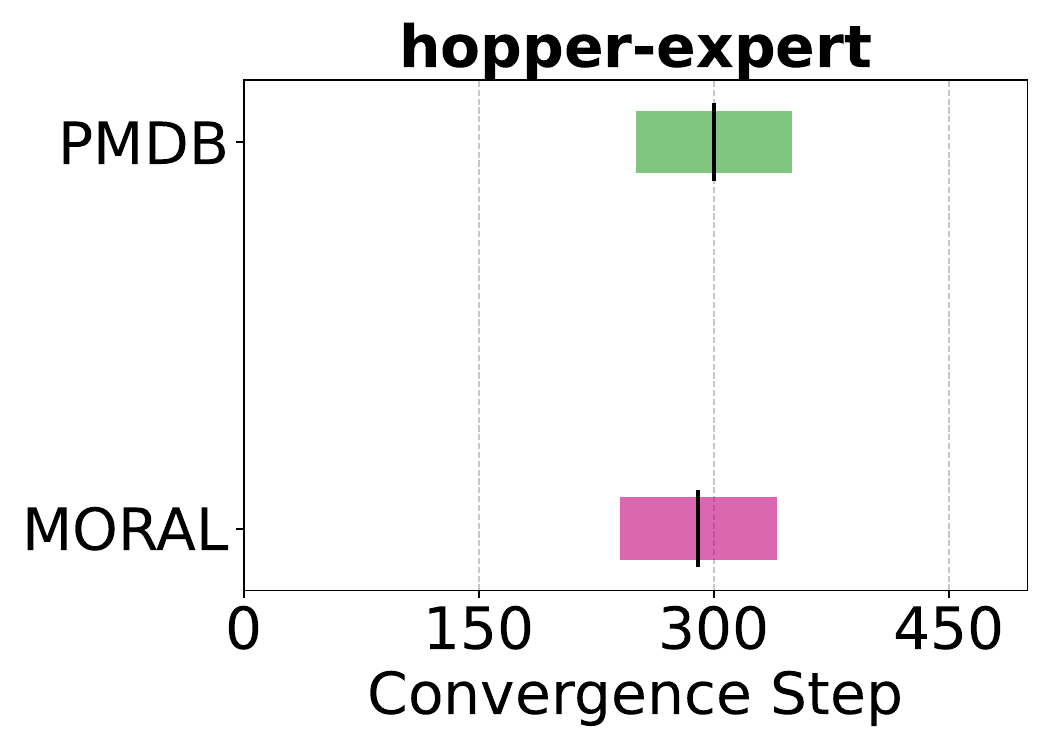}
\includegraphics[width=0.32\linewidth]{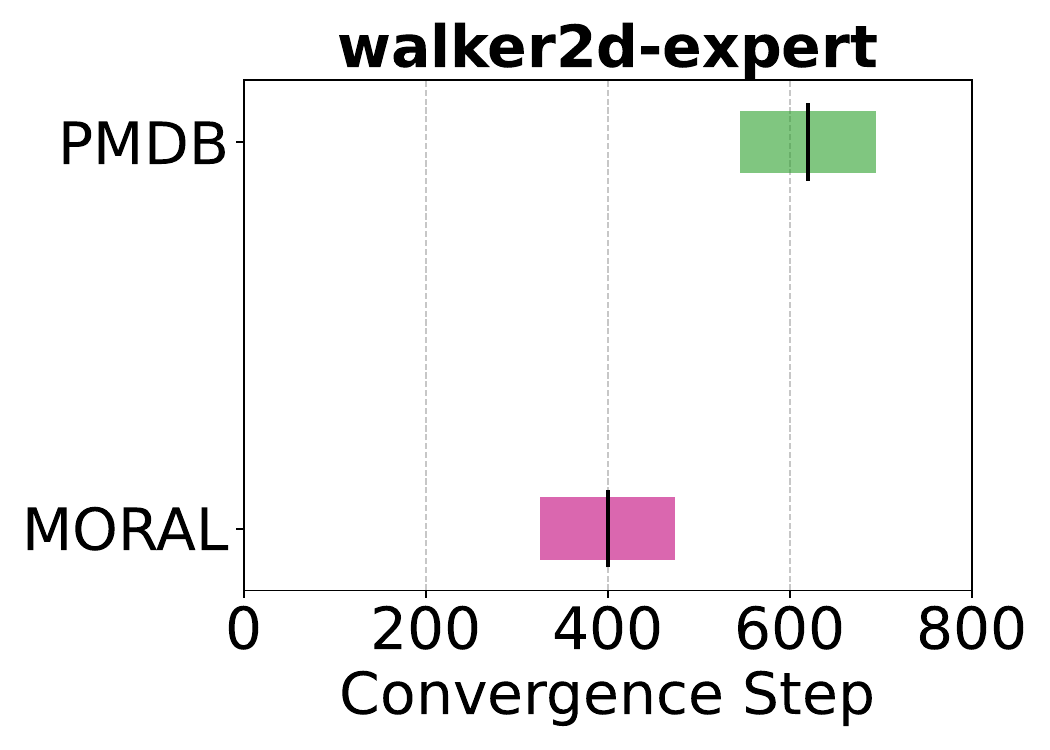}
  % \vspace{-0.6cm}%%减小图片上间隔
  \caption{Convergence step of PMDB and MORAL in $3$ expert tasks. The shaded regions are the standard deviation of each method.}
  \label{fig:conv}
  % \vspace{-3mm}
\end{figure}

\begin{figure*}[]
  \centering
  % \vspace{-0.2cm}%%减小图片上间隔
   \includegraphics[width=0.19\linewidth]{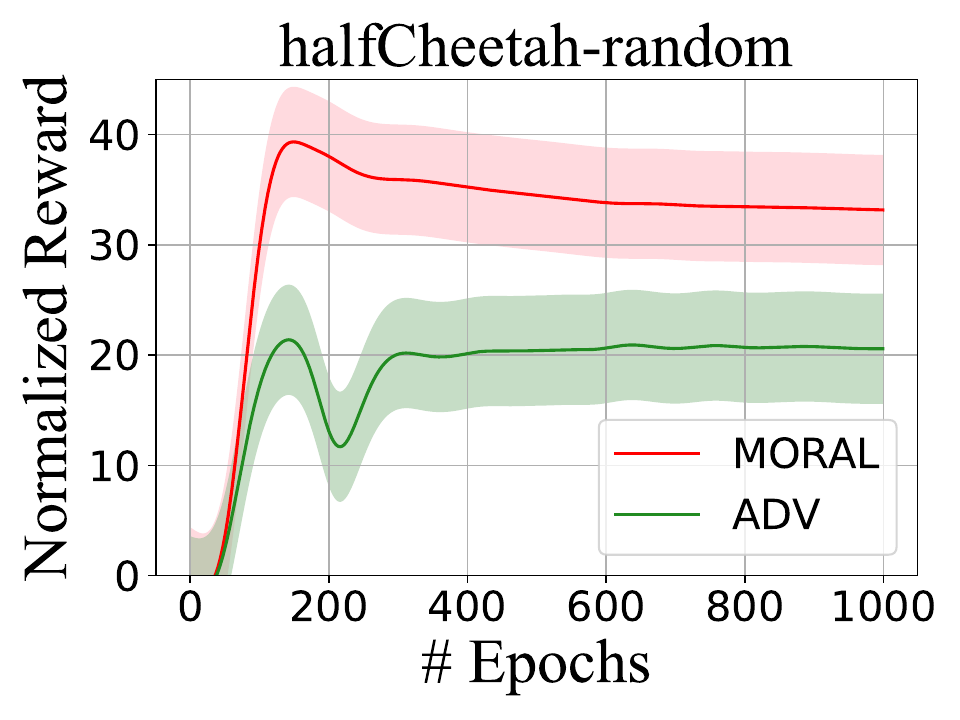}
  \includegraphics[width=0.19\linewidth]{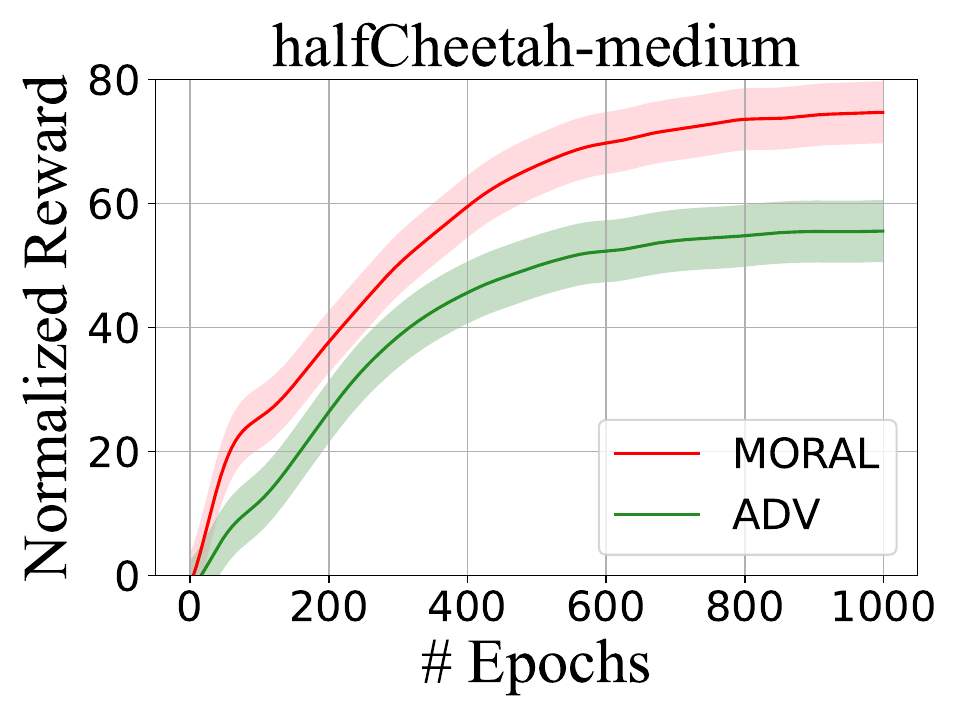}
  \includegraphics[width=0.19\linewidth]{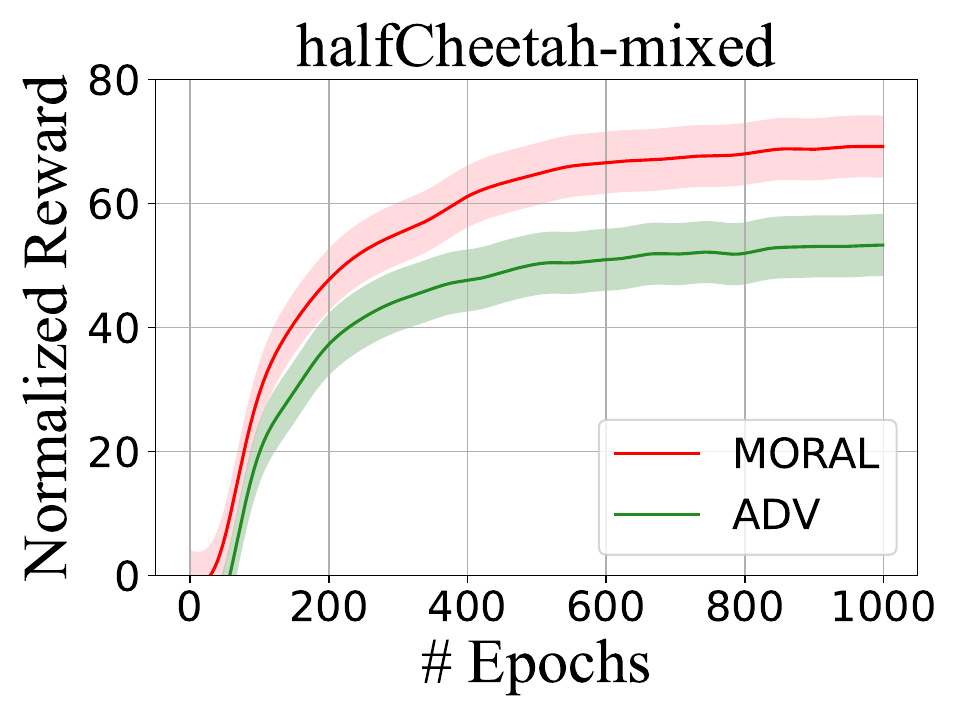}
   \includegraphics[width=0.19\linewidth]{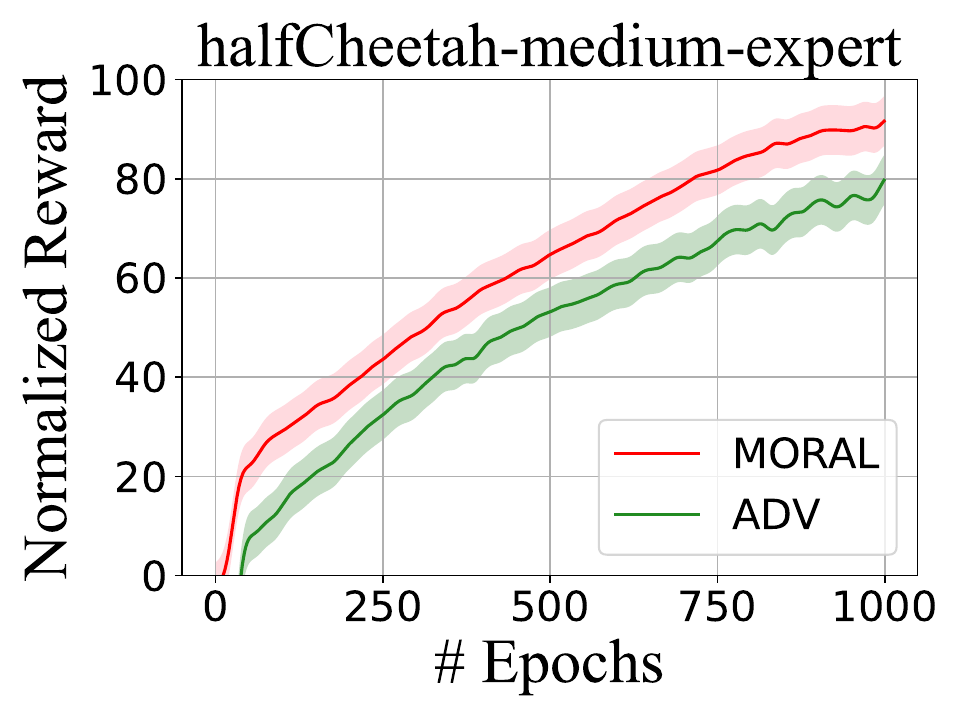}
  \includegraphics[width=0.19\linewidth]{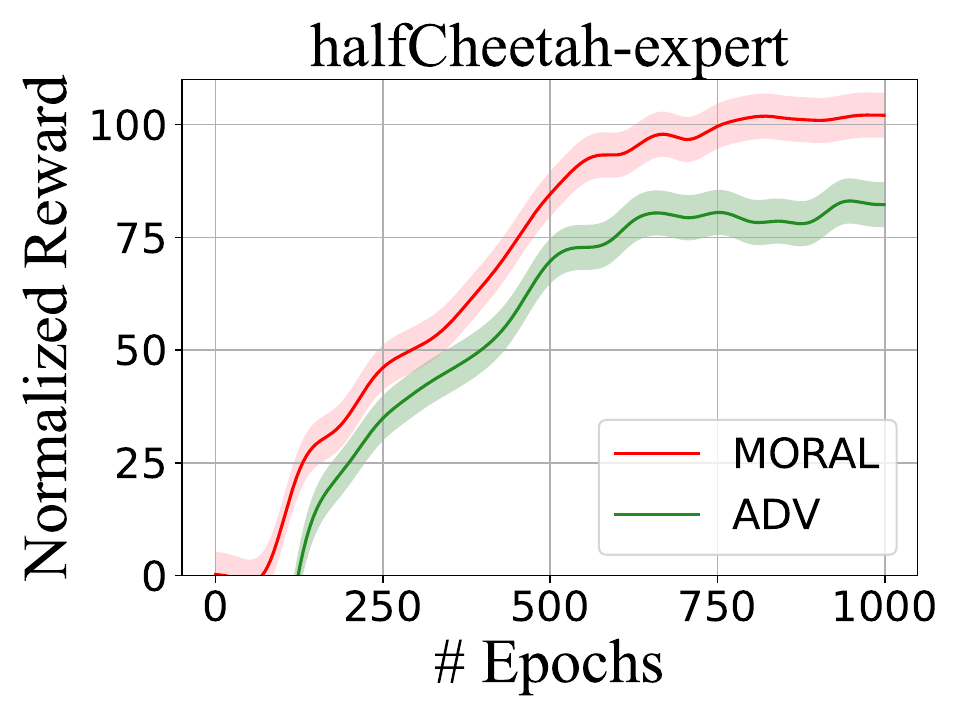}
  \vspace{2mm}
  \includegraphics[width=0.19\linewidth]{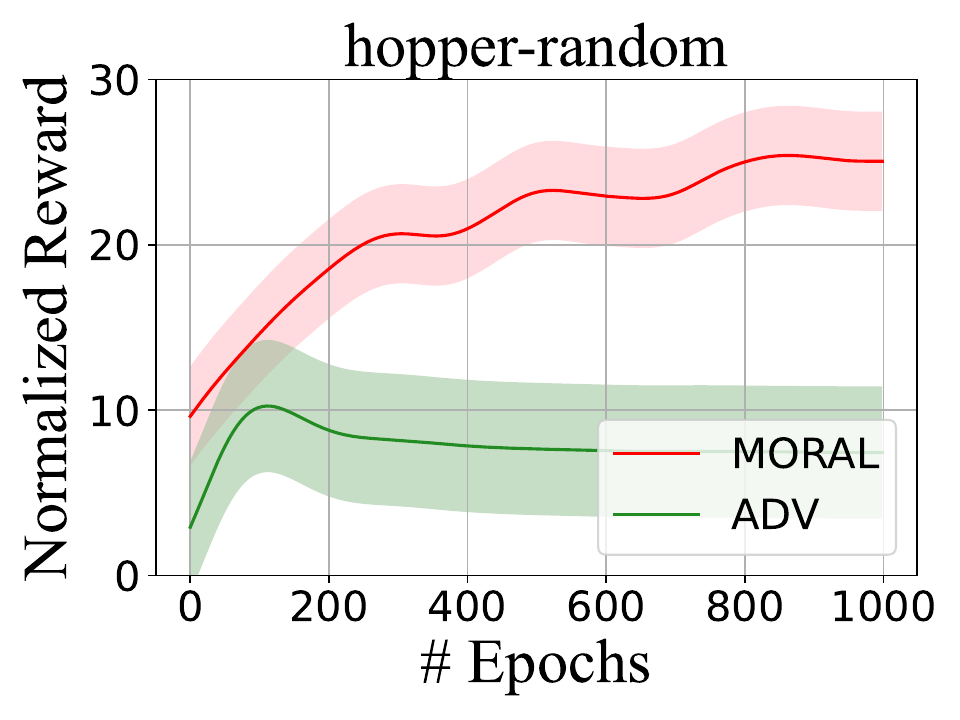}
\includegraphics[width=0.19\linewidth]{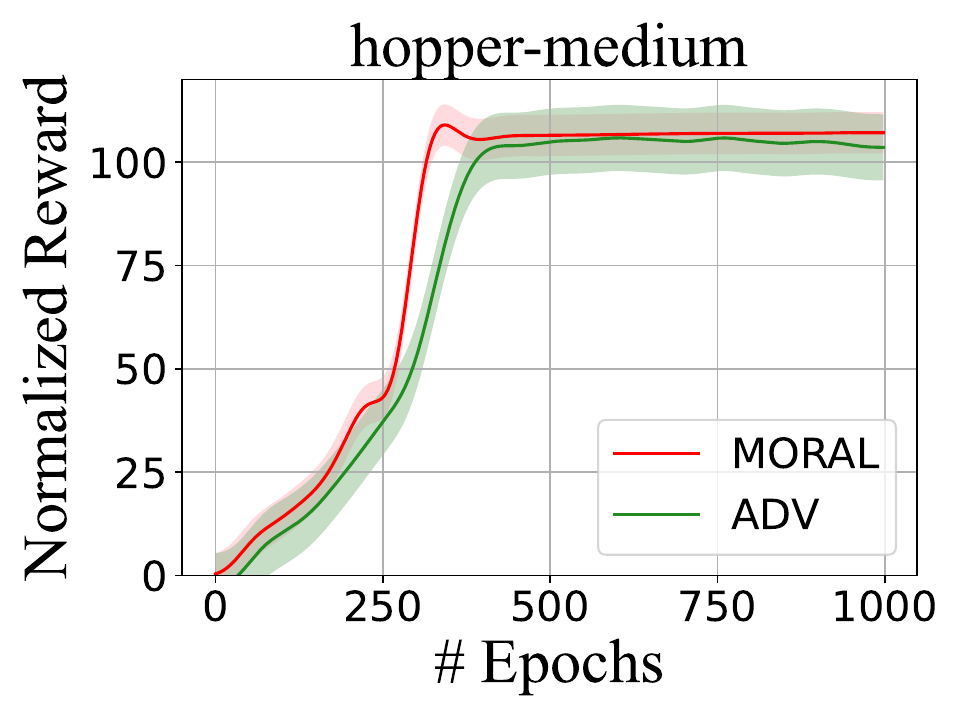} 
\includegraphics[width=0.19\linewidth]{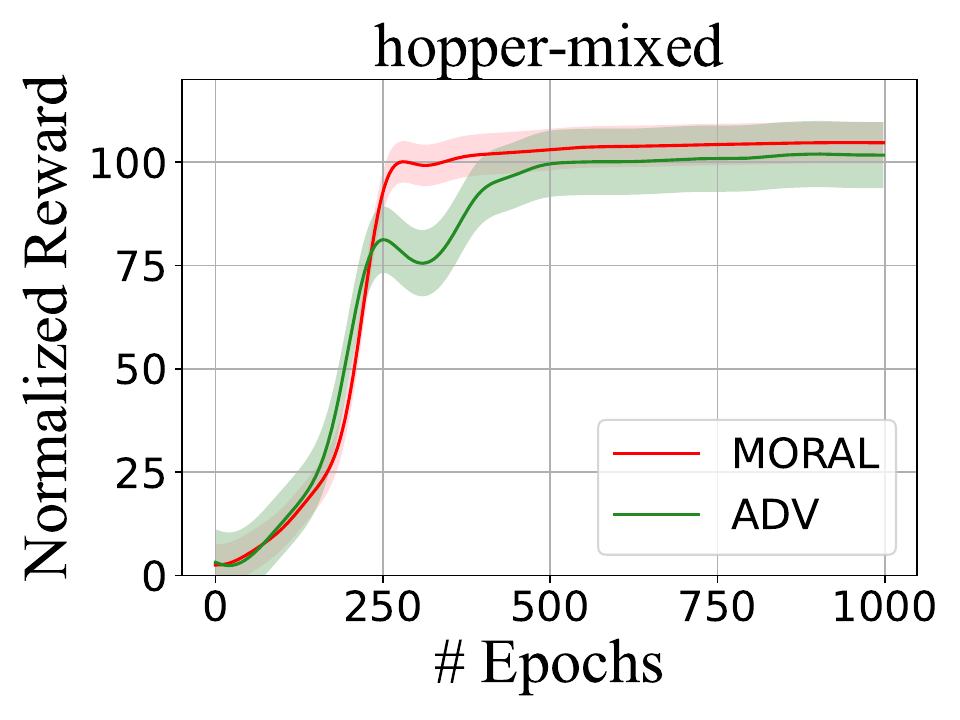}
  \includegraphics[width=0.19\linewidth]{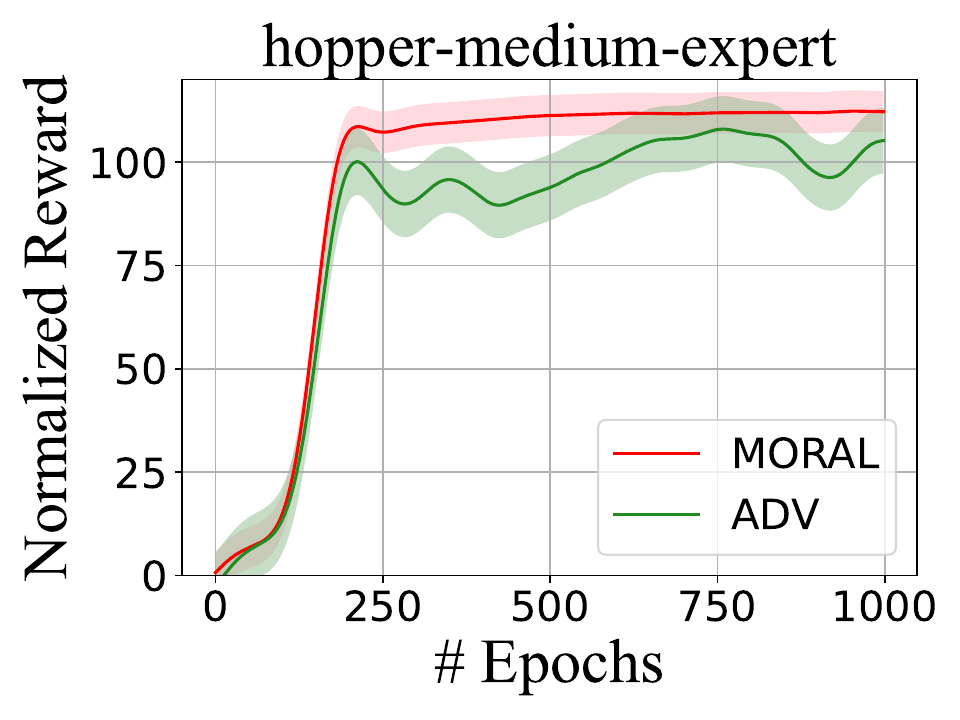}
   \includegraphics[width=0.19\linewidth]{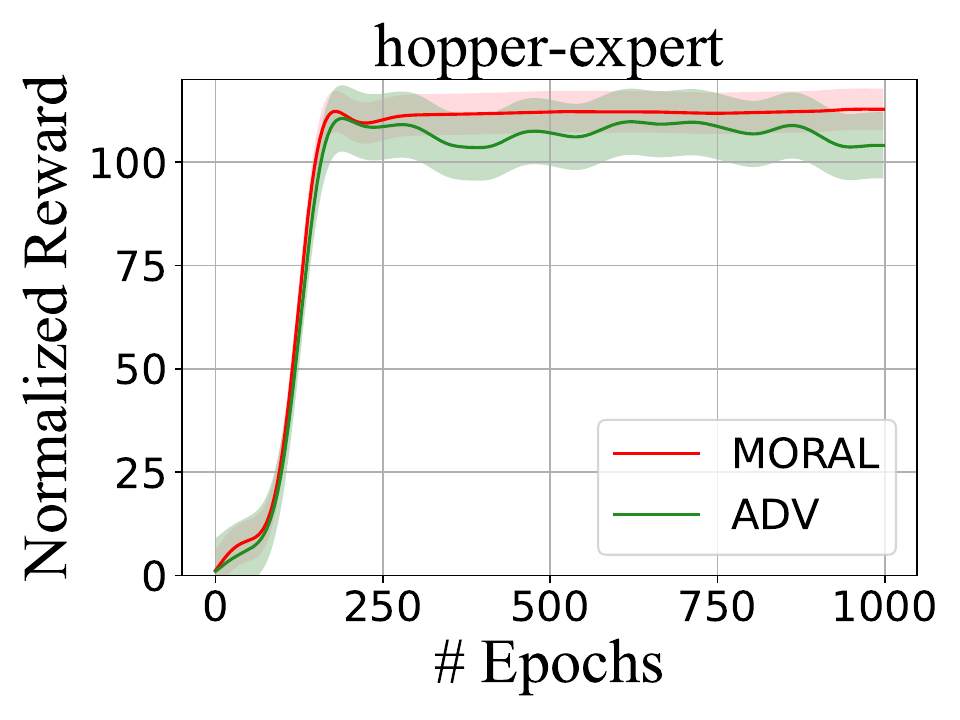}
  \vspace{2mm}
  \includegraphics[width=0.19\linewidth]{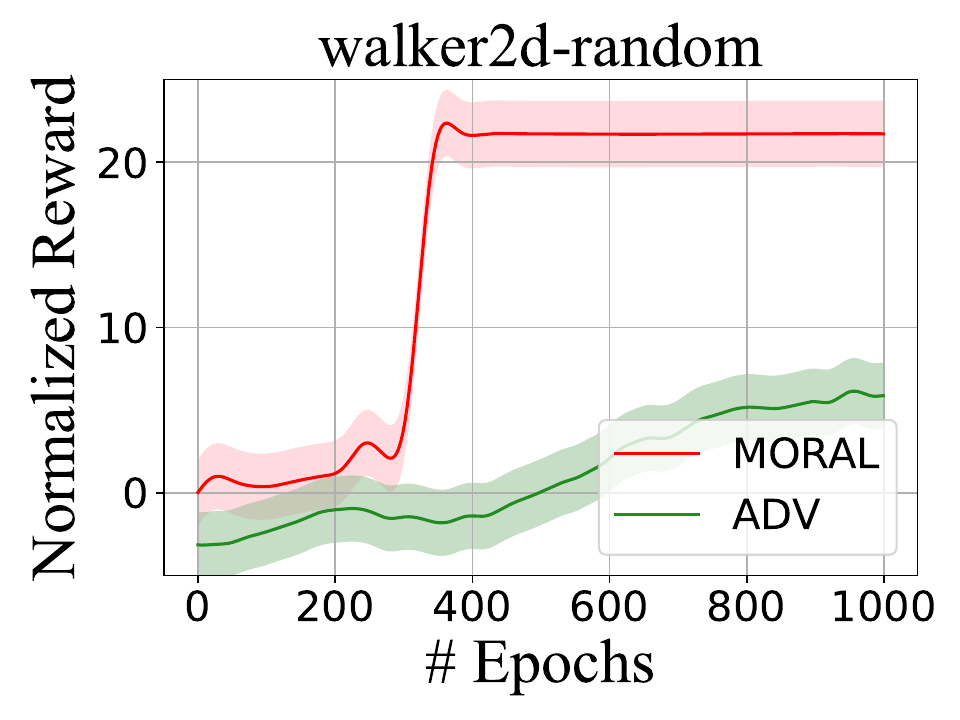}
  \includegraphics[width=0.19\linewidth]{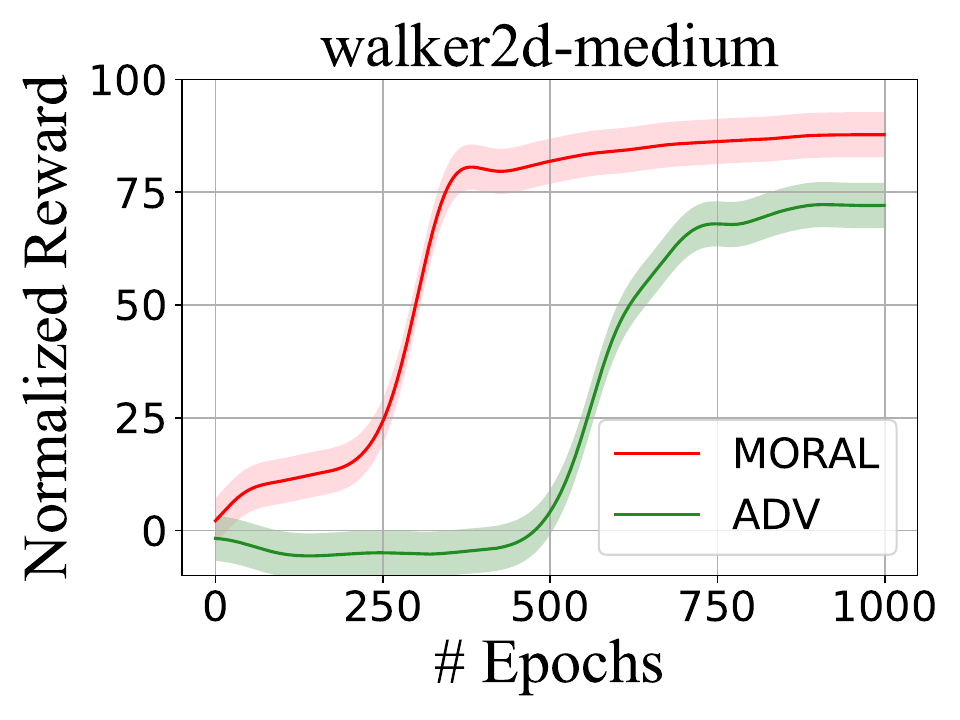}
   \includegraphics[width=0.19\linewidth]{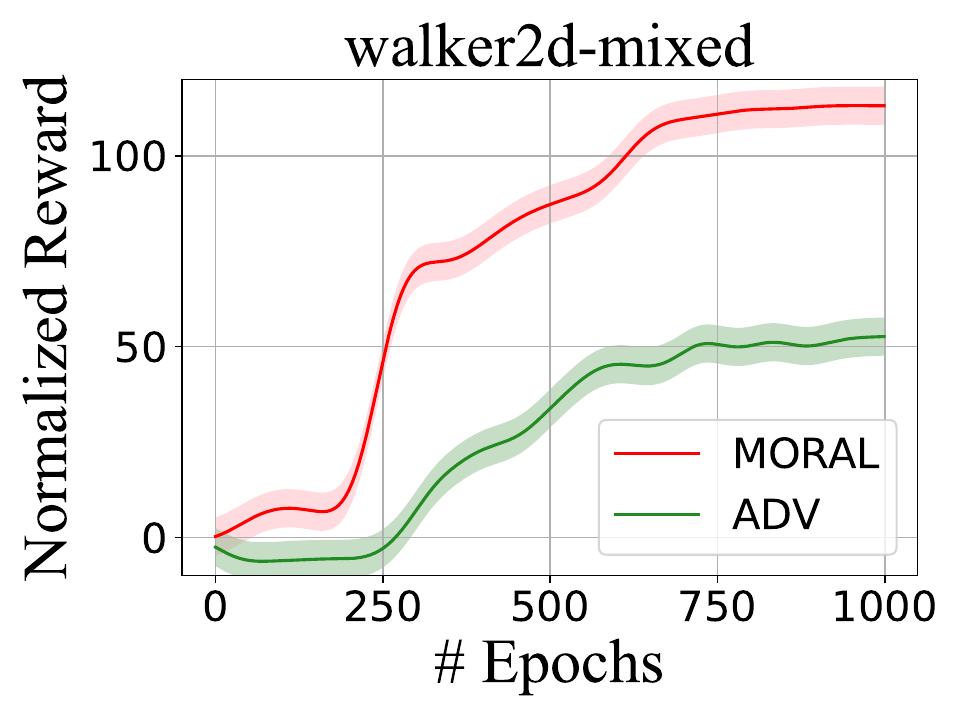}
  \includegraphics[width=0.19\linewidth]{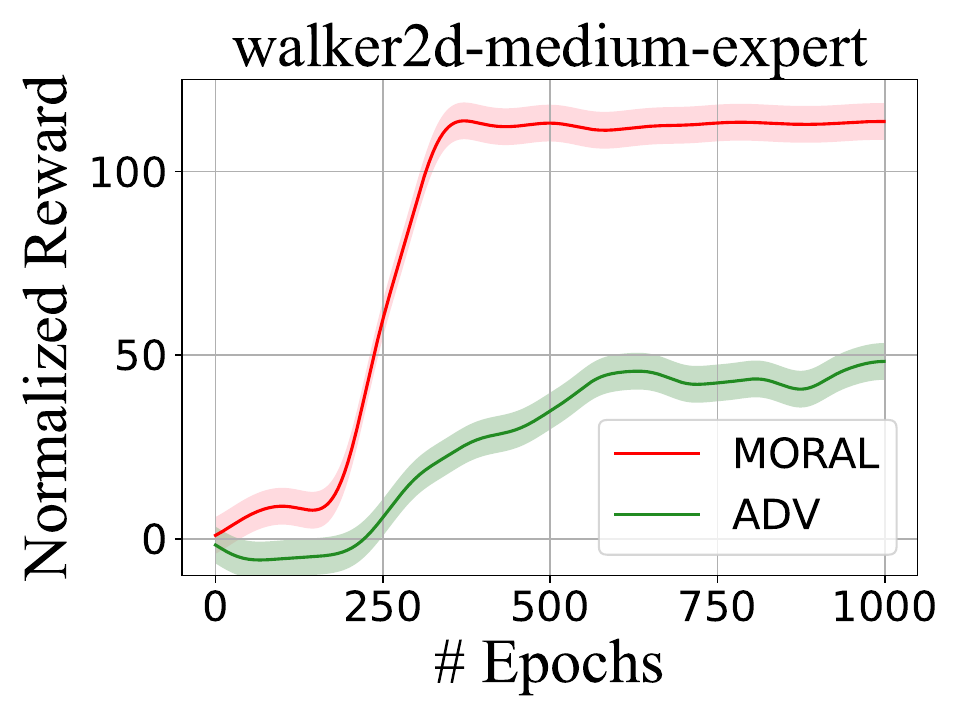}
  \includegraphics[width=0.19\linewidth]{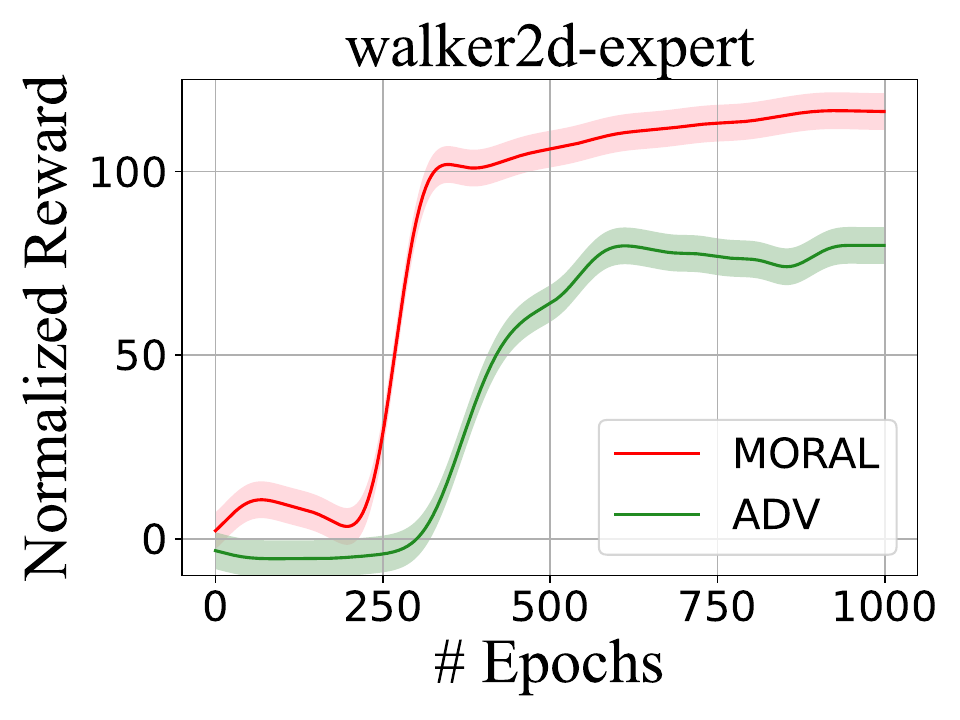}
    % \vspace{-0.5cm}%%减小图片上间隔
  \caption{Learning curves on all datasets during adversarial training. Each number is the normalized reward, averaged over $12$ random seeds and the shadow is the standard error. ADV represents the adversarial policy.}
   % \vspace{-0.5cm}%%减小图片上间隔
  \label{fig:figure2}
  % \vspace{-3mm}
\end{figure*}

% \vspace{-3mm}
\subsection{Performance Comparison}
The results of comparative experiments are shown in Table \ref{tab:comp}.  MORAL achieves 8 best results in $15$ tasks, surpassing all compared methods. Compared with model-based offline RL methods, MORAL is the best in 9 out of the 12 tasks, with the best average score of $86.3$, indicating that MORAL outperforms existing model-based methods. 
The substantial deployment of ensemble models, coupled with the adversarial biased sampling method, markedly elevates policy performance. 
Moreover, we have conducted statistical analysis by pair-wise t-test on our method with PMDB all $15$ tasks under $12$ seeds, demonstrating that MORAL achieves statistically significant improvements in $8$ out of $15$ tasks. 
More importantly, MORAL demonstrates superior performance on challenging random datasets. 
Adversarial alternating sampling has demonstrated remarkable performance in suboptimal settings by achieving data augmentation. 

By employing adversarial sampling, MORAL, leveraging the power of ensemble models, effectively mitigates policy divergence. 
The incorporation of the differential factor effectively mitigates the influence of extrapolation errors, preventing policy divergence without online exploration. MORAL addresses the problem of limited applicability while ensuring stable policy learning without requiring configuration of the rollout horizon.

% \vspace{-3mm}
\subsection{Comparison with Adversarial Methods} 
% \fan{make cells of these methods in a different color}
In order to conduct a comprehensive experimental analysis and further validate the effectiveness of this proposed adversarial approach, we choose methods ARMOR, RAMBO, and PMDB, inspired by the adversarial settings to conduct comparative experiments.

\begin{figure*}[]
  \centering
  \includegraphics[width=0.19\linewidth]{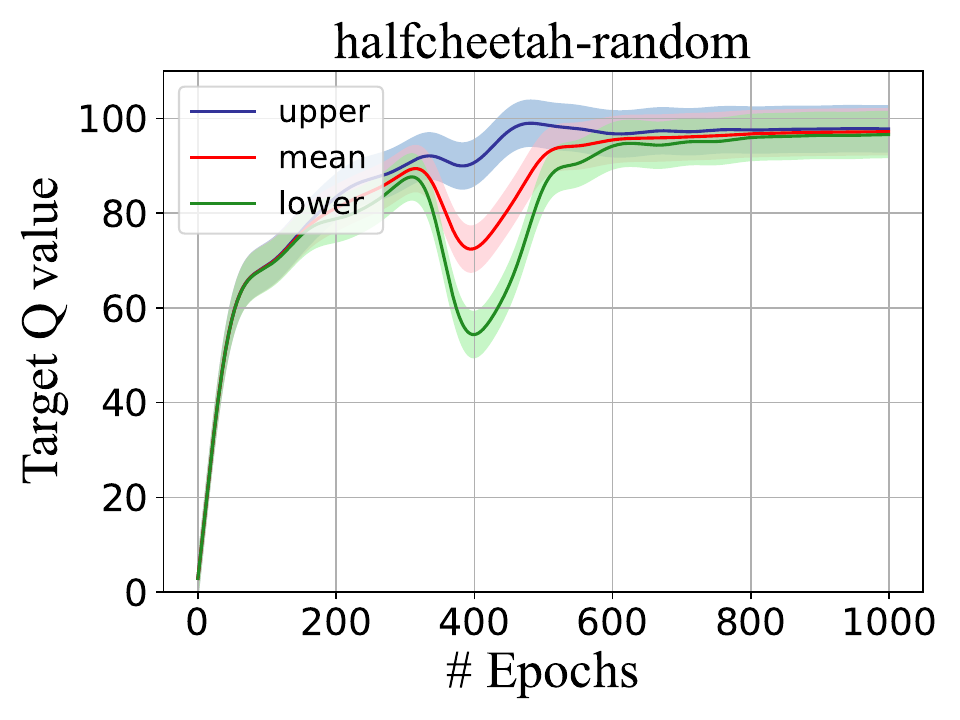}
  \includegraphics[width=0.19\linewidth]{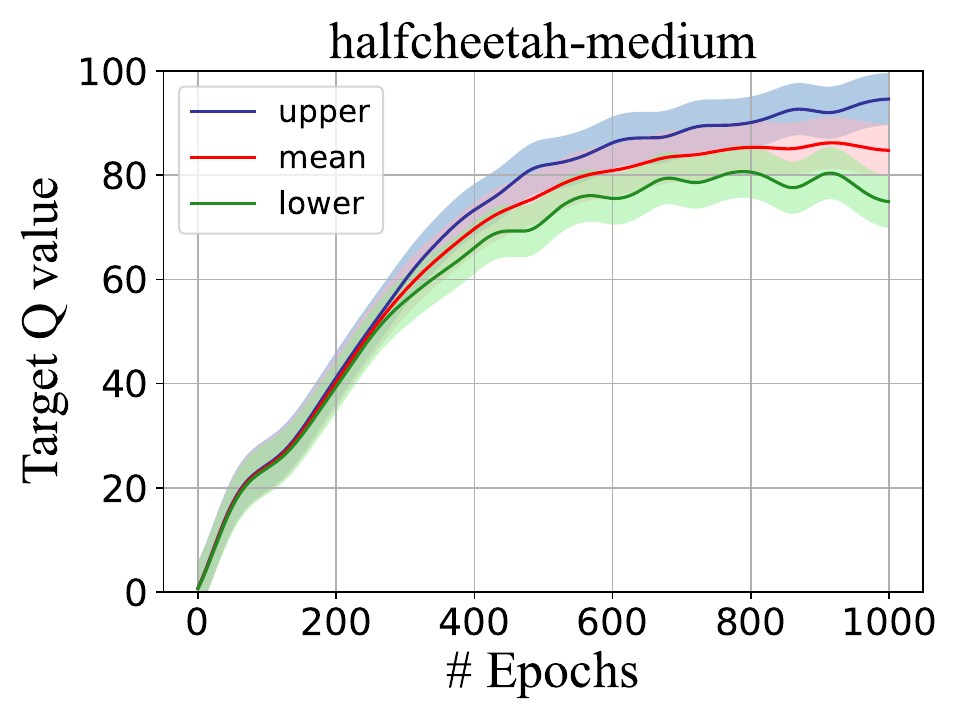}
  \includegraphics[width=0.19\linewidth]{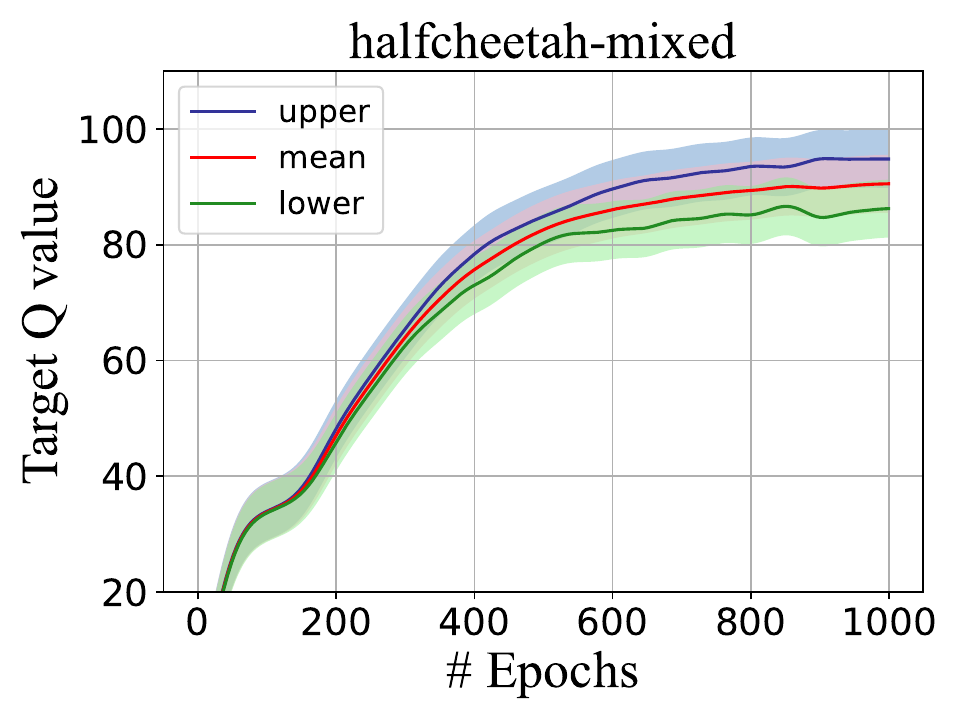}
  \includegraphics[width=0.19\linewidth]{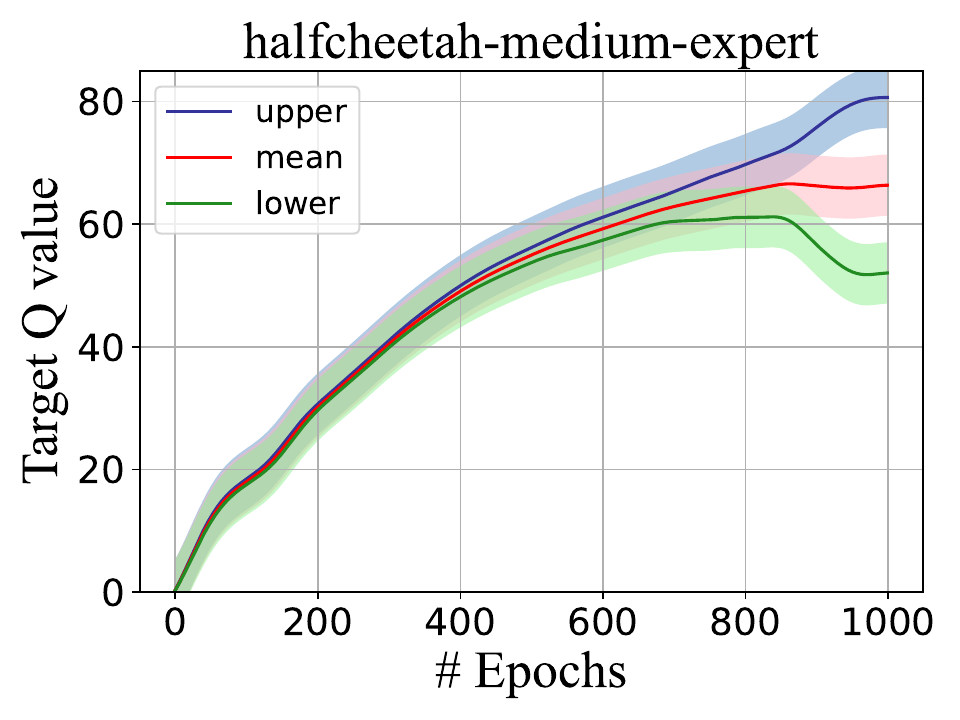}
  \includegraphics[width=0.19\linewidth]{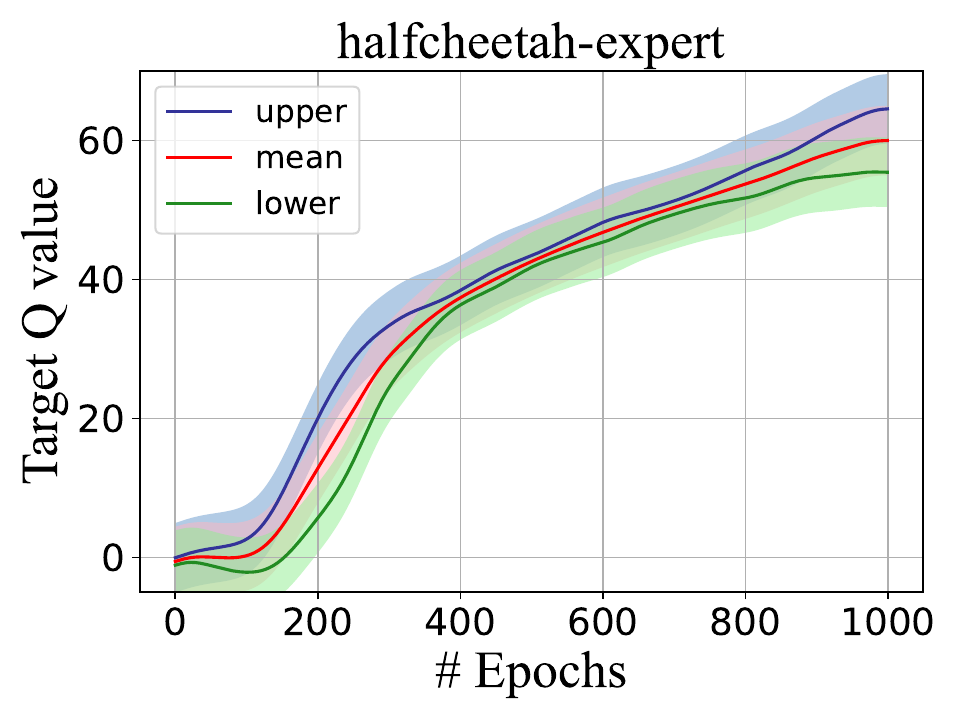}
  \vspace{2mm}
  \includegraphics[width=0.19\linewidth]{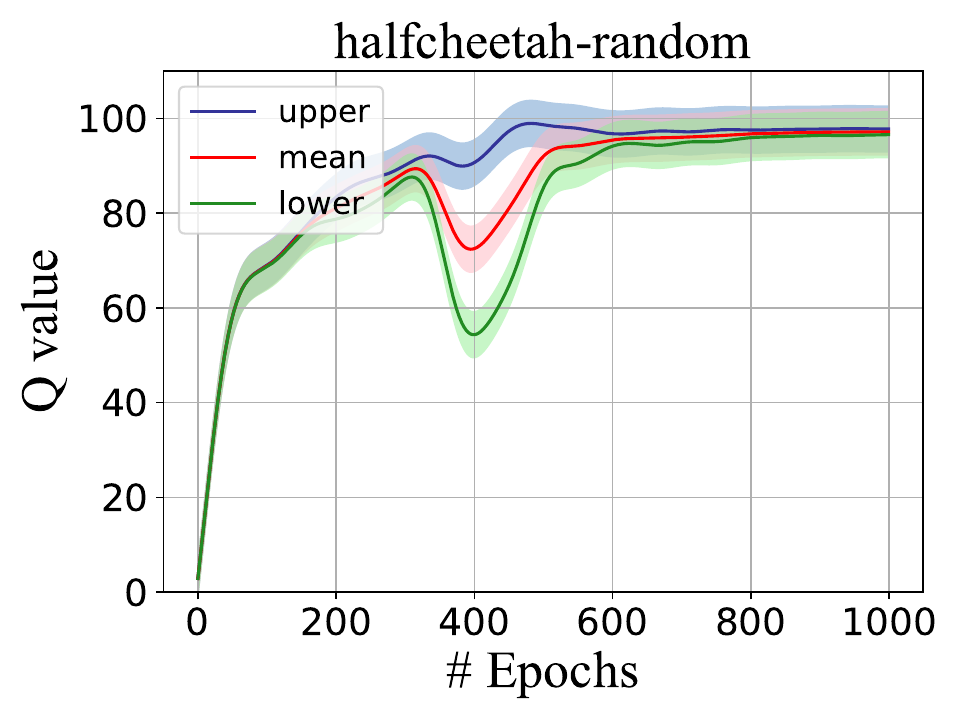}
  \includegraphics[width=0.19\linewidth]{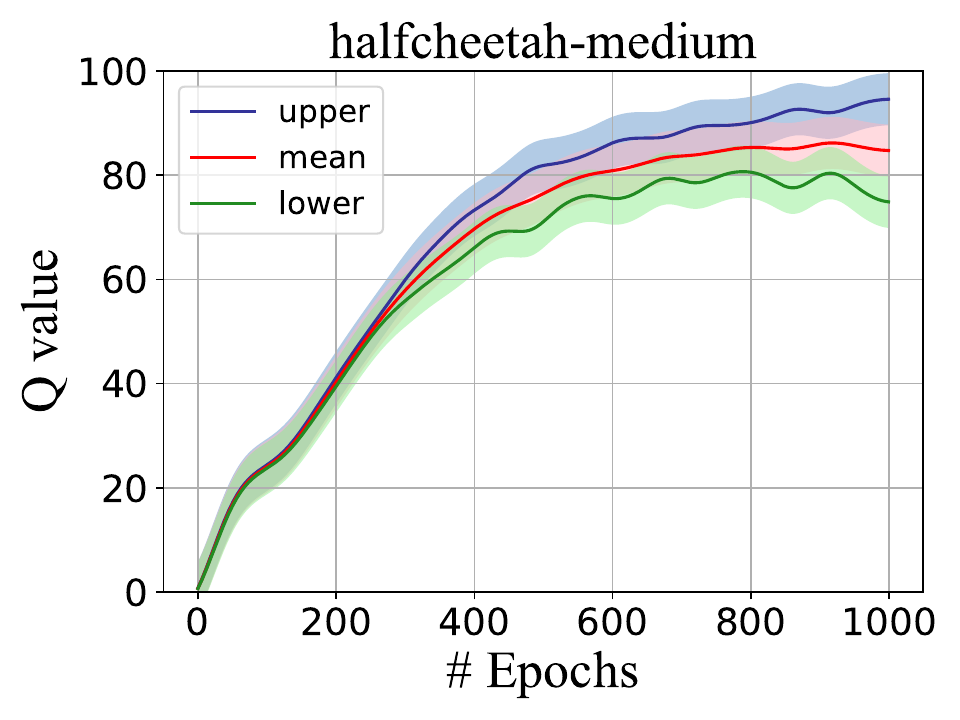}
  \includegraphics[width=0.19\linewidth]{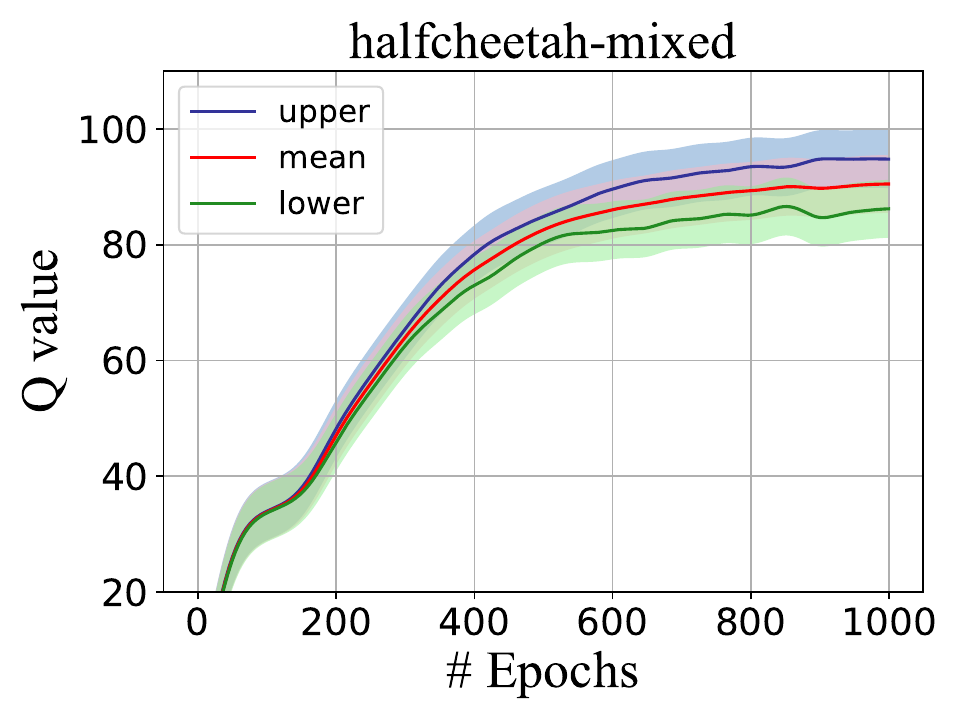}
  \includegraphics[width=0.19\linewidth]{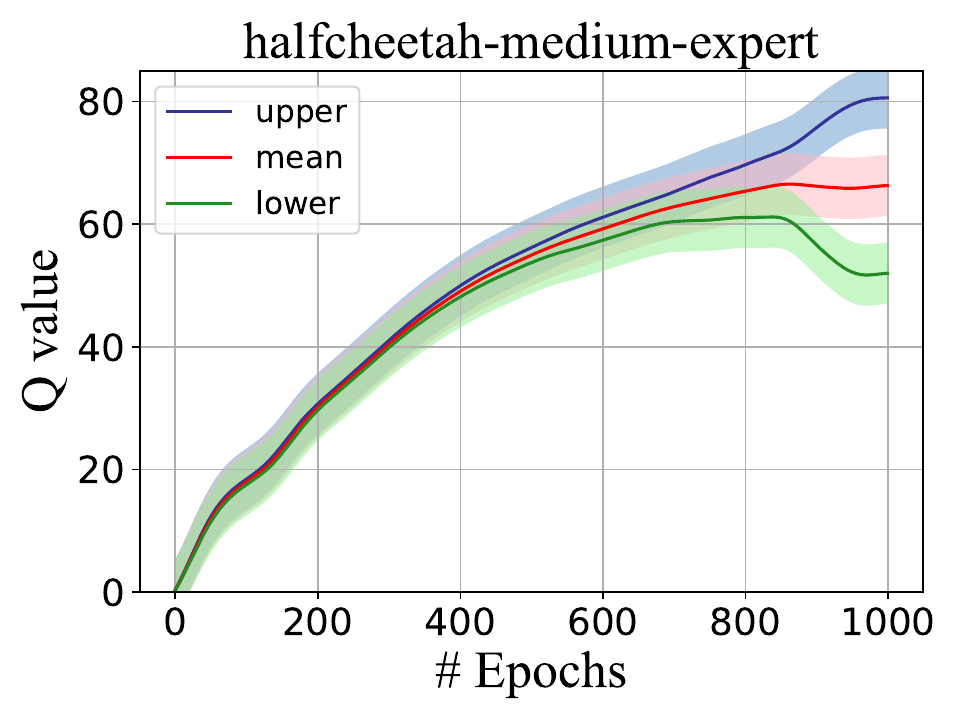}
  \includegraphics[width=0.19\linewidth]{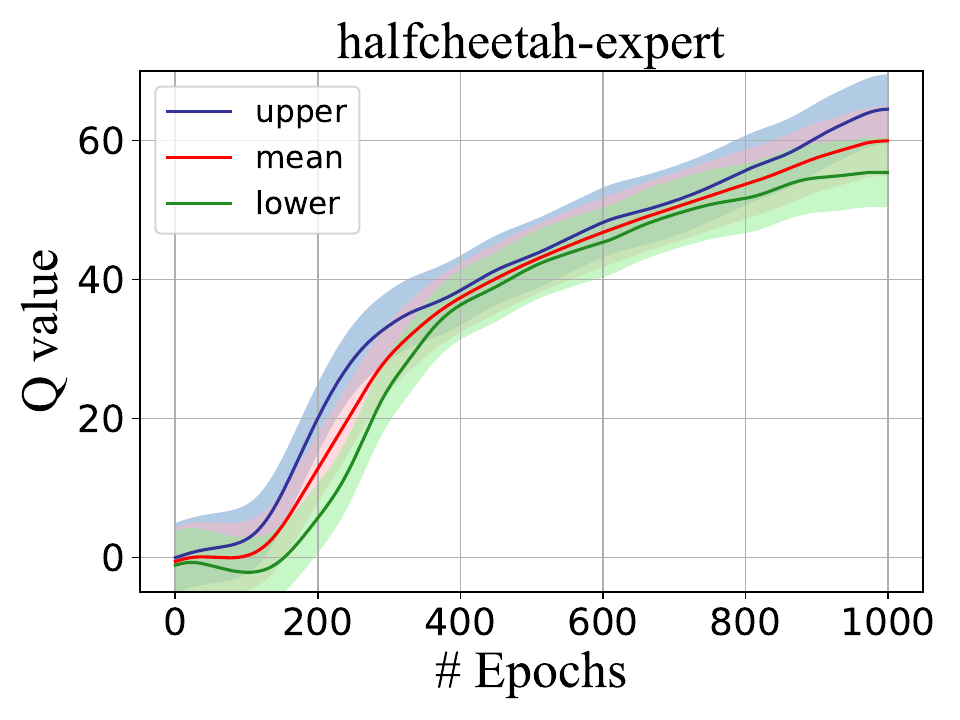}
  \vspace{2mm}
  \includegraphics[width=0.19\linewidth]{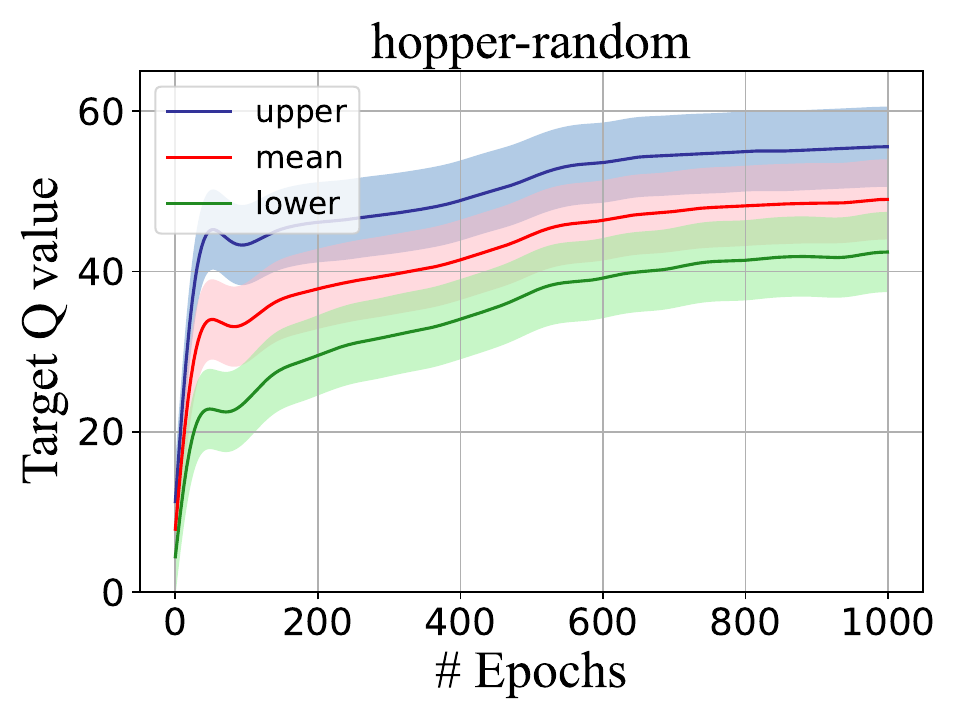}
  \includegraphics[width=0.19\linewidth]{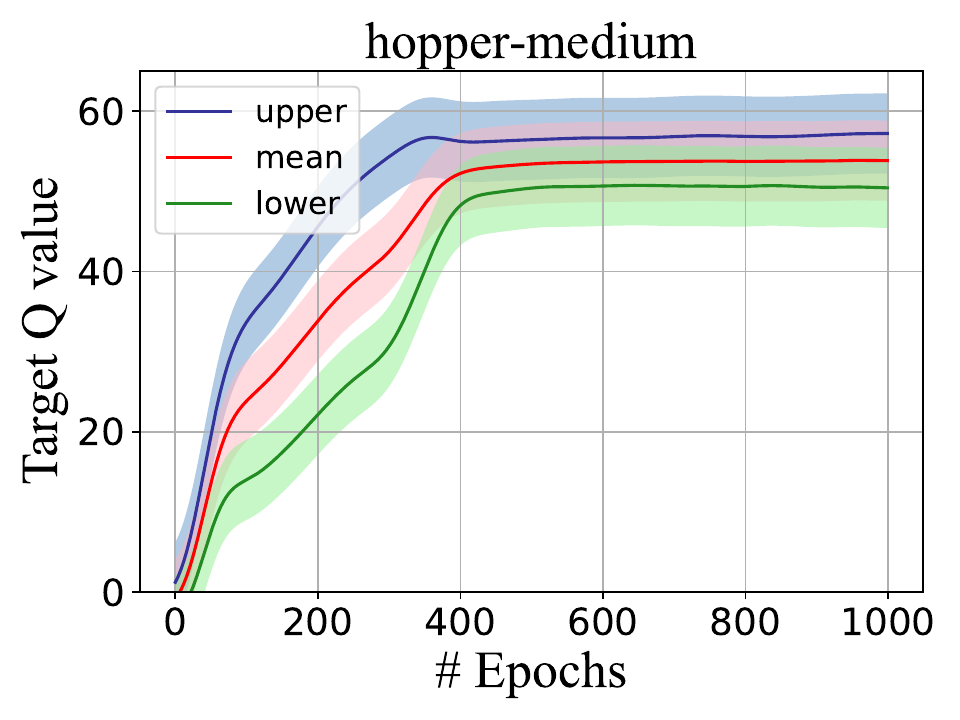}
  \includegraphics[width=0.19\linewidth]{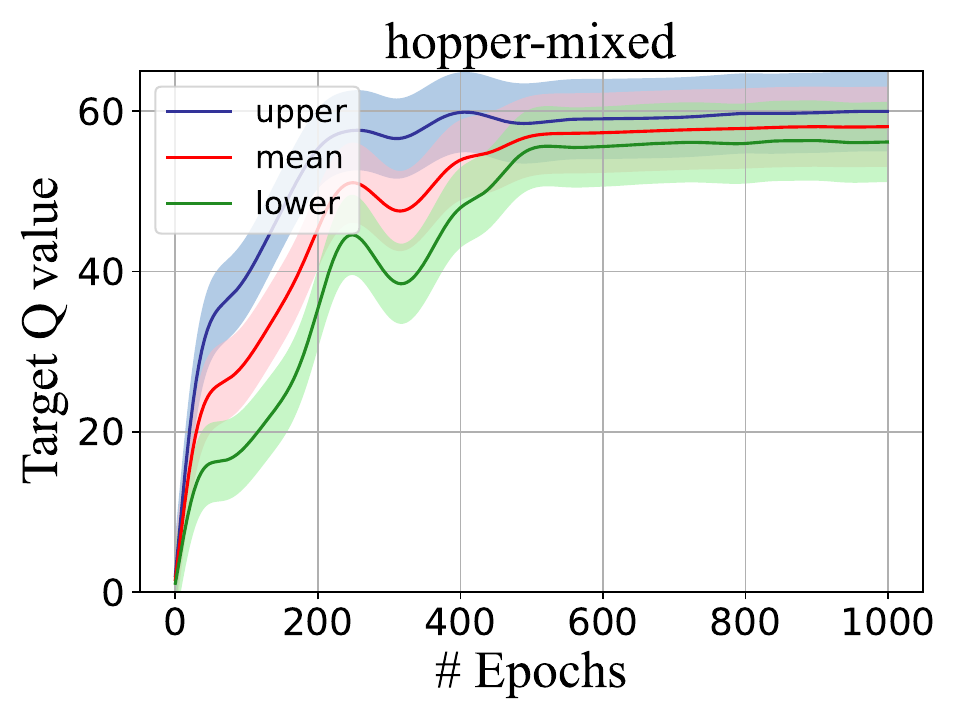}
  \includegraphics[width=0.19\linewidth]{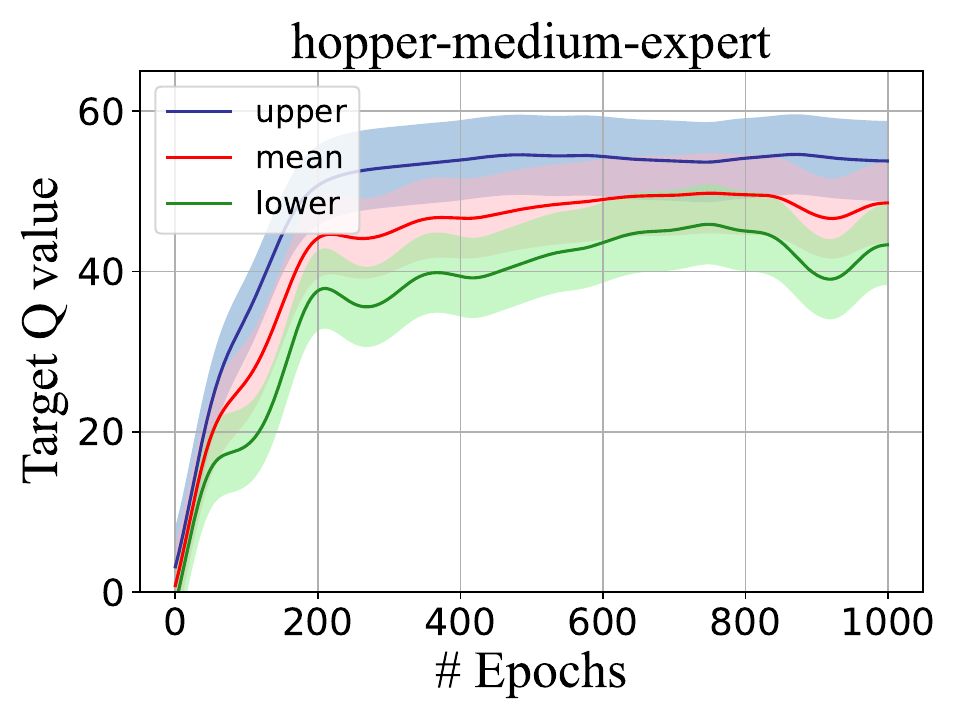}
  \includegraphics[width=0.19\linewidth]{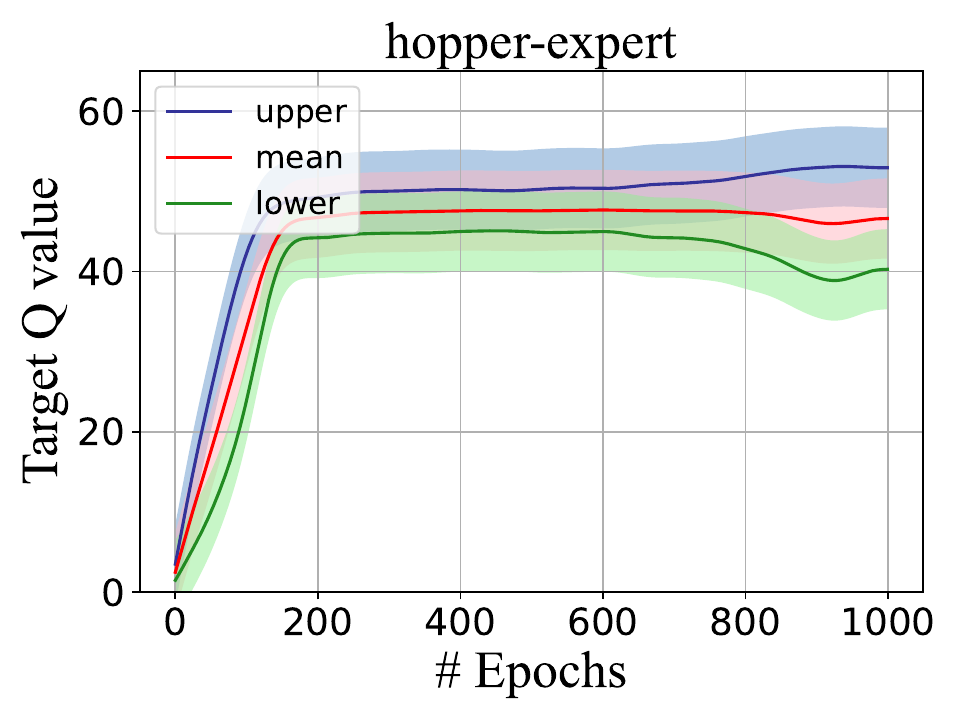}
  \vspace{2mm}
  \includegraphics[width=0.19\linewidth]{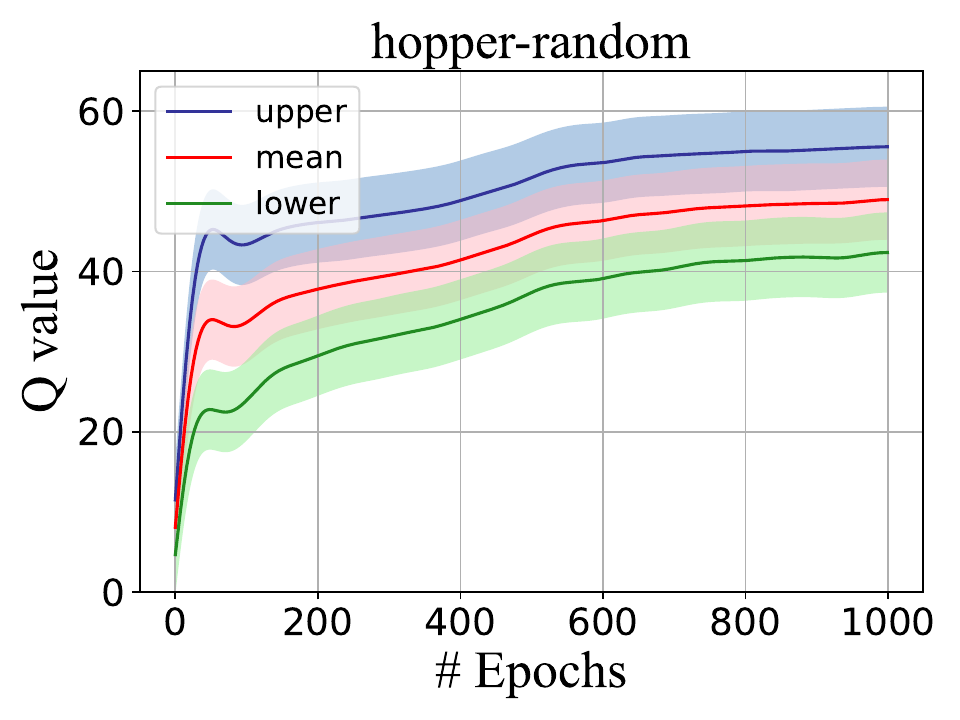}
  \includegraphics[width=0.19\linewidth]{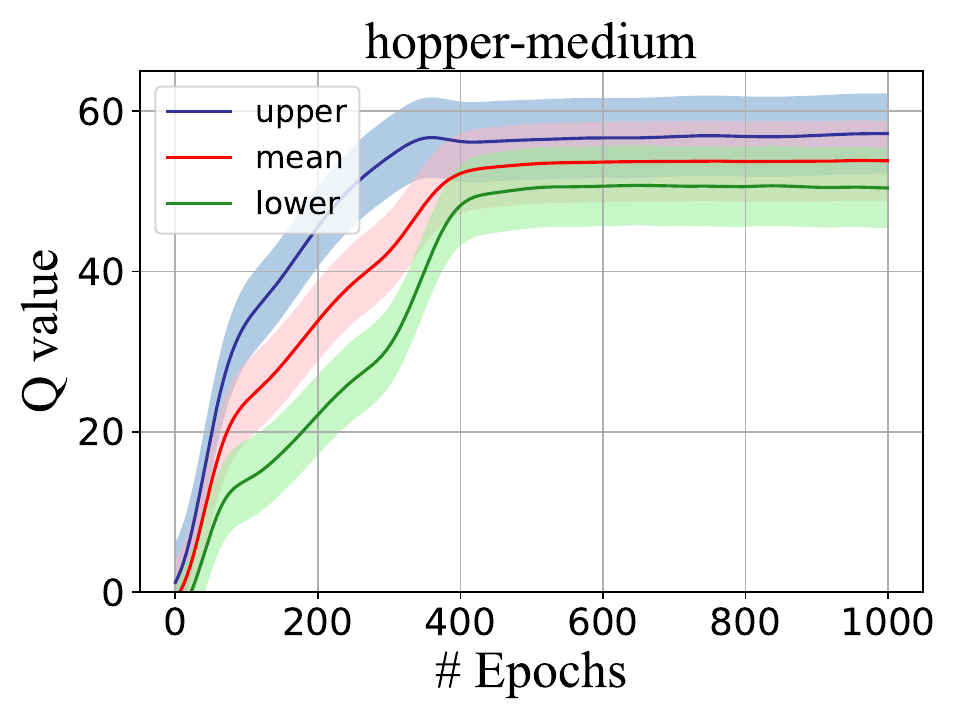}
  \includegraphics[width=0.19\linewidth]{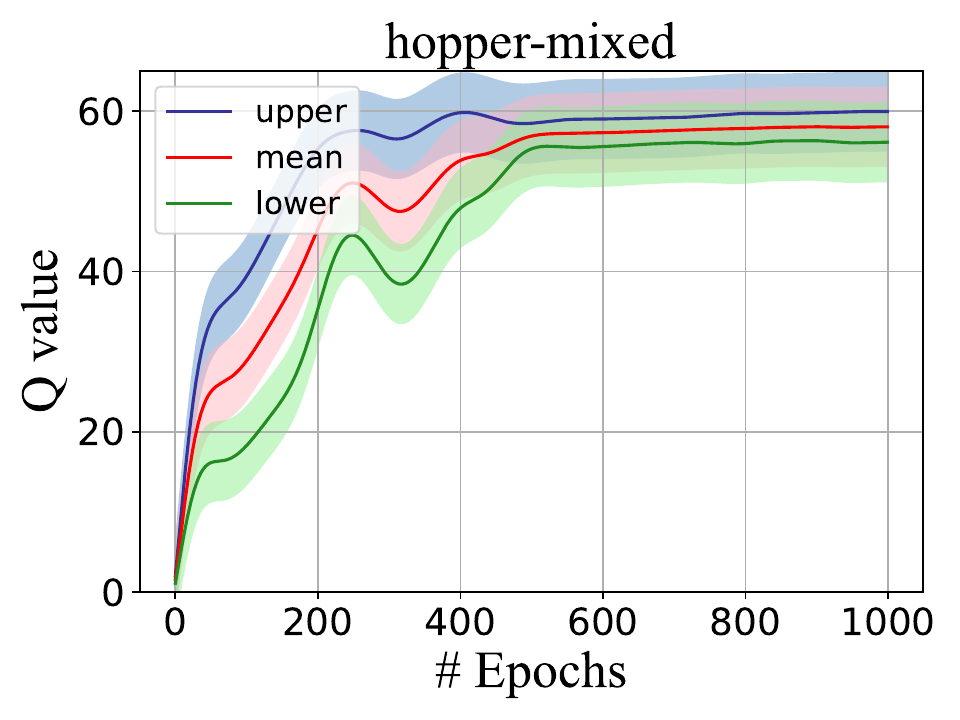}
  \includegraphics[width=0.19\linewidth]{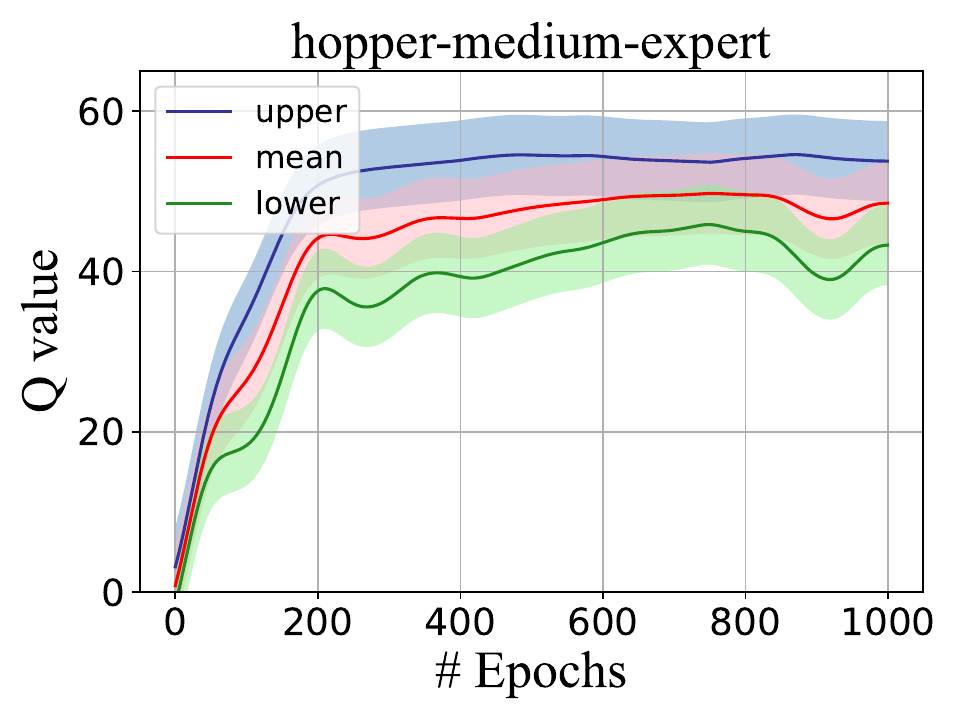}
  \includegraphics[width=0.19\linewidth]{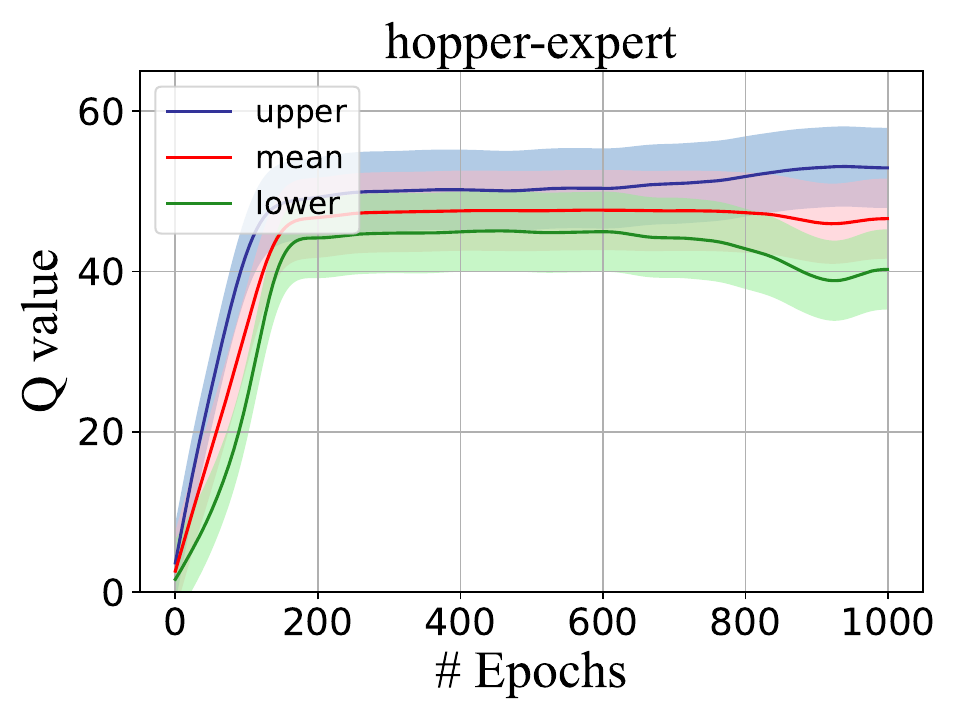}
  \vspace{2mm}
   \includegraphics[width=0.19\linewidth]{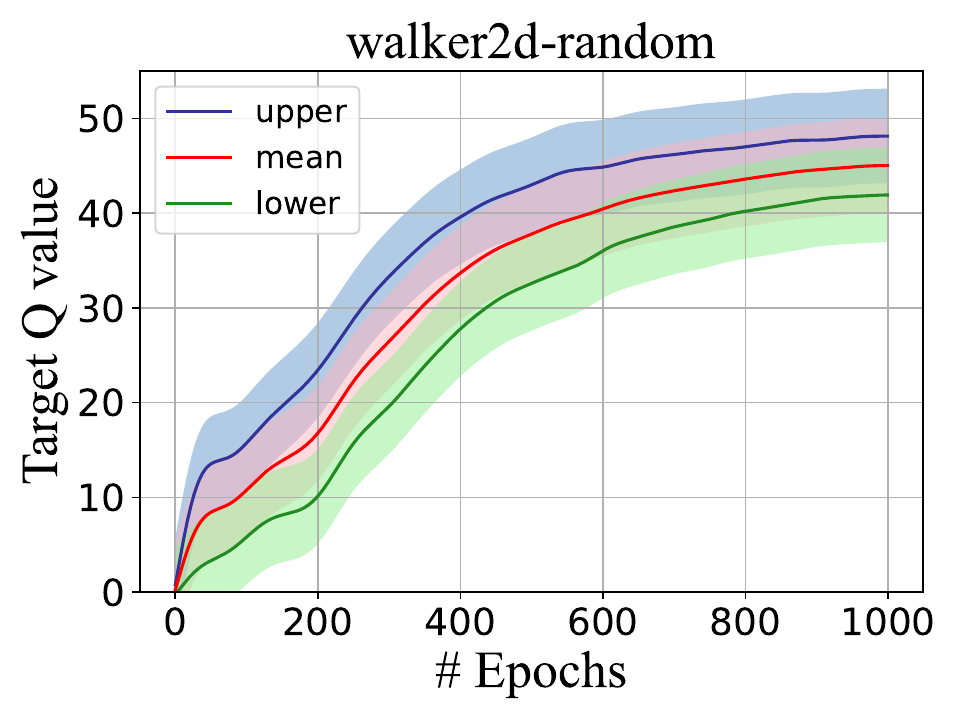}
  \includegraphics[width=0.19\linewidth]{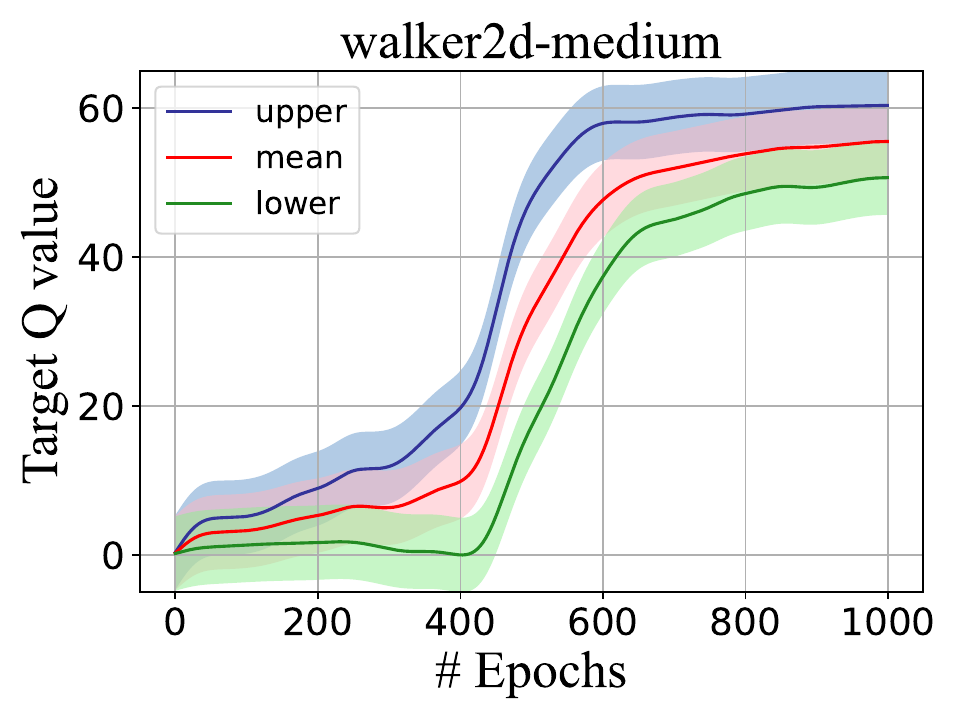}
  \includegraphics[width=0.19\linewidth]{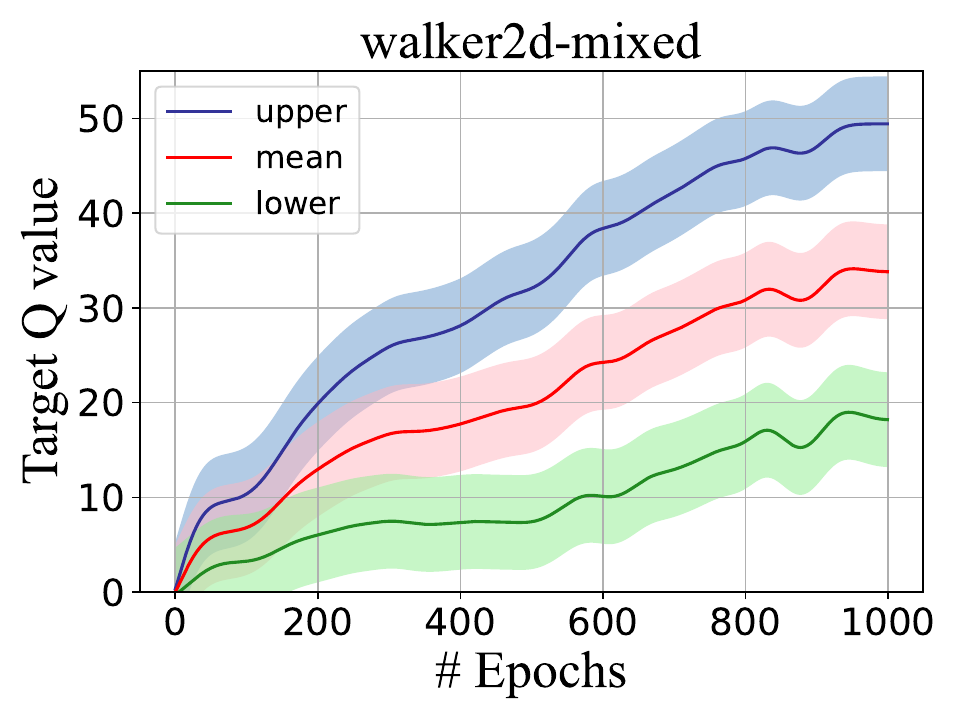}
  \includegraphics[width=0.19\linewidth]{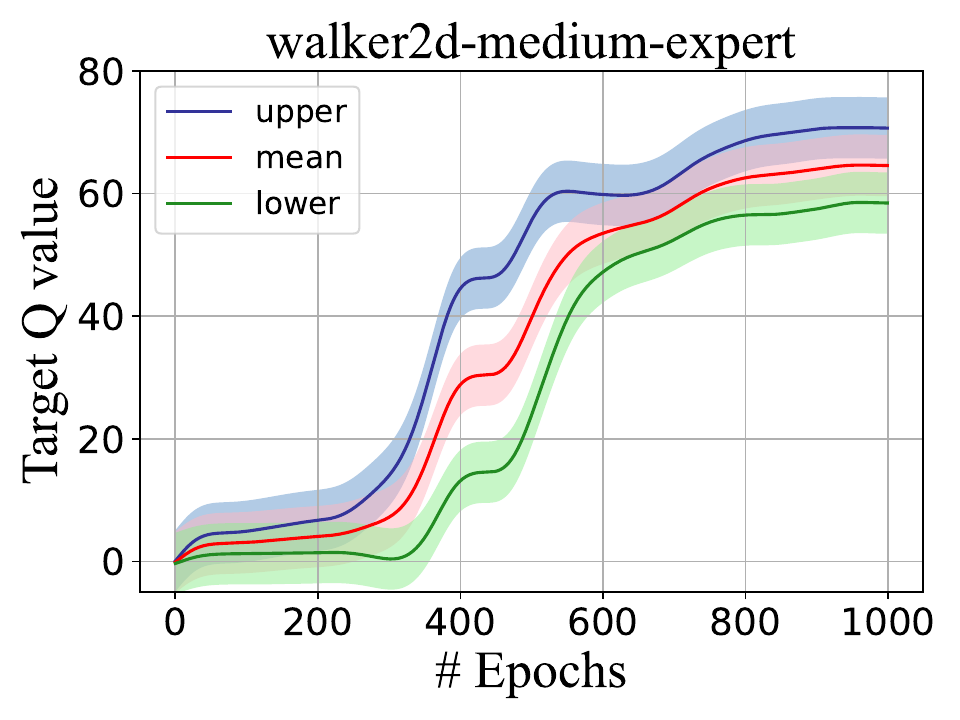}
  \includegraphics[width=0.19\linewidth]{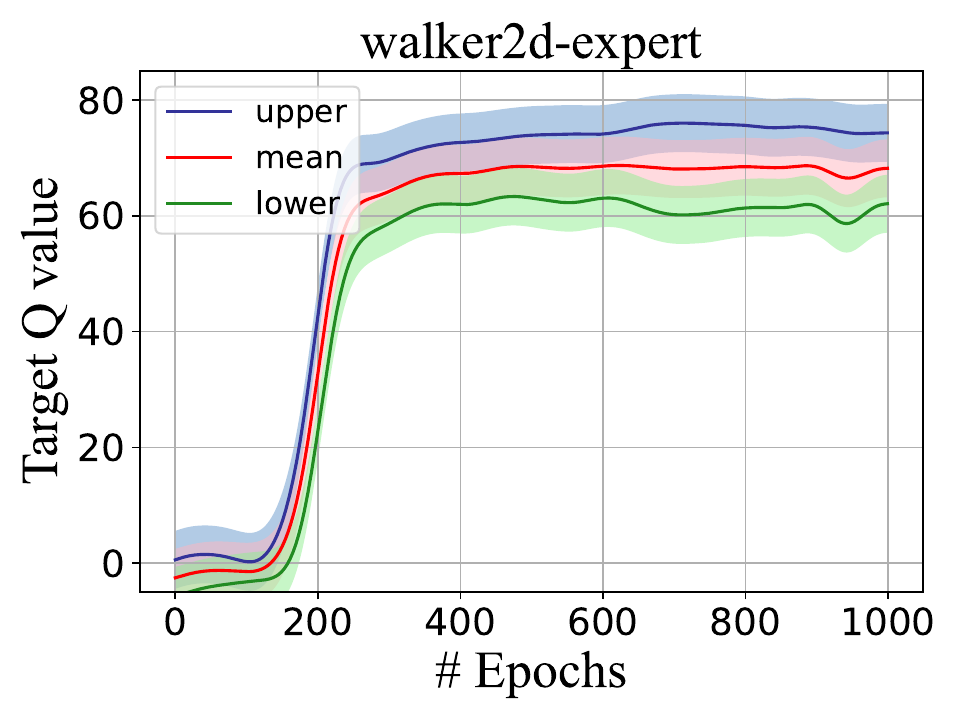}
  \vspace{2mm}
  \includegraphics[width=0.19\linewidth]{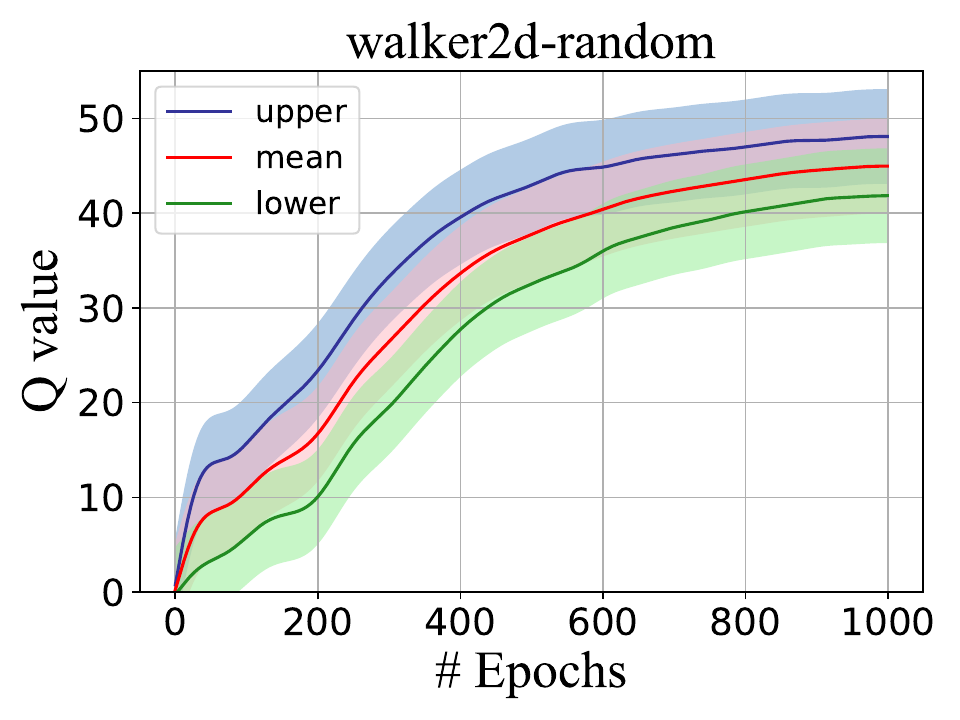}
  \includegraphics[width=0.19\linewidth]{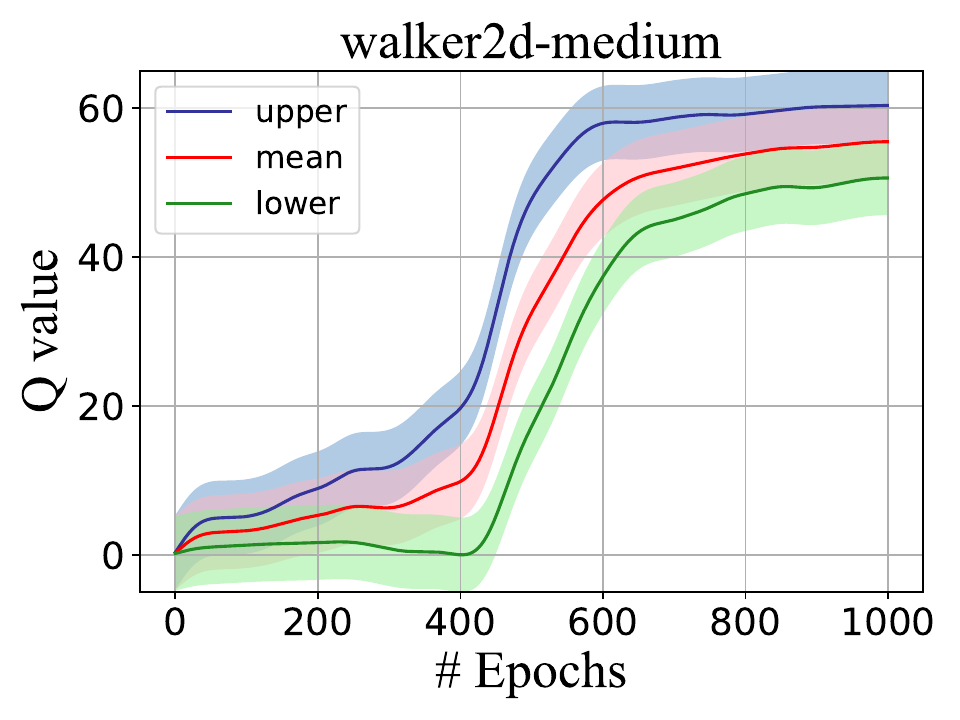}
  \includegraphics[width=0.19\linewidth]{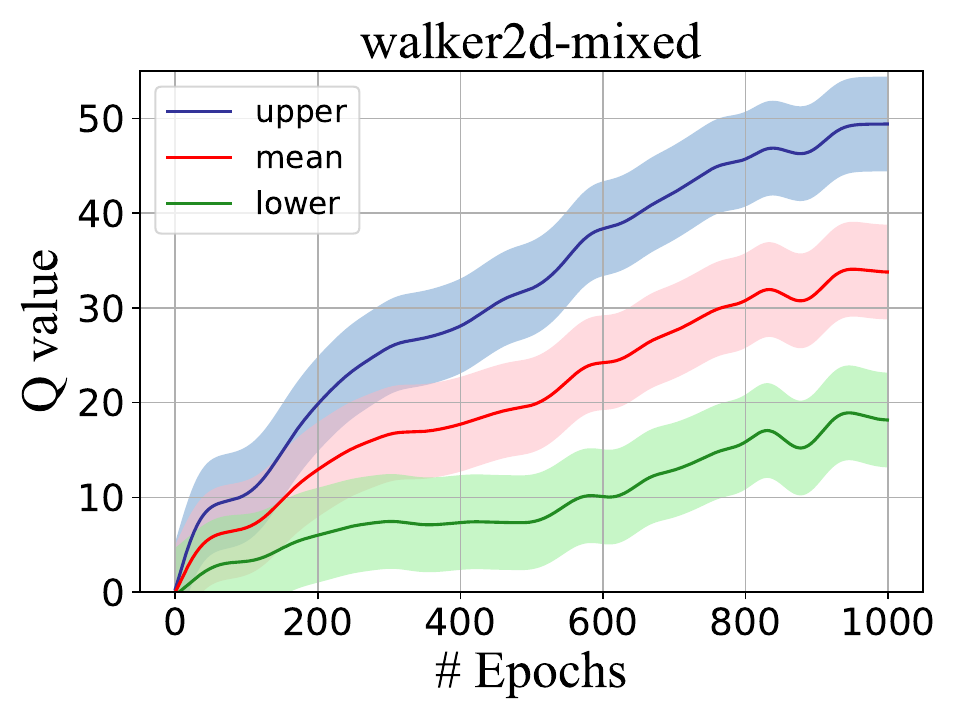}
  \includegraphics[width=0.19\linewidth]{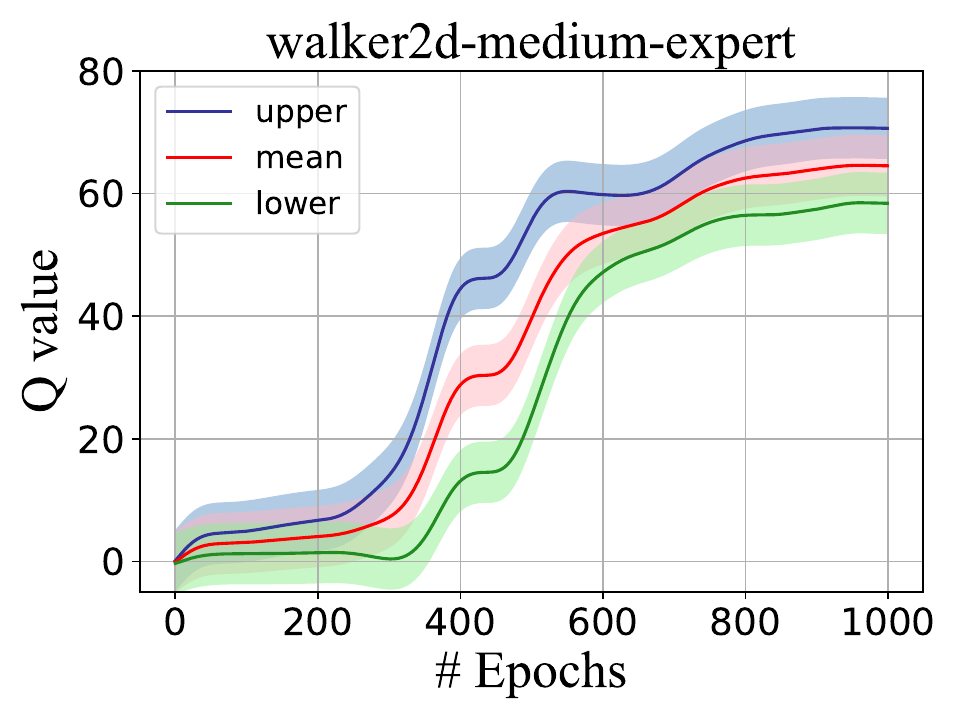}
  \includegraphics[width=0.19\linewidth]{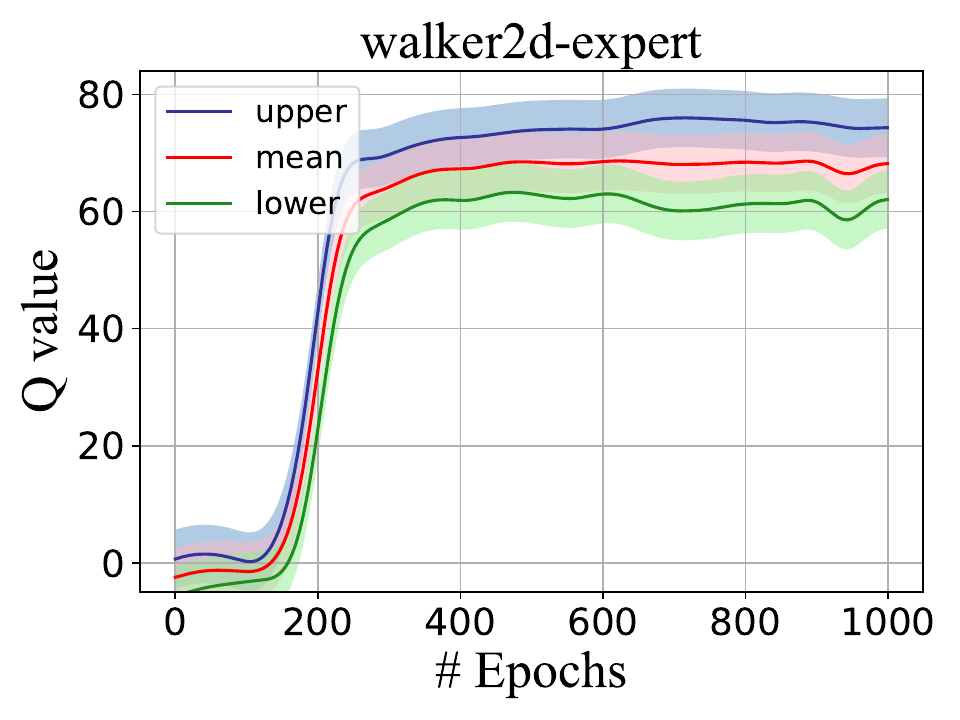}
  \caption{Learning curves of target Q-value and Q-value on three environments of 15 tasks. Three lines represent the upper, mean and lower scores respectively.}
  \label{fig:learn_process_Q}
\end{figure*}

\paragraph{Compared to RAMBO} 
\textbf{(Pessimism)} We conduct a comparative experiment, selecting RAMBO and the high-performing TD3+BC method for evaluation on two medium-expert datasets. 
Based on the learning curves depicted in Fig.~\ref{fig:3} and experimental results in Table \ref{tab:comp}, we can demonstrate that MORAL outperforms the robust adversarial method RAMBO and TD3+BC in terms of both performance and stability. Fig.~\ref{fig:3} reveals the instability in the policy learning process of RAMBO. 
RAMBO is built on the robust adversarial RL, using conservative optimization between models and policies. MORAL instead considers the entire spectrum of plausible transitions, and autonomously determines which portion merits heightened attention. 
From the view of MDP, the behavior showcases a range of adaptability, stretching from pessimistic to optimistic when evaluating policy $\pi$.
\textbf{(Applicability)} Different from RAMBO, MORAL does not require setting deterministic parameters (i.e., ensemble numbers \& rollout horizons) for different tasks and environments. Additionally, MORAL outperforms RAMBO across 12 tasks in three environments.

\paragraph{Compared to ARMOR and PMDB} We conduct a thorough analysis of MORAL in comparison to ARMOR and PMDB in three expert datasets. The results in Table \ref{tab:2} highlight the superior performance achieved by MORAL. 
Compared to PMDB, which solely relies on dynamic adversarial optimization within pessimism, MORAL incorporated the differential factor in adversarial process demonstrates superior sample efficiency and stability. 
Regarding sample efficiency, we measure the convergence step, which we define as the number of training steps required for a method to maintain a stable score within a 5\% variance band for $10$ consecutive steps. As shown in Fig.~\ref{fig:conv}, MORAL achieves the fastest convergence in expert tasks. 
Moreover, Compared to ARMOR, which involves building an actor-critic adversarial architecture, MORAL based on the adversarial policy (ADV) exhibits more accurate learning of offline transitions from expert datasets, showcasing superior performance. 
In summary, MORAL demonstrates high performance and remarkable sample efficiency in comparison to other adversarial methods.

% \vspace{-3mm}
\subsection{Property Analysis}
\paragraph{Analysis of adversarial process} 
To assess the effectiveness of the adversarial framework proposed in MORAL, we conduct a comprehensive experiment across all 15 datasets. We present the learning curves of all datasets in Fig. \ref{fig:figure2}, consistently demonstrating the performance improvement for both MORAL and the adversarial policy over the whole training epochs. Finally, both policies converge, demonstrating the stability and robustness of MORAL.

Throughout the training process across all 15 datasets, the policies of the primary player consistently outperform the adversarial policies of the second player. Remarkably, both policies exhibit notable sample efficiency, converging after a mere 600 epochs.
Even on more challenging random datasets, while the learning process shows some expected fluctuations, the secondary player's strategic perturbations serve a crucial purpose: they create a robust training environment that strengthens the primary policy's ability to handle diverse scenarios. This controlled adversarial pressure helps the primary policy achieve stable performance after 600 epochs.
The game process introduces interference from the second player to achieve robust learning of offline policies. The primary player seeks to optimize a dependable policy for the relevant MDP in response to stochastic disruptions caused by the biased samples from the second player. 
Overall, these experimental findings further substantiate the effectiveness and robustness of the incorporation of secondary player.
% To analyze the effectiveness of MZG proposed in MORAL, we show the learning curves of three different environments in medium datasets. The experimental findings, depicted in Figure \ref{fig:figure2}, reveal a consistent improvement in performance for both MORAL and the adversarial policy throughout the training epochs. Ultimately, both policies converge, showcasing stable and refined performance. 
% Throughout the training process, the effectiveness of the policy of the primary player consistently surpasses that of the adversarial policy of the second player. Furthermore, it showcases noteworthy sample efficiency, with both policies converging after mere 600 epochs.
% These experimental findings further validate the effectiveness and robustness of the proposed Markov zero-sum game.

\begin{figure*}[]
  \centering
  \includegraphics[width=0.19\linewidth]{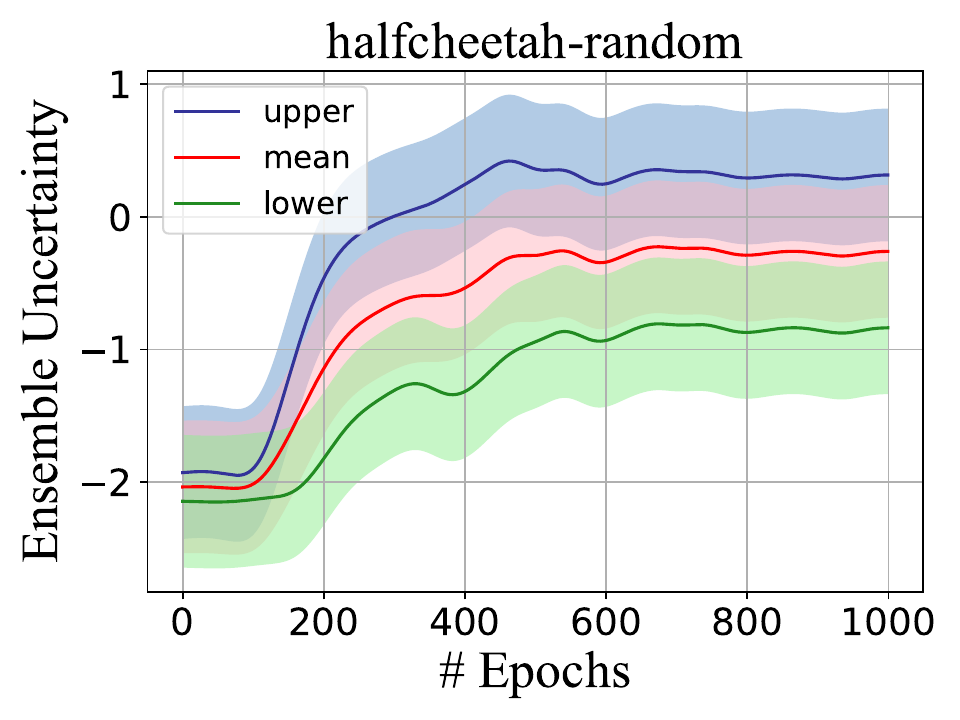}
  \includegraphics[width=0.19\linewidth]{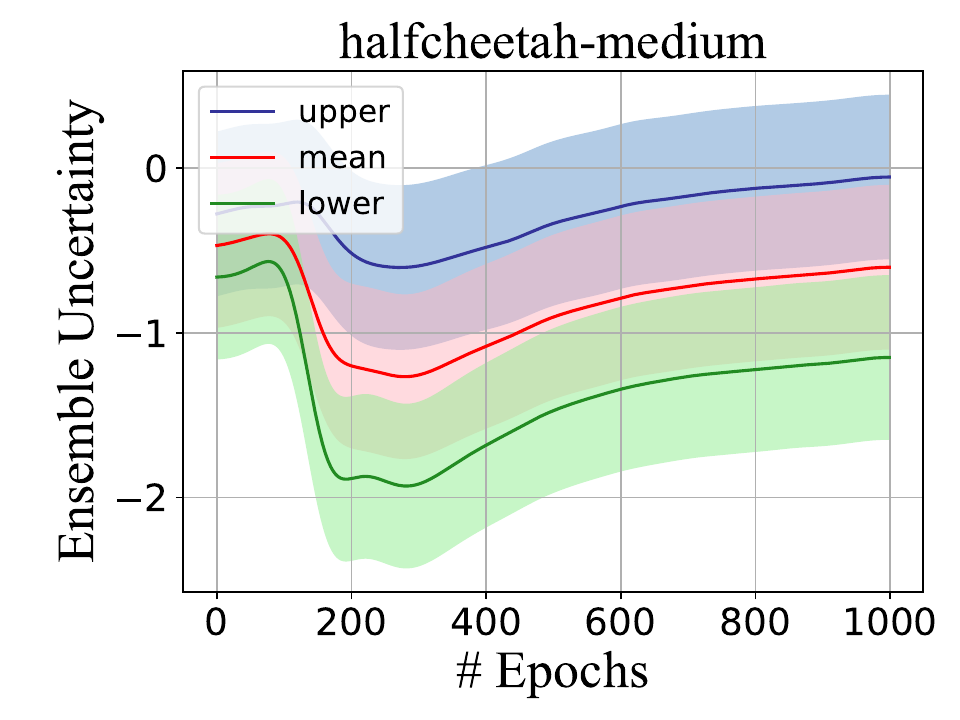}
  \includegraphics[width=0.19\linewidth]{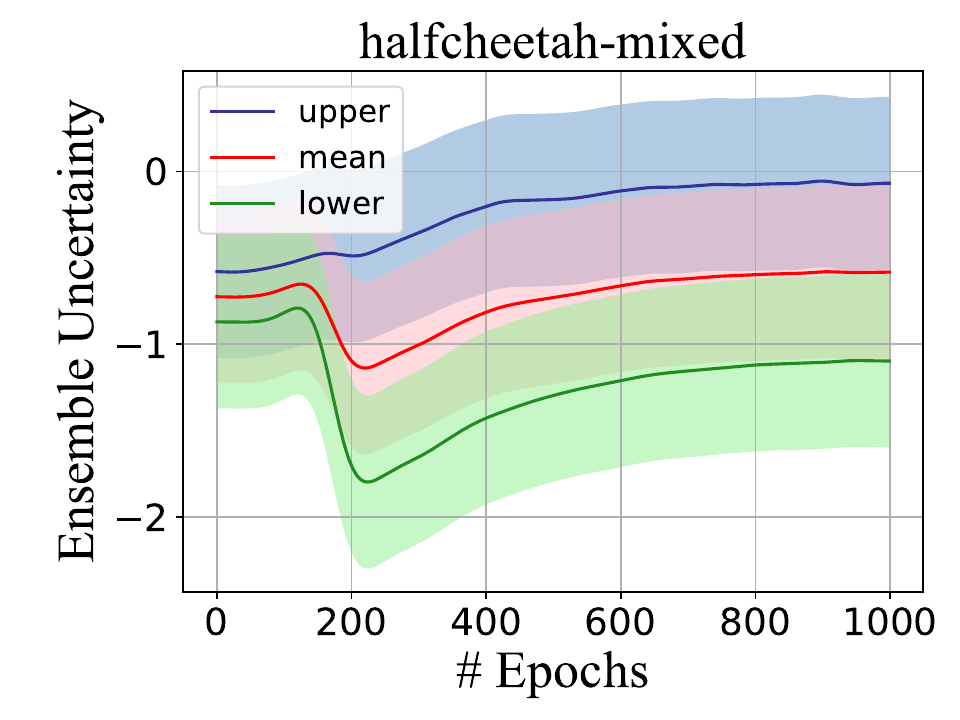}
  \includegraphics[width=0.19\linewidth]{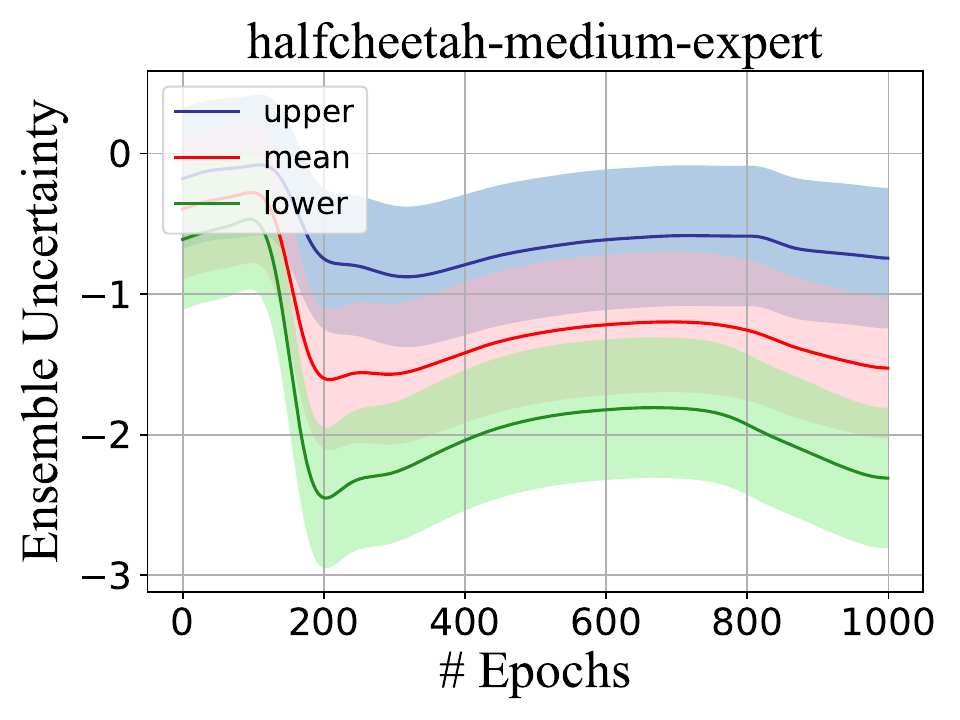}
  \includegraphics[width=0.19\linewidth]{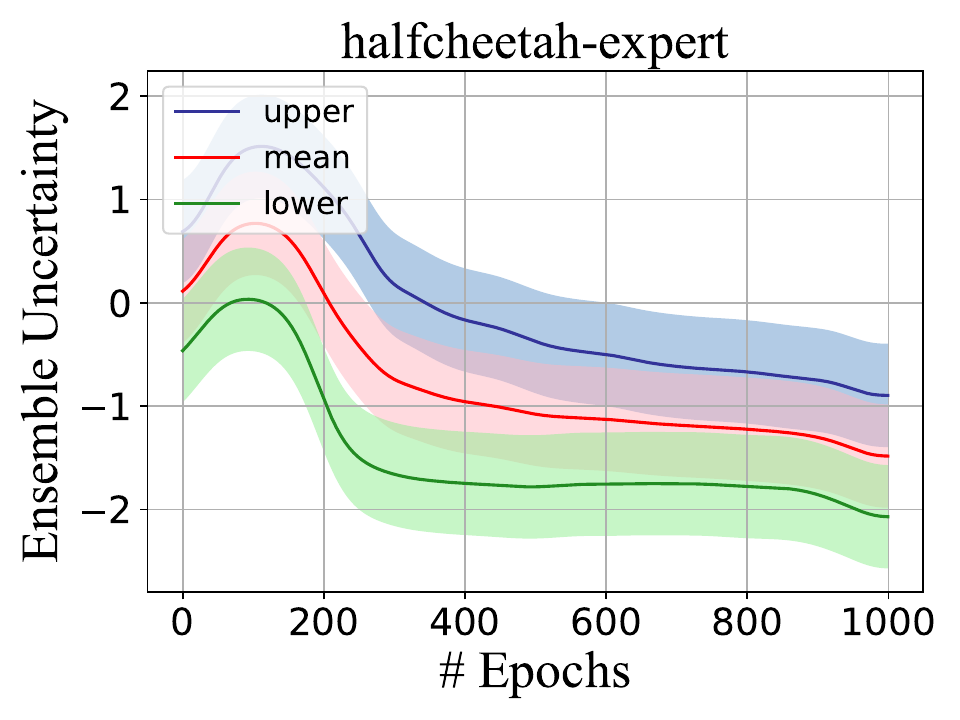}
  \vspace{2mm}
  \includegraphics[width=0.19\linewidth]{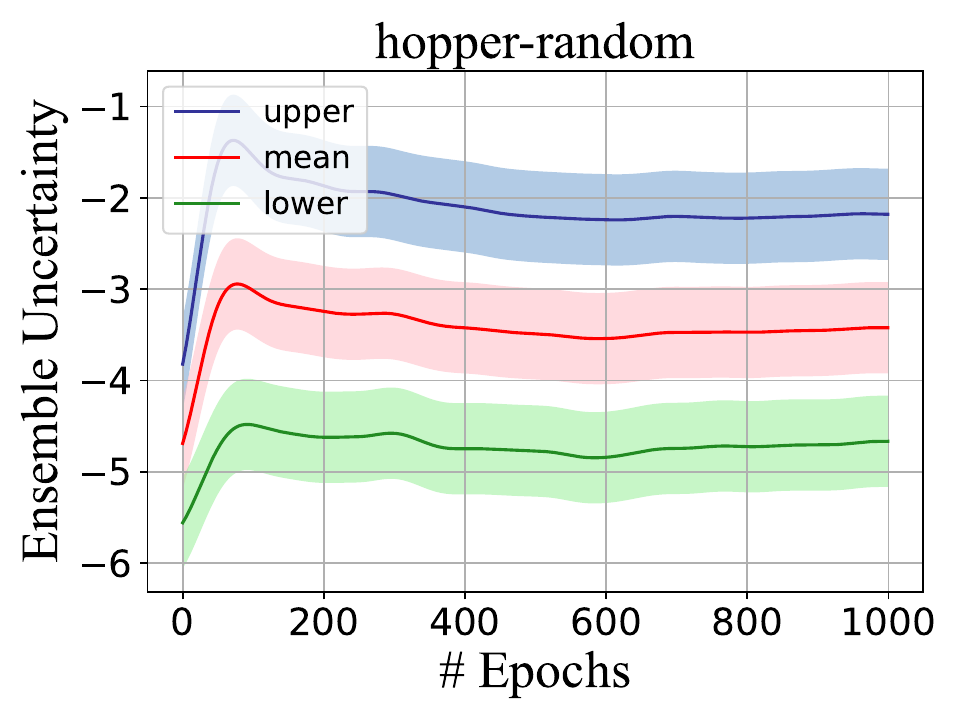}
  \includegraphics[width=0.19\linewidth]{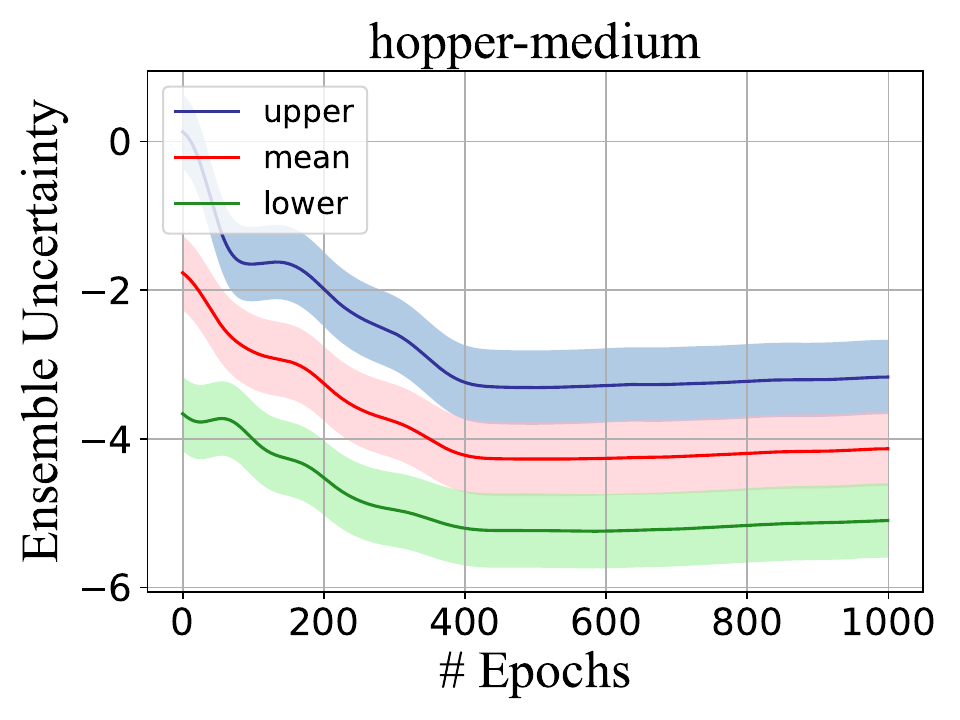}
  \includegraphics[width=0.19\linewidth]{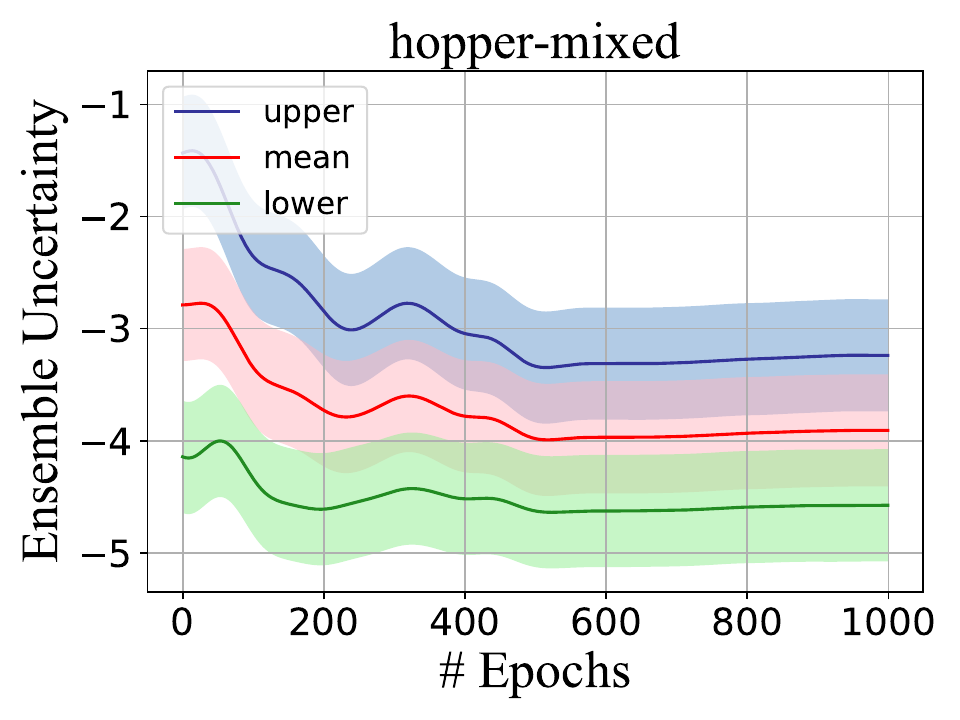}
  \includegraphics[width=0.19\linewidth]{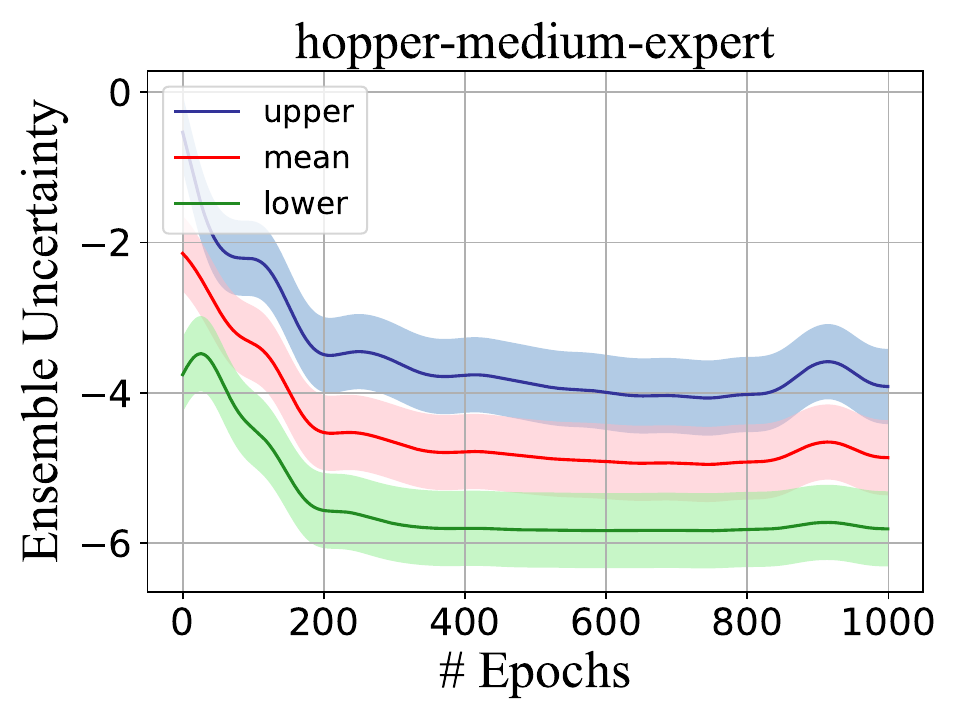}
  \includegraphics[width=0.19\linewidth]{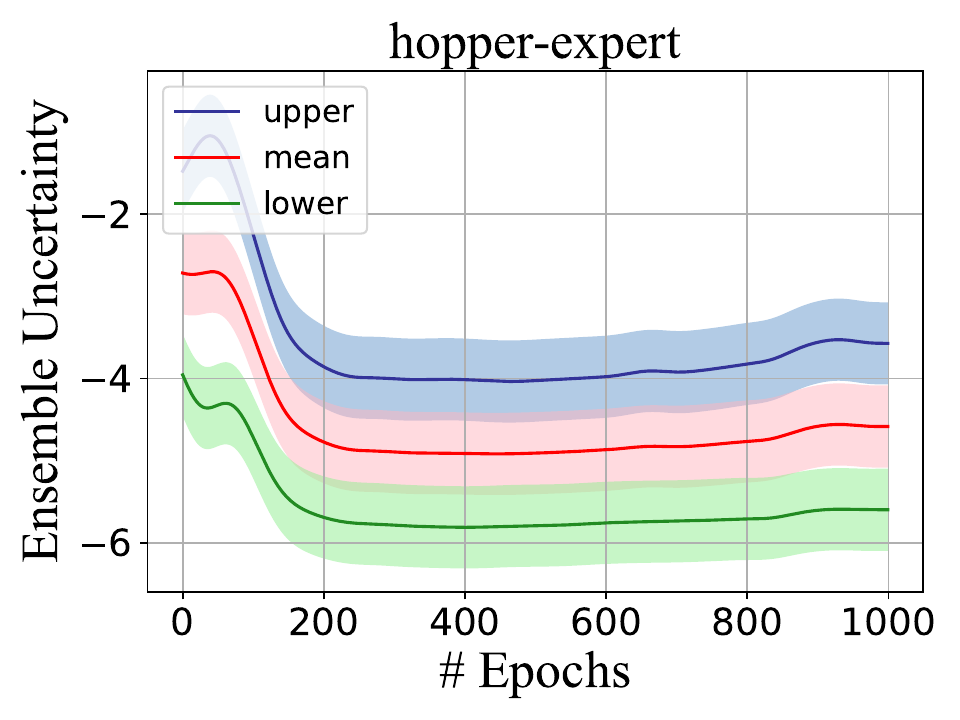}
  \vspace{2mm}
  \includegraphics[width=0.19\linewidth]{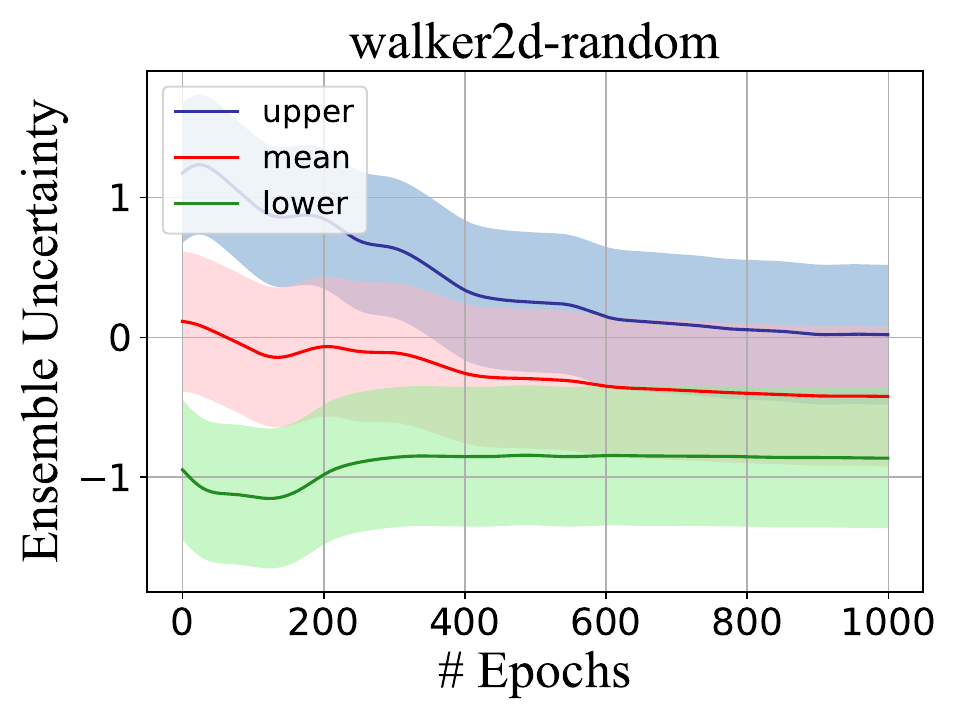}
  \includegraphics[width=0.19\linewidth]{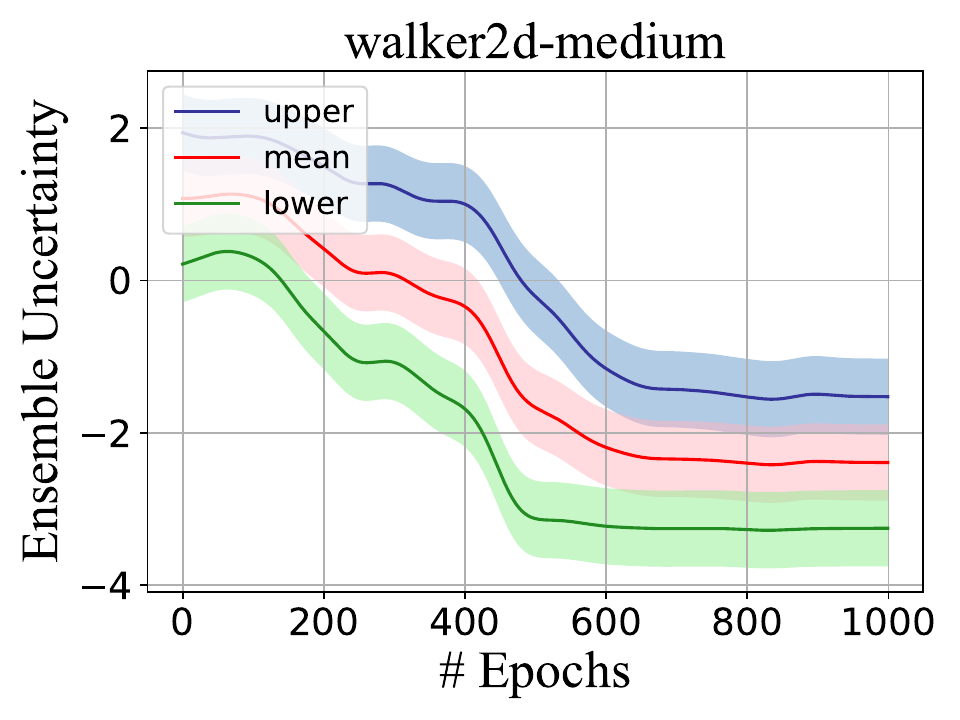}
  \includegraphics[width=0.19\linewidth]{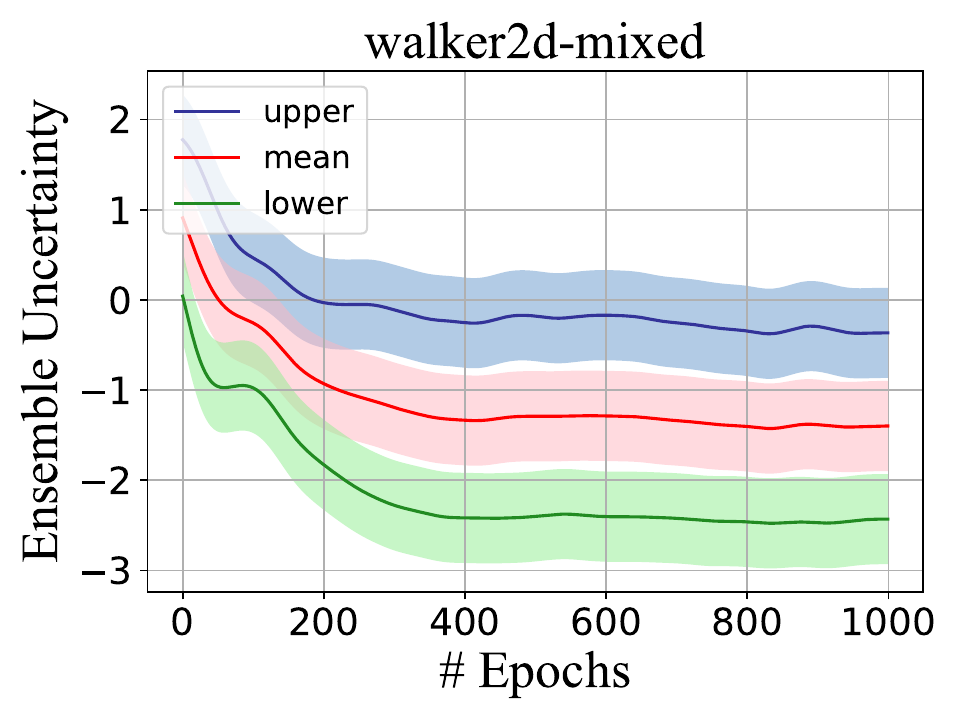}
  \includegraphics[width=0.19\linewidth]{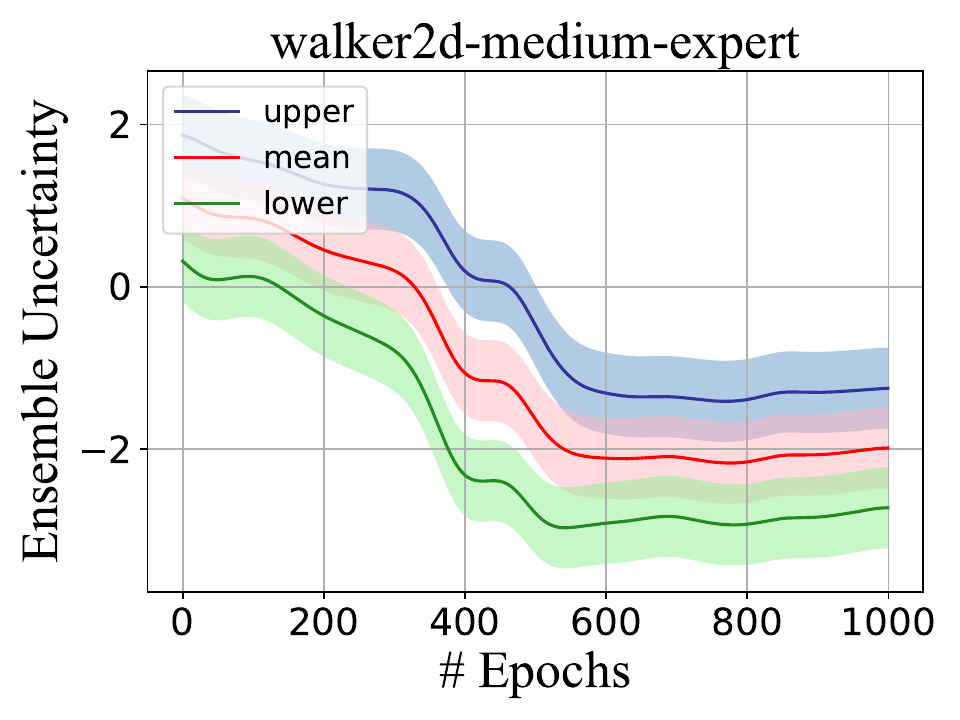}
  \includegraphics[width=0.19\linewidth]{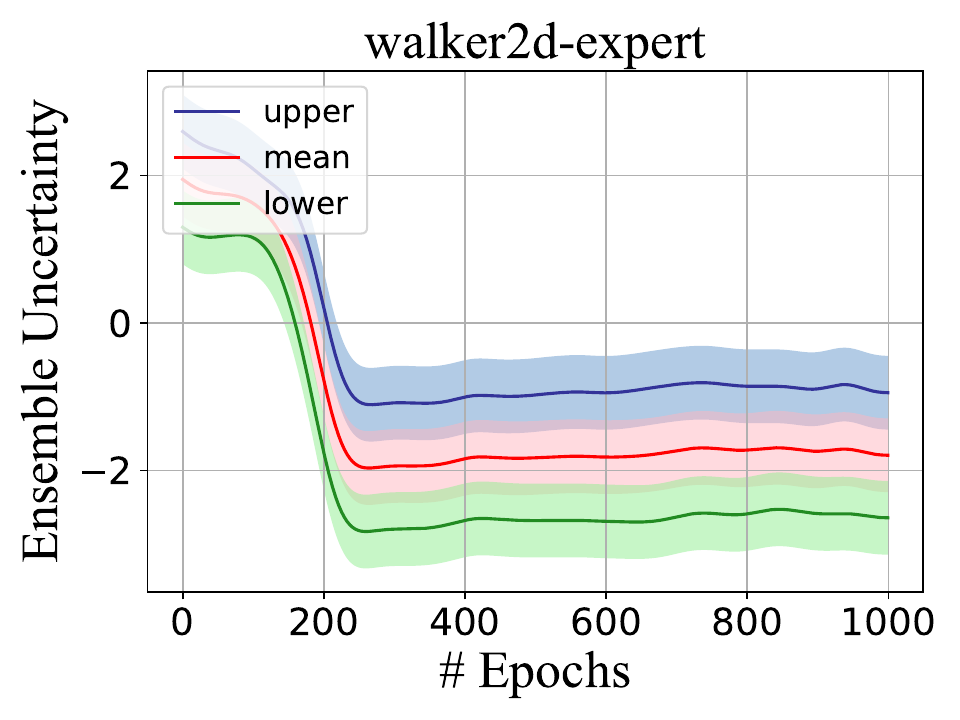}
  \caption{Learning curves of ensemble uncertainty on three environments of 15 tasks. Three lines represent the upper, mean and lower scores respectively.}
  \label{fig:learn_process_uncertainty}
\end{figure*}

\paragraph{Learning Process in MORAL} To further analyze the effectiveness of MORAL, we conduct an analysis during the learning process in MORAL. 
The target Q-value, Q-value curves on 15 datasets are shown in Fig.~\ref{fig:learn_process_Q}.
MORAL exhibits consistent improvement in both the target Q-values and Q-values during the policy learning process, ultimately achieving stabilization within a predetermined range. The alterations in the target Q-value align consistently with the fluctuations in the Q-value. This finding further underscores the effectiveness of the proposed adversarial process in enhancing the robustness and stability of the learned offline policy. 

Furthermore, the ensemble uncertainty curves during the learning process are shown in Fig.~\ref{fig:learn_process_uncertainty}. 
Ensemble uncertainty is quantified through the logarithm of the standard deviation of predicted means derived from ensemble dynamics samples.
The uncertainty gradually decreases, demonstrating the reduction of ensemble errors during task learning. Moreover, the observed increase in Q-values aligns with the gradual reduction in ensemble uncertainty. The optimization of offline policies, coupled with uncertainty reduction, ensures convergence and high-performance outcomes of the final policy. \textbf{The reduction in uncertainty substantially boosts Q-values, facilitating proficient policy optimization. Moreover, the integration of DF contributes to steering the policy away from divergence, promoting stability.}

Overall, this outcome further underscores the effectiveness of MORAL, which effectively tackles the challenge of policy instability stemming from disparities among ensemble models and the fixed rollout horizon.

\begin{table}[t]
\renewcommand{\arraystretch}{1.2}
\setlength{\tabcolsep}{3.6pt} % 设置列间距为4pt
 \caption{Impact of the hyperparameter $N'$ in MORAL on expert datasets of normalized score.} 
\begin{tabular}{ccccccc}
\toprule
   & \multicolumn{2}{c}{Hopper-expert} & \multicolumn{2}{c}{Halfcheetah-expert} & \multicolumn{2}{c}{Walker2d-expert} \\ \cline{2-7} 
$N'$  & MORAL           & ADV             & MORAL         & ADV                    & MORAL            & ADV              
\\ \hline
    5  
    & 103.1$\pm$4.5       
    & 90.3±2.8       
    & 98.2$\pm$2.1     
    & 87.4±2.7      
    & 110.3$\pm$2.3    
    & 81.3±4.7         
\\
    10  
    & \textbf{113.2$\pm$1.2}       
    & \textbf{104.2±2.3}       
    & \textbf{102.9$\pm$0.7}     
    & \textbf{95.3±1.8}      
    & \textbf{118.6$\pm$0.7}         
    & \textbf{87.3±4.2}        
\\
    15  
    & 102.3$\pm$6.2       
    & 86.5±3.9       
    & 95.0$\pm$4.1     
    & 84.3±6.0      
    & 103.2$\pm$4.1    
    & 80.1±5.3        
\\ \bottomrule
\label{tab:impact}
% \vspace{-5mm}
\end{tabular}
\end{table}

\begin{figure}[]
% \vspace{-0.5cm}
  \centering  \includegraphics[width=0.32\linewidth]{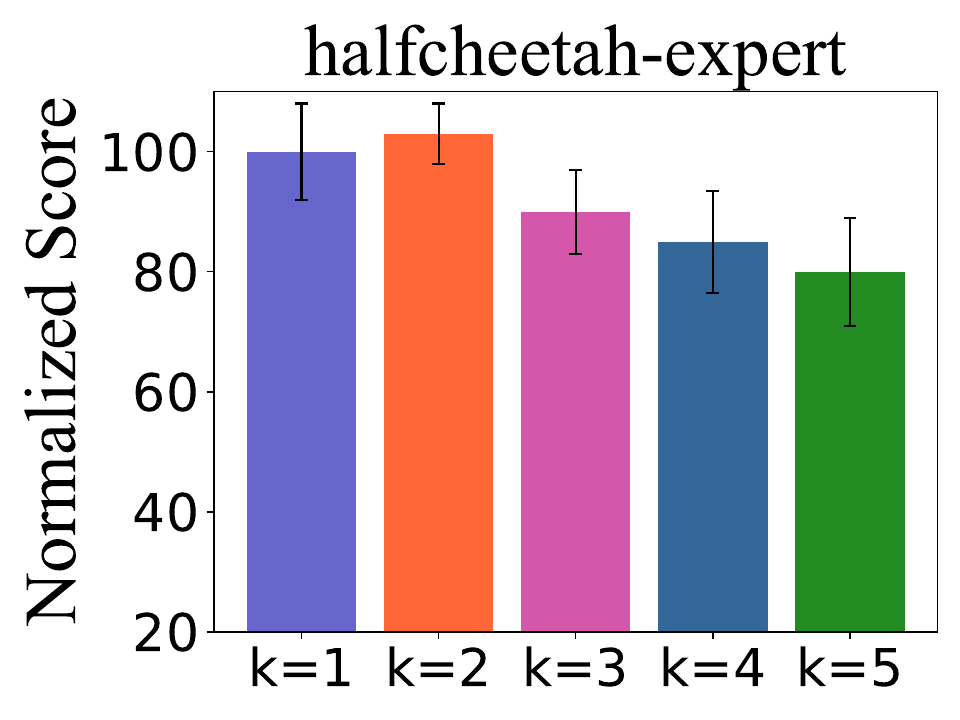}
  \includegraphics[width=0.32\linewidth]{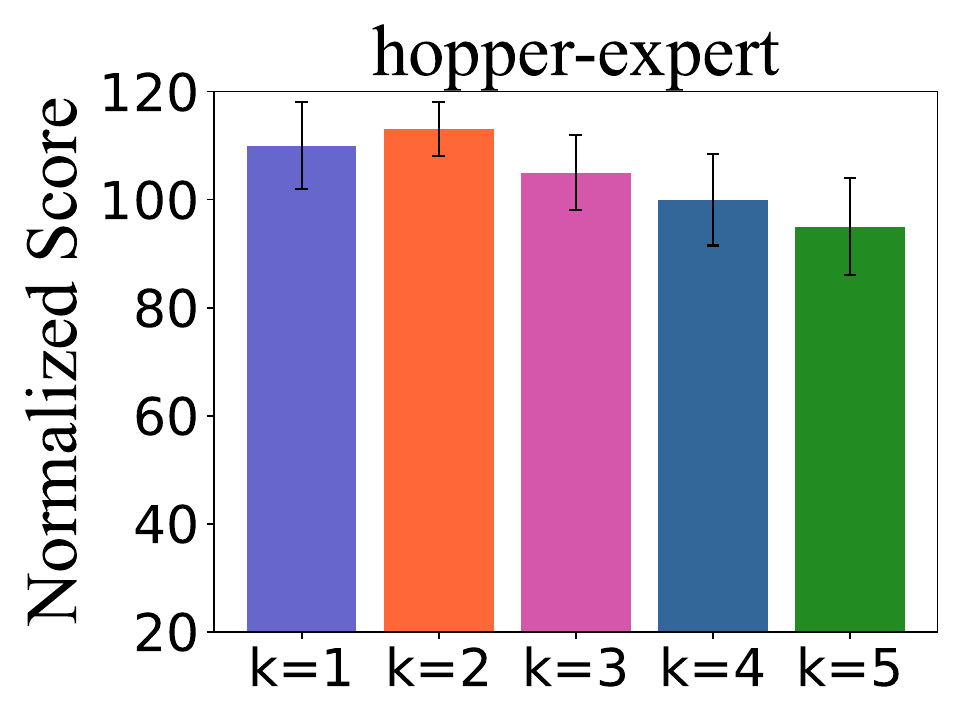}
  \includegraphics[width=0.32\linewidth]{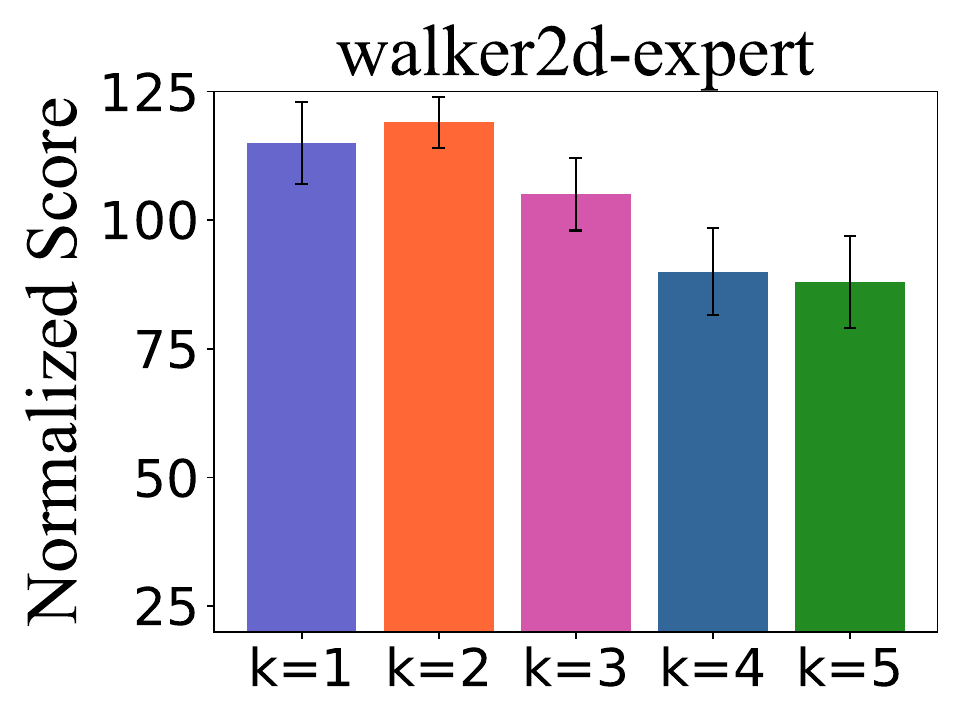}
  % \vspace{-0.2cm}%%减小图片上间隔
  \caption{Impact of the hyperparameter $k$ in MORAL on expert datasets.}
  \label{fig:impact}
  % \vspace{-0.2cm}%%减小图片上间隔
\end{figure}

% \vspace{-3mm}
\paragraph{Impact of $N'$}
We perform comparative experiments on the hyperparameter $N'$ using three expert datasets, and the results are illustrated in Table \ref{tab:impact}. It can be observed that the policy performance reached its peak when $N'$ was set to $10$ in all three environments. As a result, we set the hyperparameter $N'$ in MORAL to $10$. Through comparative experiments on this parameter setting, we can demonstrate that MORAL does not necessitate parameter adjustments in response to environmental changes.

\paragraph{Impact of $k$}
We perform comparative experiments on the hyperparameter $k$ in MORAL using three expert datasets. The results, illustrated in Fig.~\ref{fig:impact}, show that policy performance peaked when $k$ was set to $2$ in all three environments, so we set $k$ to $2$ for MORAL. Further analysis reveals that a minimal $k$ value introduces significant adversarial interference, leading to suboptimal policy outcomes, while larger $k$ values fail to effectively harness these benefits. Despite the necessity of determining the $k$ value, MORAL maintains a lower algorithmic cost compared to fixed horizon rollout approaches. 
MORAL updates policies based on ensemble models without setting different hyperparameters for various environments, enabling execution with fixed parameters and robustly enriching training data without need for environment-specific adjustments.

\begin{table*}[t]
    \caption{Ablation studies on the D4RL datasets. Each score is the normalized score, and averaged over $12$ random seeds. We bold the highest scores and underline the sub-optimal scores. ‘$\pm$’ represents the standard deviation. w/o represents without.} 
\renewcommand{\arraystretch}{1.2}
\setlength{\tabcolsep}{6pt} % 设置列间距为4pt
    \centering 
    \begin{tabular}{ccccccccccc}
    \toprule
        Environment & Dataset & MORAL & \makecell[c]{MORAL w/o DF} & \makecell[c]{MORAL w/o ADV} & RAMBO~\cite{rambo} & MOPO~\cite{mopo} & TD3+BC~\cite{td3bc}
        \\ \hline
        halfcheetah & random & \underline{35.6$\pm$0.6} &  30.4$\pm$2.1 & 25.4$\pm$4.0 & \textbf{40.0$\pm$2.3} & 35.4 & 10.2
        \\
        halfcheetah & medium & \underline{76.3$\pm$0.8} & 72.6$\pm$2.2 & 40.1$\pm$2.5 & \textbf{77.6$\pm$1.5} & 42.3 & 42.8
        \\ 
        halfcheetah & mixed & \textbf{70.0$\pm$0.4}  & 65.3$\pm$1.2 & 52.4$\pm$3.1 & \underline{68.9$\pm$2.3} & 53.1 & 43.3
        \\ 
        halfcheetah & med-expert & \textbf{109.2$\pm$2.4}  & 93.9$\pm$5.6 & 62.3$\pm$2.1 & 93.7$\pm$10.5 & 63.3 & \underline{95.9}
        \\ \hline
        Hopper & random & \textbf{35.7$\pm$9.2}  & \underline{28.7$\pm$1.5} & 10.9$\pm$9.1 & 21.6$\pm$8.0 & 11.7 & 11.0
        \\ 
        Hopper & medium & \textbf{107.4$\pm$0.8}  & \underline{100.5$\pm$3.4} & 25.0$\pm$5.1 & 92.8$\pm$6.0 & 28.0 & 98.5
        \\ 
        Hopper & mixed & \textbf{105.2$\pm$0.6} & \underline{98.4$\pm$1.4} 
        & 59.7$\pm$4.2 & 96.6$\pm$7.0 & 67.5 & 31.4
        \\ 
        Hopper & med-expert & \underline{112.0$\pm$1.1} & 105.0$\pm$2.1 & 21.8$\pm$7.9 & 83.3$\pm$9.1 & 23.7 & \textbf{112.2}
        \\ \hline
        Walker2d & random & \textbf{21.8$\pm$0.2} & \underline{14.4$\pm$1.6} &  9.4$\pm$2.6 & 11.5$\pm$10.5 & 13.6 & 1.4
        \\ 
        Walker2d & medium & \textbf{95.2$\pm$1.7} & \underline{90.1$\pm$2.1}
        & 10.8$\pm$0.8 & 86.9$\pm$2.7 & 11.8 & 79.7
        \\ 
        Walker2d & mixed & \underline{77.5$\pm$2.0} & 68.2$\pm$2.1  & 35.7$\pm$2.4 & \textbf{85.0$\pm$15.0} & 39.0 & 25.2
        \\ 
        Walker2d & med-expert & \textbf{113.6$\pm$1.5} & \underline{105.1$\pm$3.2} & 38.4$\pm$2.3 & 68.3$\pm$20.6 & 44.6 & 101.1
        \\  \hline
        \multicolumn{2}{c}{\textbf{Average}} &  \textbf{80.0} & \underline{72.7} & 32.7 & 68.9 & 36.2 & 54.4
        \\
        \bottomrule
    \end{tabular}
      % \vspace{-0.2cm}%%减小图片上间隔
    \label{tab:abla}
\end{table*}

% \vspace{-2.5mm}
\subsection{Ablation Studies} 
% To substantiate the effectiveness of the differential factor (DF) and MZG, we conducted a comparative ablation experiment involving MORAL, MORAL w/o DF, and MORAL w/o MZG across 12 D4RL datasets. The results in Table \ref{tab:abla} demonstrate that DF enhances MORAL's performance, improving stability and overall outcomes. MORAL w/o DF outperforms the adversarial method RAMBO, indicating the benefits of adversarial data augmentation. Conversely, MORAL w/o MZG performs poorly, failing to match MOPO's performance. Compared to TD3+BC, MORAL w/o DF shows better performance, proving the effectiveness of the data-augmented framework. These findings highlight the crucial roles of MZG and DF in enhancing policy performance and stability. The adversarial game framework using ensemble models robustly improves policy performance through data augmentation under stochastic perturbations, offering a universal architecture for model-based offline RL without costly configurations across diverse tasks.
To substantiate the effectiveness of DF and adversarial process, we conduct a comparative ablation experiment involving MORAL, MORAL without (w/o) DF and  MORAL w/o ADV across 12 D4RL datasets. 
The results presented in Table \ref{tab:abla} demonstrate that the DF can enhance the performance of MORAL, validating that the DF can enhance stability and achieving performance improvements. Furthermore, MORAL w/o DF outperforms RAMBO, showcasing this adversarial data augmentation can improve the offline policy performance. Conversely, MORAL w/o ADV demonstrates subpar policy performance, failing to reach levels comparable to MOPO. 
Compared to the model-free method TD3+BC, MORAL without DF can achieve better performance than TD3+BC.

These experimental outcomes solidify the effectiveness of the adversarial framework and the DF. 
The proposed adversarial framework, leveraging extensive ensemble models, consistently demonstrates robust enhancements in policy performance. Additionally, MORAL w/o DF exhibits a significant performance decline in difficult tasks, underscoring the crucial role of the differential factor in mitigating extrapolation errors and improving the policy stability.
In conclusion, the ablation experiments further demonstrate the effectiveness of the game architecture and the differential factor within the MORAL. We establish the adversarial process to conduct alternating sampling, robustly enhancing policy performance through data augmentation under stochastic perturbations. This approach efficiently utilizes ensemble models, circumventing the necessity for costly configurations across diverse datasets and tasks, while introducing a universal architecture for model-based offline RL. 

\subsection{Analysis of Time Cost} 
\label{sec:time_cost}
To further analyze the time cost of MORAL, we choose the adversarial method RAMBO~\cite{rambo} for comparison. In RAMBO, the offline learning process involves dynamic ensemble models construction, rollout with a fixed horizon, and adversarial model update. 
In MORAL, the offline learning process includes dynamic ensemble models construction and alternating sampling for offline policy optimization. 
The construction of dynamic models in model-based offline RL is a computationally efficient process that does not require excessive consumption of computational resources.
Within the same environmental platform, we conduct comparative experiments using the medium and random datasets across three different environments. The experimental results, as depicted in Figure \ref{fig:time}, demonstrate that MORAL exhibits superior time efficiency compared to RAMBO. This finding further corroborates the effectiveness and practical applicability of MORAL. 

\begin{figure}[t]
  \centering
    % \vspace{-0.5cm}%%减小图片上间隔
  \includegraphics[width=0.46\linewidth]{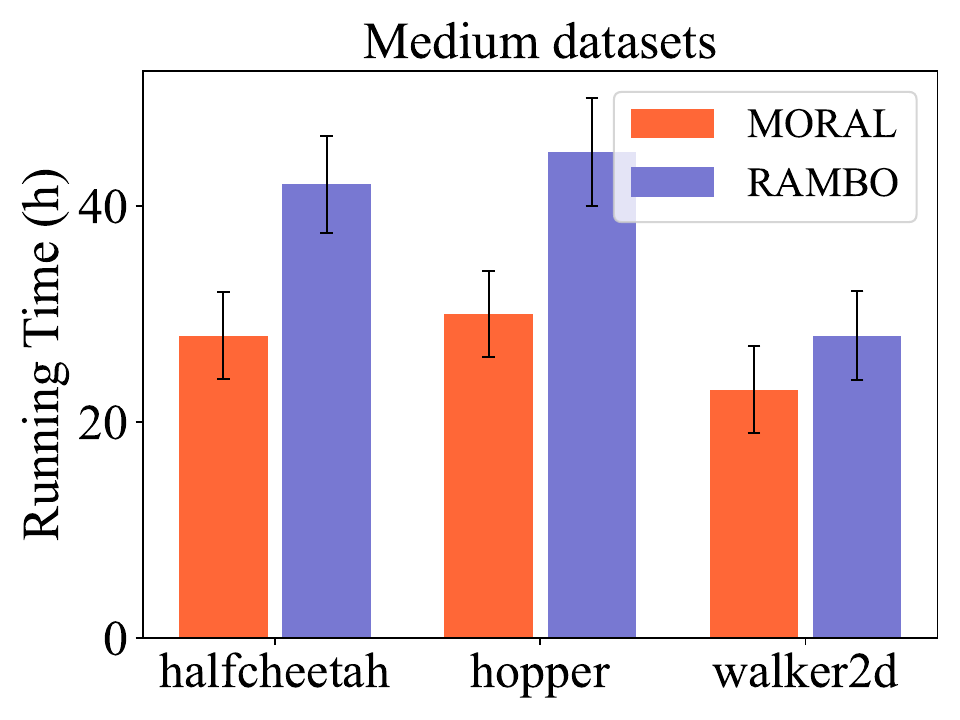}
  \includegraphics[width=0.46\linewidth]{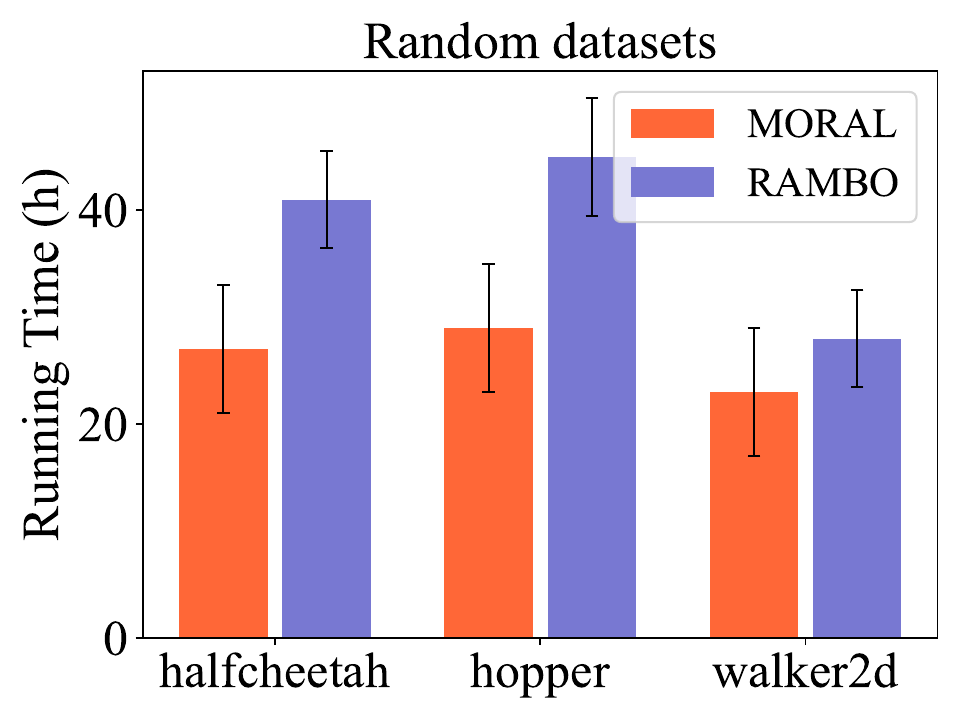}
  % \vspace{-0.2cm}%%减小图片上间隔
  \caption{Running time of MORAL and RAMBO in medium and random datasets of three environments.}
  \label{fig:time}
  % \vspace{-0.2cm}%%减小图片上间隔
\end{figure}

\section{Conclusion}
\label{sec:conclusion}
This study has proposed an offline model-based adversarial data-augmented optimization method. By combining adversarial game-theoretic techniques with the differential factor, the offline rollout procedure is modeled into a two-player alternating sampling. 
Extensive experiments on the D4RL datasets have substantiated the remarkable sample efficiency and wide-ranging applicability of MORAL across diverse domains. MORAL has also achieved stable and robust learning performance across all tasks. 

\textbf{Limitation and Future Work:} A key limitation of MORAL lies in its generalization capabilities. 
For our future work, we will concentrate our efforts on extending this framework to encompass model-free offline RL methods. Additionally, we will explore the incorporation of hierarchical and meta-learning approaches into the MORAL framework to improve generalization in out-of-distribution regions across challenging long-horizon tasks.

% \vfill
% \clearpage
\bibliographystyle{IEEEtran}
\bibliography{tnnls.bib}
\end{document}